\pdfoutput=1
\PassOptionsToPackage{dvipsnames}{xcolor}
\PassOptionsToPackage{ruled}{algorithm2e}
\PassOptionsToPackage{noabbrev,nameinlink,capitalize}{cleveref}
\documentclass[11pt,pmlr,cleveref]{jmlr}

\makeatletter
\renewcommand\ps@jmlrtps{%
  \let\@mkboth\@gobbletwo
  \def\@oddhead{}\let\@evenhead\@oddhead
  \def\@oddfoot{\hfill\small\rmfamily\thepage\hfill}\let\@evenfoot\@oddfoot}
\makeatother
\newenvironment{proofof}[1]{\begin{proof}\textbf{of {#1}}}{\end{proof}}

\usepackage[utf8]{inputenc}
\usepackage[T1]{fontenc}
\usepackage{times}
\usepackage{booktabs}
\usepackage{array}
\usepackage{microtype}
\usepackage{enumitem}
\usepackage{dsfont}
\usepackage{mathtools}
\usepackage{tikz}
\usetikzlibrary{backgrounds,calc,quotes,arrows.meta,shapes.misc,positioning,graphs}
\usepackage{pgfplots}
\pgfplotsset{compat=1.18}
\pgfplotsset{plot coordinates/math parser=false}

\newtheorem{assumption}{Assumption}

\creflabelformat{equation}{#2\textup{#1}#3}
\DeclareRobustCommand{\prstep}[1]{\text{\color{white!40!black}#1}}

\DeclareMathOperator{\KL}{KL}
\DeclareMathOperator{\Law}{Law}
\newcommand{\KLdiv}[2]{\KL\left(#1\,\|\,#2\right)}

\newcommand{\E}{\mathds{E}}
\newcommand{\Ep}[2]{\E_{#1}\left[#2\right]}
\newcommand{\Ex}[1]{\E\left[#1\right]}

\newcommand{\Ber}{\mathcal{B}\mkern-1mu\textit{er}}
\newcommand{\Unif}{\mathcal{U}\mkern-1mu\textit{nif}}

\DeclareMathOperator*{\argmax}{arg\,max}
\DeclareMathOperator*{\argmin}{arg\,min}

\newcommand{\veps}{\varepsilon}
\newcommand{\dd}{\mathop{}\!\mathrm{d}}
\newcommand{\T}{^\top}

\newcommand{\defeq}{\triangleq}
\newcommand{\eqdef}{\triangleq}

\newcommand{\mbg}{\mathbf{g}}
\newcommand{\mbw}{\mathbf{w}}
\newcommand{\mbx}{\mathbf{x}}
\newcommand{\mbz}{\mathbf{z}}

\newcommand{\calB}{\mathcal{B}}
\newcommand{\calF}{\mathcal{F}}
\newcommand{\calG}{\mathcal{G}}
\newcommand{\calH}{\mathcal{H}}
\newcommand{\calJ}{\mathcal{J}}
\newcommand{\calK}{\mathcal{K}}
\newcommand{\calL}{\mathcal{L}}
\newcommand{\calM}{\mathcal{M}}
\newcommand{\calQ}{\mathcal{Q}}
\newcommand{\calV}{\mathcal{V}}

\newcommand{\dsB}{\mathds{B}}
\newcommand{\dsH}{\mathds{H}}
\newcommand{\dsN}{\mathds{N}}
\newcommand{\dsR}{\mathds{R}}

\newcommand{\dsone}{\mathds{1}}

\crefname{assumption}{Assumption}{Assumptions}
\Crefname{assumption}{Assumption}{Assumptions}
\crefname{algocf}{Algorithm}{Algorithms}
\Crefname{algocf}{Algorithm}{Algorithms}
\crefname{lemma}{Lemma}{Lemmas}
\Crefname{lemma}{Lemma}{Lemmas}
\crefname{proposition}{Proposition}{Propositions}
\Crefname{proposition}{Proposition}{Propositions}
\crefname{corollary}{Corollary}{Corollaries}
\Crefname{corollary}{Corollary}{Corollaries}
\crefname{definition}{Definition}{Definitions}
\Crefname{definition}{Definition}{Definitions}
\crefname{remark}{Remark}{Remarks}
\Crefname{remark}{Remark}{Remarks}
\crefname{example}{Example}{Examples}
\Crefname{example}{Example}{Examples}
\crefname{appendix}{Appendix}{Appendices}
\Crefname{appendix}{Appendix}{Appendices}
\crefname{subappendix}{Appendix}{Appendices}
\Crefname{subappendix}{Appendix}{Appendices}
\DontPrintSemicolon

\SetCommentSty{mycommfont}
\newcommand{\Bres}{B_{\mathrm{res}}}
\newcommand{\Cad}{C_{\mathrm{ad}}}
\newcommand{\eapp}[1][]{\veps_{\mathrm{app}#1}}
\newcommand{\eopt}[1][]{\veps_{\mathrm{opt}#1}}
\newcommand{\Rseq}{\mathfrak{R}}
\let\oldPr\Pr
\renewcommand{\Pr}{\oldPr\nolimits}

\hypersetup{
  colorlinks=true,
  citecolor=MidnightBlue,
  linkcolor=MidnightBlue,
  urlcolor=MidnightBlue
}

\title[FQI with Continuous Actions and Adaptive Data]{Fitted Q-Iteration with Continuous Actions and Adaptive Data: Smoothed Coverage and Minimax Rates}

\author{%
 \Name{Manuel Haußmann} \Email{haussmann@imada.sdu.dk}\\
 \Name{Mustafa Mert Çelikok} \Email{celikok@imada.sdu.dk}\\
 \Name{Melih Kandemir} \Email{kandemir@imada.sdu.dk}\\
 \addr University of Southern Denmark
}

\begin{document}

\maketitle

\begin{abstract}%
Fitted Q-iteration (FQI) is well understood when the data come from a fixed distribution and the action set is finite, whereas value-based continuous control collects its data adaptively and maximizes over continuous actions.
Both features break the standard analysis: the regression data are dependent, and every density-ratio concentrability coefficient of a deterministic greedy policy is infinite when the exploration has no atoms.
Conditioning on the history freezes the regression target, so each round of FQI is a regression with martingale noise, and we obtain localized rates that are fast for classes of polylogarithmic metric entropy.
We transfer the regression error to the loss of the greedy policy by smoothing the measures on which this loss is evaluated, at a price that depends on the action dimension and the smoothness of the residuals in the action.
Under deterministic dynamics the state is smoothed as well, and the state dimension enters the price.
Minimax lower bounds with the same exponents in the sample size show that no learner avoids these prices, for restart data in the case of the state dimension.
The coverage conditions do not involve the unknown greedy policies; they hold under exploration floors with restarts and, for action smoothing, bounded transition densities.
\end{abstract}


\section{Introduction}\label{sec:intro}

\emph{Fitted Q-iteration (FQI)} turns value iteration into a sequence of regression problems: at round $k$ one regresses the Bellman targets $R+\gamma\max_{a'}\widehat Q_k(S',a')$, with reward $R$, next state $S'$ and discount $\gamma$, on observed state--action pairs and obtains the next iterate $\widehat Q_{k+1}$.
Its finite-sample theory is well developed when the data come from a fixed distribution and the action set is finite \citep{munos2008,antos2008,chen2019}: regression errors are transferred to the loss of the final greedy policy through concentrability coefficients \citep{munos2007}.
Value-based continuous control breaks both premises, since each batch is collected by a behavior policy that depends on the history, such as the greedy action plus exploration noise, and the greedy step maximizes over a continuous action space.
We analyze FQI in this regime, on Borel state spaces and compact action spaces.
Each round regresses on a fresh batch with an exploration floor in state and action, as with restarts in a simulator; a single trajectory, thinned to every second transition, meets this floor only under a minorization of the dynamics (\cref{prop:protocols}), and replay and approximate maximization are outside our scope (\cref{subsec:scope,sec:discussion}).

\textbf{Two Obstacles.}
\emph{Adaptive data.} The regression data of round $k$ depend on the earlier iterates, so i.i.d.\ uniform convergence does not apply.
Conditioning on the history freezes the regression target, and each round becomes a regression with martingale noise, which sequential symmetrization controls at a rate no faster than $n^{-1/4}$ in the batch size $n$ (\cref{thm:residual}).
A finite cover of the class in the supremum norm localizes this bound without mixing or design conditions, and the residual decays as $\sqrt{\log N_\infty(B/n)/n}$, with $N_\infty(B/n)$ the covering number at scale $B/n$ of a class bounded by $B$ (\cref{thm:fast-entropy}).
\emph{Coverage with continuous actions.} The standard analysis transfers the regression error, measured in $L^2$ of the data distribution, to the policy loss through the densities of the state--action distributions of greedy policies with respect to the data distribution.
Greedy policies are deterministic, so under atomless exploration, such as Gaussian noise, these densities do not exist and every concentrability coefficient of this type is infinite (\cref{lem:obstructions}(i)).
We instead smooth the measures on which the loss is evaluated, in the action (\cref{fig:smoothing}) and, under deterministic dynamics, also in the state.
The price is a modulus of continuity of the residuals, which brings the smoothed dimensions into the rate (\cref{lem:transfer,lem:transferSA}); it shrinks with the smoothness of the residuals and vanishes for finite actions or under a norm equivalence in the action.

\begin{figure}[t]
\centering
\begin{tikzpicture}
\begin{axis}[
  name=coverage, width=6.4cm, height=4.2cm,
  axis lines=left, axis line style={-}, ticks=none, clip=false,
  xmin=0, xmax=1, ymin=0, ymax=2.1,
  xlabel={\small action $a$}, ylabel={\small density},
]
  \addplot[draw=none, fill=gray!20, domain=0:1, samples=80] {0.62+0.22*sin(deg(2*pi*x))+0.08*cos(deg(6*pi*x))} \closedcycle;
  \addplot[gray!70, domain=0:1, samples=80] {0.62+0.22*sin(deg(2*pi*x))+0.08*cos(deg(6*pi*x))};
  \addplot[gray!70, dashed, domain=0:1] {0.3};
  \node[gray!80!black, anchor=south west, font=\footnotesize] at (axis cs:0,0.3) {floor $c$};
  \fill[blue!70!black, opacity=0.3] (axis cs:0.47,0) rectangle (axis cs:0.63,1.3);
  \draw[blue!70!black, thick] (axis cs:0.47,0) rectangle (axis cs:0.63,1.3);
  \draw[orange!80!black, line width=1.4pt, -latex] (axis cs:0.55,0) -- (axis cs:0.55,2.05);
  \node[orange!80!black, anchor=south west, font=\footnotesize, align=left] at (axis cs:0.57,1.7) {greedy action $b$:\\ratio $\infty$};
  \node[blue!70!black, anchor=west, font=\footnotesize, align=left] at (axis cs:0.65,1.1) {window of radius $\sigma$:\\ratio $\leq1/(2\sigma c)$};
\end{axis}
\begin{axis}[
  at={(coverage.south east)}, anchor=south west, xshift=2.3cm,
  width=6.4cm, height=4.2cm,
  axis lines=left, axis line style={-}, ticks=none, clip=false,
  xmin=0, xmax=1, ymin=0, ymax=1.6,
  xlabel={\small radius $\sigma$}, ylabel={\small bound on the loss},
]
  \addplot[blue!70!black, thick, domain=0.04:1, samples=80] {0.14/sqrt(x)};
  \addplot[orange!80!black, thick, domain=0:1, samples=2] {1.1*x};
  \addplot[black, thick, densely dashed, domain=0.06:1, samples=80] {0.14/sqrt(x)+1.1*x};
  \node[blue!70!black, anchor=west, font=\footnotesize] at (axis cs:1,0.14) {transfer $\propto\sigma^{-d_A/2}$};
  \node[orange!80!black, anchor=west, font=\footnotesize] at (axis cs:1,1.1) {modulus $L\sigma$};
  \node[anchor=south east, font=\footnotesize] at (axis cs:0.72,1.0) {sum};
  \draw[gray] (axis cs:0.14,0) -- (axis cs:0.14,0.528);
  \fill (axis cs:0.14,0.528) circle (1.6pt);
  \node[anchor=north, font=\footnotesize] at (axis cs:0.14,0) {$\sigma^*$};
\end{axis}
\end{tikzpicture}
\caption{\emph{Smoothed Coverage.} Left: at one state, data whose density in the action is at least $c$ ({\color{gray}gray}) never cover the {\color{orange!80!black}greedy action} $b$, so its density ratio is infinite.
Averaging the action over a {\color{blue!70!black}window} of radius $\sigma$ gives a ratio of order $\sigma^{-d_A}$ for $d_A$-dimensional actions ($1/(2\sigma c)$ in the figure).
Right: the price is the {\color{orange!80!black}modulus} of the residuals over the window ($L\sigma$ for $L$-Lipschitz residuals), so the bound on the loss falls with $\sigma$ through the {\color{blue!70!black}transfer} and grows through the modulus.
Balancing both at $\sigma^*$ puts $d_A$ into the exponents of \cref{tab:rates}.}
\label{fig:smoothing}
\end{figure}

\begin{table}[t]
\centering
\caption{\emph{Rates at a Glance.} Expected loss of the final greedy policy after $N$ transitions and lower bound for any learner, up to logarithmic factors.
Here $d_S$ and $d_A$ are the state and action dimensions and $\beta$ the smoothness of the residuals in the action.
Rows one and two assume bounded transition densities, row three Lipschitz residuals and, for the lower bound, states drawn from bounded densities, as with restarts; ``--'' marks no lower bound.
Both continuous regimes have matching exponents.}
\label{tab:rates}
\small
\setlength{\tabcolsep}{4pt}
\begin{tabular}{lll}
\toprule
\textsc{Regime} & \textsc{FQI}, upper bound & \textsc{Any learner}, lower bound\\
\midrule
finite $A$ or norm equivalence & $N^{-1/2}$ (\cref{cor:rates}(a), (c)) & -- \\
smoothness $\beta\leq2$ in the action & $N^{-\beta/(d_A+2\beta)}$ (\cref{cor:rates}(b), (b$'$)) & $N^{-\beta/(d_A+2\beta)}$ (\cref{thm:lower}(a))\\
deterministic dynamics, restarts & $N^{-1/(d_S+d_A+2)}$ (\cref{cor:ratesSA}) & $N^{-1/(d_S+d_A+2)}$ (\cref{thm:lower}(b))\\
\bottomrule
\end{tabular}
\end{table}

\begin{theorem}[Informal]\label{thm:informal}
Let FQI run with exact least squares over a Bellman-complete class whose $\log N_\infty(B/n)$ is polylogarithmic in $n$, on $N=nK$ transitions in $K$ rounds, where $K$ is of order $(1-\gamma)^{-1}\log N$, with a density floor of the data in state and action and, except under deterministic dynamics, bounded transition densities.
Up to logarithmic factors and factors that do not depend on $N$, its final greedy policy has expected loss at most of the order of the upper bounds in \cref{tab:rates}.
In the continuous regimes, no learner improves the exponents on families that meet these conditions (under deterministic dynamics, no learner that draws its states from bounded densities, as restarts do).
\end{theorem}
So continuous actions cost $d_A$ in the exponent when the residuals are only $\beta$-smooth in the action, and under deterministic dynamics, restart data cost $d_S$ as well.

\textbf{Contributions.} \emph{(i)} In \cref{subsec:rates}, we derive the rates of \cref{tab:rates} and minimax lower bounds with the same exponents in $N$ (under deterministic dynamics, for learners whose states have bounded densities).
\emph{(ii)} We state the main result under a transfer condition on the realized residuals (\cref{thm:main}) and meet it by smoothing in the action or in state and action (\cref{lem:transfer,lem:transferSA}), which needs a density floor of the data and, for action smoothing, bounded transition densities (\cref{prop:sufficient}).
\emph{(iii)} For adaptive fresh batches, we bound the per-round Bellman residual with a target frozen by conditioning, globally through a sequential complexity and locally through the metric entropy of the class (\cref{thm:residual,thm:fast-entropy}).
\emph{(iv)} We make every object of the algorithm measurable (\cref{sec:setting,lem:measurability}).

\textbf{Related Work.} \emph{From fixed data to continuous actions.} Error propagation with concentrability coefficients goes back to \citet{munos2003,munos2007}.
Finite-sample analyses combine it with regression bounds for samples from a fixed distribution \citep{munos2008,fan2020deepq}, for a single $\beta$-mixing path of a fixed behavior policy \citep{antos2008}, or for i.i.d.\ data under Bellman completeness \citep{chen2019,duan2021}.
These analyses assume finite actions, and the classical ones coverage of all policies, which no data distribution provides with continuous actions (\cref{app:coverage}); \citet{nguyentang2022} assume it for a continuous action space.
\citet{antos2007} handle continuous actions by a search over a policy class of finite capacity, for Lipschitz problems and a fixed behavior policy.
Their rate is $N^{-1/(4(d_A+1))}$ with high probability, up to logarithmic factors, where $N$ is the length of the path.
With exact greedy steps and adaptive data, \cref{cor:rates}(b) gives the exponent $1/(d_A+2)$ for classes whose $\log N_\infty(B/n)$ is polylogarithmic in $n$, in expectation, up to logarithmic factors for $K$ of order $(1-\gamma)^{-1}\log N$.
Under stability conditions on the Bellman operator and the occupation measures, \citet{duan2024taming} bound the policy loss of FQI with linear features and i.i.d.\ data by the squared Bellman residual, which gives the rate $1/n$; our transfer is linear in the residual and needs no stability condition.

\emph{Coverage of a single policy.} Pessimistic methods need coverage of a single comparator only \citep{jin2021pessimism,xie2021bellman,uehara2022pessimistic}, but FQI is not pessimistic, and density-ratio versions of this coverage inherit the obstruction for a deterministic comparator.
Coverage on the function class, such as feature-space coverage \citep{duan2020minimax,jin2021pessimism} or the ratio of Bellman-residual norms of \citet{xie2021bellman}, is a transfer condition like \cref{asm:transfer}.

\emph{Smoothing and Lipschitz structure.} Smoothing of deterministic policies is standard in off-policy evaluation and contextual bandits with continuous actions \citep{kallus2018,krishnamurthy2020,kallus2020doubly}; here it is part of the analysis, so its kernel may be signed (\cref{lem:transfer2}).
The Lipschitz transfer of \citet{antos2007} and Theorem~2 of \citet{oprescu2024lowrank} bound the error under a deterministic policy by a power of its $L^1$ norm under uniform actions; our transfers use the $L^2$ norm of the regression and give larger exponents (\cref{subsec:smoothing}).
Adaptive discretization for Lipschitz \emph{reinforcement learning (RL)} \citep{sinclair2019adaptive,cao2020provably} and algorithms for continuum-armed bandits \citep{kleinberg2008multi,bubeck2011xarmed} have regret exponents $(d+1)/(d+2)$, with $d$ a covering, zooming or near-optimality dimension.
Our Lipschitz rates $N^{-1/(d+2)}$ have the same denominator, with $d=d_A$ under bounded transition densities and $d=d_S+d_A$ for deterministic dynamics with restarts; \cref{thm:lower} uses the lower-bound construction of continuum-armed bandits.

Adaptive data and online analyses are discussed in \cref{app:related}, with a comparison table.

\section{Setting and well-posedness}\label{sec:setting}

A \emph{Borel space} is a Borel subset of a Polish space with its Borel $\sigma$-algebra; it is metrizable.
$\Delta(X)$ denotes the Borel probability measures on $X$; a \emph{kernel} from $X$ to $Y$ is a map $q:X\to\Delta(Y)$ such that $x\mapsto q(E\mid x)$ is Borel for every Borel $E\subseteq Y$; $\calB(X)$ is the Banach space of bounded Borel functions with the supremum norm.
For measures, $\nu\leq C\mu$ means $\nu(E)\leq C\mu(E)$ for all Borel $E$ (equivalently, $\nu\ll\mu$ and $\dd\nu/\dd\mu\leq C$), and $\|h\|_{p,\mu}\defeq(\int|h|^p\dd\mu)^{1/p}$.

\begin{definition}[MDP]\label{def:mdp}
A \emph{Markov decision process (MDP)} is $(S,A,P,r,\gamma)$, where $S$ is a Borel space, $A$ is compact with metric $\varrho_A$, $P$ is a kernel from $S\times A$ to $S$, $r\in\calB(S\times A)$ with $\|r\|_\infty\leq R_{\max}$, and $\gamma\in(0,1)$.
Performance is evaluated from an initial distribution $\rho\in\Delta(S)$.
The policy class $\Pi$ consists of the kernels from $S$ to $A$; a Borel map $g:S\to A$ is identified with the deterministic policy $\delta_{g(s)}$.
\end{definition}

Let $\calQ\defeq\calB(S\times A)$, and let $\calQ_0$ be the set of \emph{upper semicontinuous (USC)} functions in $\calQ$.
For $\pi\in\Pi$ let ${(P^\pi h)(s,a)\defeq\int\!\!\int h(s',a')\,\pi(\dd a'\mid s')P(\dd s'\mid s,a)}$; for ${\eta\in\Delta(S\times A)}$, $\eta P^\pi$ is the law of $(S',A')$ when $(S,A)\sim\eta$, $S'\sim P(\cdot\mid S,A)$, $A'\sim\pi(\cdot\mid S')$, so that $\int h\,\dd(\eta P^\pi)=\int P^\pi h\,\dd\eta$, and $\eta P$ is the law of $S'$.
For $\nu_S\in\Delta(S)$ and a Borel map $g$, $\nu_S\otimes\delta_g$ is the law of $(S,g(S))$ with $S\sim\nu_S$.

\begin{assumption}[Semicontinuous model]\label{asm:model}
$r$ is upper semicontinuous and $P$ is weakly continuous on $S\times A$: $(s,a)\mapsto\int h\,\dd P(\cdot\mid s,a)$ is continuous for every bounded continuous $h:S\to\dsR$.
\end{assumption}

For $Q\in\calQ_0$ let $V_Q(s)\defeq\max_{a\in A}Q(s,a)$.
The Bellman operators are $T^\pi Q\defeq r+\gamma P^\pi Q$ for $Q\in\calQ$ and $(T^*Q)(s,a)\defeq r(s,a)+\gamma\int V_Q(s')\,P(\dd s'\mid s,a)$ for $Q\in\calQ_0$.
A \emph{greedy selector} of $Q\in\calQ_0$ is a Borel map $g:S\to A$ with $Q(s,g(s))=V_Q(s)$ for all $s$.

\begin{lemma}[Well-posedness]\label{lem:wellposed}
Let \cref{asm:model} hold.
(i) $\calQ_0$ is a closed convex cone in $(\calQ,\|\cdot\|_\infty)$.
(ii) For $Q\in\calQ_0$, $V_Q$ is USC and $Q$ has a greedy selector.
(iii) $T^\pi$ maps $\calQ$ into $\calQ$, $T^*$ maps $\calQ_0$ into $\calQ_0$, and $T^\pi Q\leq T^*Q$ for $Q\in\calQ_0$, with equality if $\pi$ is a greedy selector of $Q$.
(iv) $T^\pi$ and $T^*$ are monotone $\gamma$-contractions in $\|\cdot\|_\infty$ on $\calQ$ and on $\calQ_0$, respectively.
\Cref{lem:wellposed-full} gives the full form.
\end{lemma}

Part (ii) is \citet[Prop.~7.33]{bertsekas1978} applied to $-Q$; the other parts are elementary (\cref{app:setting}).
The fixed-point theory is classical; see \citet[Ch.~9]{bertsekas1978} and \citet[Ch.~4]{hernandez1996}.

\begin{theorem}[Fixed points and optimality]\label{thm:fixedpoints}
Under \cref{asm:model}, $T^\pi$ has a unique fixed point ${Q^\pi\in\calQ}$ and $T^*$ has a unique fixed point $Q^*\in\calQ_0$.
With $V^\pi(s)\defeq\int Q^\pi(s,a)\,\pi(\dd a\mid s)$ and $V^*\defeq V_{Q^*}$, every greedy selector $\pi^*$ of $Q^*$ satisfies $V^*=V^{\pi^*}=\max_{\pi\in\Pi}V^\pi$ (full form in \cref{thm:fixedpoints-full}).
\end{theorem}

All policies are Borel, since every maximization is over a function in $\calQ_0$ (\cref{rem:policyclass}).

\textbf{From Bellman Residuals to Policy Loss.} Let $Q_0,\dots,Q_K$ be arbitrary USC functions in $\calQ$, with residuals $e_k\defeq Q_{k+1}-T^*Q_k\in\calQ$.
Let $\pi^*$ and $\pi_k$ ($0\leq k\leq K$) be greedy selectors of $Q^*$ and of $Q_k$, and let $\widehat\pi_K\defeq\pi_K$.
The following version of the propagation argument of \citet{munos2007} collects all evaluation measures of round $k$ into one probability measure $\Xi_k$, on which coverage will be imposed.

\begin{proposition}[Residuals to policy loss]\label{prop:propagation}
Under \cref{asm:model} there are $\Xi_0,\dots,\Xi_{K-1}\in\Delta(S\times A)$ such that
\begin{equation}\label{eq:propagation}
\|V^*-V^{\widehat\pi_K}\|_{1,\rho}\;\leq\;\frac{4}{1-\gamma}\sum_{k=0}^{K-1}\gamma^{K-1-k}\int|e_k|\,\dd\Xi_k\;+\;\frac{4\gamma^K}{1-\gamma}\,\|Q_0-Q^*\|_\infty .
\end{equation}
By its definition \eqref{eq:Xi} in \cref{app:setting}, $\Xi_k$ averages the state--action laws reached from $\rho$ by following $\widehat\pi_K$ for a geometric number of steps, with the last action from $\pi^*$ or $\widehat\pi_K$, and then taking ${K-1-k}$ steps with $\pi^*$ or with $\pi_{K-1},\dots,\pi_{k+1}$.
So it is a countable convex combination of state--action marginals $\nu_S\otimes\delta_g$ of deterministic nonstationary Markov policies started from $\rho$, with ${\nu_S\in\{\rho\}\cup\{\eta P:\eta\in\Delta(S\times A)\}}$ and $g\in\{\pi^*,\pi_{k+1},\widehat\pi_K\}$.
\end{proposition}

With a fixed design $\mu$ and the all-policy coefficient $C\defeq\sup\{\|\dd\nu/\dd\mu\|_\infty\}$ over such marginals, Cauchy--Schwarz gives $\int|e_k|\dd\Xi_k\leq\sqrt C\|e_k\|_{2,\mu}$, and \eqref{eq:propagation} yields the classical bound of order ${\sqrt C(1-\gamma)^{-2}\max_k\|e_k\|_{2,\mu}}$ plus an initialization term \citep{munos2007,munos2008}.
\Cref{sec:coverage} shows why this step fails with continuous actions and how to repair it.

\section{Adaptive data and Bellman regression}\label{sec:regression}

Round $k$ of FQI regresses labels built from $\widehat Q_k$ on data that may depend on the whole history.
We fix the class and the data model, and bound the per-round Bellman residual (\cref{thm:residual,thm:fast-entropy}).

\begin{assumption}[Function class]\label{asm:class}
$\calF=\{f_\theta:\theta\in\Theta\}$, where $\Theta$ is a Borel space (e.g., a closed subset of $\dsR^p$ or of a separable Hilbert space), $\theta\mapsto f_\theta$ is continuous from $\Theta$ into $(\calQ,\|\cdot\|_\infty)$, $f_\theta\in\calQ_0$ for every $\theta$, and $\sup_\theta\|f_\theta\|_\infty\leq B$ with $B\geq R_{\max}/(1-\gamma)$.
We write $\Bres\defeq R_{\max}+(1+\gamma)B$.
\end{assumption}

Linear classes with bounded continuous features and a closed bounded parameter set satisfy \cref{asm:class}, as do balls of a \emph{reproducing-kernel Hilbert space (RKHS)} with a continuous bounded kernel.
So do neural nets with continuous activations and bounded weights on a bounded input set; in all three cases, $B$ is the larger of the sup-norm bound and $R_{\max}/(1-\gamma)$.
The parametrization makes every object of the algorithm measurable: $(\theta,s,a)\mapsto f_\theta(s,a)$, $(\theta,s)\mapsto V_{f_\theta}(s)$ and ${(\theta,s,a)\mapsto T^*f_\theta(s,a)}$ are Borel, and there is a Borel map $g:\Theta\times S\to A$ such that $g(\theta,\cdot)$ is a greedy selector of $f_\theta$.
Suprema over $\Theta$ of $\|\cdot\|_\infty$-continuous functionals reduce to a countable dense subset (\cref{lem:measurability}).

\begin{assumption}[Adaptive transition batches]\label{asm:adaptive-batches}
Each round $k\in\{0,\dots,K-1\}$ collects $n$ transitions, and there are $\sigma$-fields $\calH_k=\calH_{k,0}\subseteq\calH_{k,1}\subseteq\dots\subseteq\calH_{k,n}\subseteq\calH_{k+1}$ such that the current parameter $\widehat\theta_k$ is $\calH_k$-measurable ($\widehat Q_k\defeq f_{\widehat\theta_k}$), each transition $Z_{k,i}=(S_{k,i},A_{k,i},R_{k,i},S'_{k,i})$ is $\calH_{k,i}$-measurable with $|R_{k,i}|\leq R_{\max}$, and, conditionally on $\calH_{k,i-1}\vee\sigma(S_{k,i},A_{k,i})$, the reward has mean $r(S_{k,i},A_{k,i})$ and $S'_{k,i}\sim P(\cdot\mid S_{k,i},A_{k,i})$.
The \emph{design measures} are $\mu_{k,i}\defeq$ conditional law of $(S_{k,i},A_{k,i})$ given $\calH_{k,i-1}$, and $\bar\mu_k\defeq\frac1n\sum_{i=1}^n\mu_{k,i}$.
\end{assumption}

The conditional law of the state--action pairs is otherwise unrestricted: it may depend on the whole history and on $\widehat Q_k$.
The $\sigma$-fields may contain information beyond the regression data, such as skipped transitions or internal randomization.
The measure $\bar\mu_k$ is $\calH_{k,n-1}$-measurable; it is $\calH_k$-measurable when each pair $(S_{k,i},A_{k,i})$ is conditionally independent of $\calH_{k,i-1}$ given $\calH_k$.

The samplewise Bellman labels and the empirical loss of round $k$ are
\begin{equation*}
Y_{k,i}\defeq R_{k,i}+\gamma V_{\widehat Q_k}(S'_{k,i}),\qquad \widehat L_k(f)\defeq\frac1n\sum_{i=1}^n\bigl(f(S_{k,i},A_{k,i})-Y_{k,i}\bigr)^2 .
\end{equation*}

\begin{assumption}[Approximate empirical risk minimization]\label{asm:erm}
For a constant $\eopt[,k]\geq0$, the parameter $\widehat\theta_{k+1}$ is $\calH_{k,n}$-measurable and satisfies $\widehat L_k(f_{\widehat\theta_{k+1}})\leq\inf_{\theta\in\Theta}\widehat L_k(f_\theta)+\eopt[,k]^2$.
\end{assumption}

\Cref{alg:fqi} in \cref{app:setting} summarizes the procedure.
Typical behavior rules are $\epsilon$-greedy with respect to $\widehat Q_k$, or mixtures of Gaussian noise around the greedy action with a uniform component; \cref{prop:protocols} gives sampling protocols under which such rules also provide the coverage needed in \cref{sec:coverage}.

\subsection{Scope of the sampling model}\label{subsec:scope}
\Cref{asm:adaptive-batches} requires \emph{fresh} transitions: every transition enters exactly one regression, in the round in which it is collected.
A single i.i.d.\ data set reused in all rounds is covered by a classical argument uniform over the pair of regression function and target (\cref{prop:single-dataset}).
General replay reuse, as in SAC \citep{haarnoja2018soft} or TD3 \citep{fujimoto2018addressing}, breaks the conditional-mean identity of \cref{subsec:residual} and is not covered (\cref{app:scope}).
The assumption alone does not provide coverage: consecutive transitions of one trajectory that continues across rounds give an atomic state design, for which no density-ratio transfer exists (\cref{lem:obstructions}(ii)), though class-based transfers may apply (\cref{prop:trajectory-gram}).

\subsection{Sequential complexity and the per-round residual}\label{subsec:residual}

For all $k,i$, let $p_{k,i}$ be a regular conditional distribution of $(S_{k,i},A_{k,i},Y_{k,i})$ given $\calH_{k,i-1}$, concentrated on $S\times A\times[-R_{\max}-\gamma B,R_{\max}+\gamma B]$, and let $\mu_{k,i}$ be its $(S,A)$-marginal.
We define
\begin{equation*}
L_k(f)\defeq\frac1n\sum_{i=1}^n\int(f(s,a)-y)^2\,p_{k,i}(\dd s,\dd a,\dd y),\qquad\Delta_k\defeq\sup_{\theta\in\Theta}|L_k(f_\theta)-\widehat L_k(f_\theta)|.
\end{equation*}
For a fixed $f$, $L_k(f)$ is a version of $\frac1n\sum_{i=1}^n\Ex{(f(S_{k,i},A_{k,i})-Y_{k,i})^2\mid\calH_{k,i-1}}$, and for every outcome $|L_k(f_\theta)-L_k(f_{\theta'})|\leq2\Bres\|f_\theta-f_{\theta'}\|_\infty$, so $L_k$ can be evaluated at the data-dependent $\widehat Q_{k+1}$.
Conditionally on $\calH_k$ the target $\widehat Q_k$ is fixed, so $\Delta_k$ is the supremum of a martingale-difference process indexed by $\theta$, and sequential symmetrization \citep{rakhlin2015ptrf} gives $\Ex{\Delta_k}\leq\alpha_n$; Markov's inequality and a union bound over the rounds then give $\Pr(\max_{k<K}\Delta_k>K\alpha_n/\delta)\leq\delta$ for every $\delta\in(0,1)$ (\cref{thm:deviation} in \cref{app:regression}).
Here $\alpha_n\defeq2\overline\Rseq_n/n$, where $\overline\Rseq_n$ is the sequential Rademacher complexity \citep{rakhlin2015} of the squared-loss class with the target parameter held fixed, in the worst case over that parameter (\cref{def:seqrad}).
For the Hilbert-space classes of \cref{prop:hilbert}, which include linear classes and RKHS balls, $\alpha_n\leq2\Bres^2/\sqrt n$ and $\Delta_k$ has sub-Gaussian tails; general classes, including neural nets, go through the sequential contraction inequality (\cref{app:classes}).

Let $h_k\defeq T^*\widehat Q_k$ and $\eapp[,k]\defeq\inf_{\theta}\|f_\theta-h_k\|_{2,\bar\mu_k}$ (a per-round, design-dependent version of the inherent Bellman error of \citet{munos2008}), and let $e_k\defeq\widehat Q_{k+1}-T^*\widehat Q_k$.

\begin{theorem}[Per-round Bellman residual]\label{thm:residual}
Under \cref{asm:model,asm:class,asm:adaptive-batches,asm:erm}, almost surely for every $k$,
\begin{equation}\label{eq:residual}
\|e_k\|_{2,\bar\mu_k}\;\leq\;\eapp[,k]+\sqrt{2\Delta_k}+\eopt[,k].
\end{equation}
So $\Ex{\|e_k\|_{2,\bar\mu_k}}\leq\Ex{\eapp[,k]}+\sqrt{2\alpha_n}+\eopt[,k]$.
With probability at least $1-\delta$, \eqref{eq:residual} holds for all $k$ with $\sqrt{2\Delta_k}$ replaced by $\sqrt{2K\alpha_n/\delta}$, or by $2\Bres(2\log(6K/\delta)/n)^{1/4}$ for the classes of \cref{prop:hilbert}.
\end{theorem}

The proof rests on the conditional-mean identity
\begin{equation}\label{eq:condmean}
\Ex{Y_{k,i}\mid\calH_{k,i-1},S_{k,i},A_{k,i}}=h_k(S_{k,i},A_{k,i}),
\end{equation}
which holds because $\widehat Q_k$ is $\calH_{k,i-1}$-measurable and the reward and next state are drawn from the MDP.
So conditioning on the history freezes the target, and each round is a regression with martingale noise (\cref{app:regression}).
If $\calF$ is Bellman complete ($T^*\calF\subseteq\calF$), then $\eapp[,k]=0$ for every design.

\subsection{\texorpdfstring{Fast rates for classes of finite metric entropy}{Fast rates for classes of finite metric entropy}}\label{subsec:fast}

The bound of \cref{thm:residual} cannot decay faster than $n^{-1/4}$.
The global complexity $\overline\Rseq_n$ is at least of order $\sqrt n$ as soon as some loss is nonzero, so $\alpha_n$ is at least of order $n^{-1/2}$ (\cref{app:fast}).
 The square root in \eqref{eq:residual} is the price of controlling the deviation uniformly over the class instead of near $h_k$.
Localizing the argument, i.e., comparing the noise of each candidate with its own distance to $h_k$, only needs a finite cover of the class in the supremum norm.

\begin{theorem}[Fast rates under a metric-entropy condition]\label{thm:fast-entropy}
Let \cref{asm:model,asm:class,asm:adaptive-batches,asm:erm} hold, and for $\epsilon>0$ let $N_\infty(\epsilon)$ be the smallest size of a set $\Theta_\epsilon\subseteq\Theta$ such that every $\theta\in\Theta$ has some $\theta^\circ\in\Theta_\epsilon$ with $\|f_\theta-f_{\theta^\circ}\|_\infty\leq\epsilon$, and suppose that $B>0$ and $N_\infty(B/n)<\infty$.
Then, for every $\delta\in(0,1)$, with probability at least $1-\delta$, simultaneously for all $k<K$,
\begin{equation}\label{eq:fast-entropy}
\|e_k\|_{2,\bar\mu_k}\leq3\,\eapp[,k]+2\,\eopt[,k]+10\,\Bres\sqrt{\bigl(\log N_\infty(B/n)+\log(6K/\delta)\bigr)/n},
\end{equation}
and $\Ex{\|e_k\|_{2,\bar\mu_k}}\leq3\Ex{\eapp[,k]}+2\eopt[,k]+11\Bres\sqrt{(\log N_\infty(B/n)+\log(6n))/n}$ for every $k$.
\end{theorem}

No mixing and no condition on the design are needed for this bound, because the error is measured in $L^2(\bar\mu_k)$, the average of the conditional design laws, which is also the norm of the transfer condition (\cref{asm:transfer}).
For each element of a finite cover, the noise term and the gap between the empirical and the design norm are martingales, and a union bound over the cover replaces the uniform complexity (\cref{app:fast}).
The same cover argument is used for well-specified classes \citep{agarwal2012contextual,russo2013,jin2021bellman}, while here misspecification enters through $\eapp[,k]$.

The rate depends on the class only through $\log N_\infty(B/n)$.
For networks of fixed architecture with $p$ weights and bounded layer norms, $\log N_\infty(B/n)$ is of order $p\log n$ (\cref{lem:entropy}), so the residual decays as $\sqrt{p\log n/n}$, whereas the width-free bound of \cref{prop:nn} decays as $n^{-1/4}$ up to logarithmic factors.
For balls in the RKHS of a Gaussian kernel on a cube in $\dsR^{d_S+d_A}$, $\log N_\infty(\epsilon)$ is polylogarithmic in $1/\epsilon$~\citep{zhou2002}, which gives the rate $\sqrt{\mathrm{polylog}(n)/n}$ instead of the $n^{-1/4}$ of \cref{prop:hilbert}.
For $d$-dimensional linear classes, a volume argument bounds $\log N_\infty(B/n)$ by $d$ times a logarithmic factor, so the residual decays at the parametric rate up to logarithmic factors (\cref{cor:fast} in \cref{app:fast}).

\section{\texorpdfstring{Coverage, transfers and rates}{Coverage, transfers and rates}}\label{sec:coverage}

We apply \cref{prop:propagation} to the FQI iterates $Q_k=\widehat Q_k$, so that $e_k=\widehat Q_{k+1}-T^*\widehat Q_k$ and $\Xi_k$ is built from greedy selectors of $Q^*$ and of the iterates; the output is $\widehat\pi_K=g(\widehat\theta_K,\cdot)$.
It remains to bound $\int|e_k|\dd\Xi_k$ by the residual $\|e_k\|_{2,\bar\mu_k}$ controlled in \cref{thm:residual,thm:fast-entropy}.

\subsection{\texorpdfstring{Obstructions to density-ratio transfer}{Obstructions to density-ratio transfer}}\label{subsec:obstructions}

\begin{lemma}[Two obstructions]\label{lem:obstructions}
Fix a round $k$ and a realization of the data.
(i) If $\bar\mu_k(\{(s,g(s)):s\in S\})=0$ for every Borel $g\colon S\to A$ (in particular, if the conditional action laws of $\bar\mu_k$ are atomless), then $\Xi_k\not\ll\bar\mu_k$.
The same holds for every state--action marginal of a deterministic policy, so the all-policy coefficient of any such design is infinite.
(ii) If the state marginal of $\bar\mu_k$ is concentrated on a countable set while $\rho$ and all $P(\cdot\mid s,a)$ are atomless, then for every $\sigma\geq0$ no multiple of $\bar\mu_k$ dominates the smoothed measure $\Xi_k^\sigma$, the measure $\Xi_k$ averaged over action windows of radius $\sigma$ (\cref{def:smoothing}).
This is the case when round $k$ uses consecutive transitions of a single trajectory that continues from the previous round.
\end{lemma}

Part (i) makes the unsmoothed coefficient infinite under Gaussian-type exploration.
Part (ii) shows that no smoothing in the action can compensate for missing randomness in the state.

\subsection{\texorpdfstring{Transfer condition and main result}{Transfer condition and main result}}\label{subsec:main}

\Cref{lem:obstructions} rules out the classical transfer, so the main result assumes a condition on the realized residuals and design.
\Cref{subsec:smoothing,subsec:deterministic} derive it from smoothing, and \cref{subsec:class-transfers} from the class.

\begin{assumption}[Transfer condition]\label{asm:transfer}
There are constants $C_{\mathrm{tr}},\omega_{\mathrm{tr}}\geq0$ and an event of probability one on which, for every $k<K$,
\begin{equation}\label{eq:transfer}
\int|e_k|\,\dd\Xi_k\leq\sqrt{C_{\mathrm{tr}}}\,\|e_k\|_{2,\bar\mu_k}+\omega_{\mathrm{tr}}.
\end{equation}
Here $\Xi_k$ is the measure of \cref{prop:propagation} for the FQI iterates, built from some greedy selectors of $Q^*,\widehat Q_1,\dots,\widehat Q_{K-1}$ (which may be random) and from $\widehat\pi_K$.
\end{assumption}

\begin{theorem}[Adaptive-data FQI on continuous spaces]\label{thm:main}
Under \cref{asm:model,asm:class,asm:adaptive-batches,asm:erm,asm:transfer}, the greedy policy $\widehat\pi_K=g(\widehat\theta_K,\cdot)$ satisfies, almost surely,
\begin{equation}\label{eq:main-path}
\|V^*-V^{\widehat\pi_K}\|_{1,\rho}\leq\frac{4}{1-\gamma}\sum_{k=0}^{K-1}\gamma^{K-1-k}\Bigl[\sqrt{C_{\mathrm{tr}}}\,\|e_k\|_{2,\bar\mu_k}+\omega_{\mathrm{tr}}\Bigr]+\frac{8B\gamma^K}{1-\gamma}.
\end{equation}
Consequently, if $\Ex{\|e_k\|_{2,\bar\mu_k}}\leq\veps_k$ for all $k$, then
\begin{equation}\label{eq:main}
\Ex{\|V^*-V^{\widehat\pi_K}\|_{1,\rho}}\leq\frac{4}{1-\gamma}\sum_{k=0}^{K-1}\gamma^{K-1-k}\Bigl[\sqrt{C_{\mathrm{tr}}}\,\veps_k+\omega_{\mathrm{tr}}\Bigr]+\frac{8B\gamma^K}{1-\gamma},
\end{equation}
and, for every event $E$ on which $\|e_k\|_{2,\bar\mu_k}\leq\bar\veps_k$ for all $k$, almost surely on $E$,
\begin{equation*}
\|V^*-V^{\widehat\pi_K}\|_{1,\rho}\leq\frac{4(1-\gamma^K)}{(1-\gamma)^2}\max_{k<K}\Bigl[\sqrt{C_{\mathrm{tr}}}\,\bar\veps_k+\omega_{\mathrm{tr}}\Bigr]+\frac{8B\gamma^K}{1-\gamma}.
\end{equation*}
\end{theorem}

\Cref{thm:residual} gives $\veps_k=\Ex{\eapp[,k]}+\sqrt{2\alpha_n}+\eopt[,k]$ for general classes and \cref{thm:fast-entropy} gives $\veps_k=3\Ex{\eapp[,k]}+2\eopt[,k]+11\Bres\sqrt{(\log N_\infty(B/n)+\log(6n))/n}$ for classes of finite metric entropy; their high-probability versions provide $\bar\veps_k$ for all $k$ simultaneously.
So the policy loss is a discounted sum of the transferred per-round residuals plus an initialization term that decays as $\gamma^K$.

\emph{Proof sketch.}
On the event of \cref{asm:transfer}, we insert \eqref{eq:transfer} into \eqref{eq:propagation} for the realized iterates and the selectors of \cref{asm:transfer}, and use $\|\widehat Q_0-Q^*\|_\infty\leq2B$ (\cref{app:coverage}).

With a total budget of $N=nK$ transitions, $K\geq(1-\gamma)^{-1}\log N$ gives $\gamma^K\leq1/N$, so the initialization term is at most $8B/((1-\gamma)N)$, and the statistical terms depend on $n=N/K$.

\subsection{\texorpdfstring{Action smoothing}{Action smoothing}}\label{subsec:smoothing}

\begin{definition}[Action smoothing]\label{def:smoothing}
Fix $\lambda_A\in\Delta(A)$ with full support.
For $\sigma\geq0$ and $b\in A$ let $\bar\dsB_\sigma(b)$ be the closed $\varrho_A$-ball of radius $\sigma$ around $b$, and let $v_A(\sigma)\defeq\inf_{b\in A}\lambda_A(\bar\dsB_\sigma(b))$, positive for $\sigma>0$ by compactness and full support.
For $\sigma>0$ let $U_\sigma(\cdot\mid b)\defeq\lambda_A(\,\cdot\cap\bar\dsB_\sigma(b))/\lambda_A(\bar\dsB_\sigma(b))$, and let $U_0(\cdot\mid b)\defeq\delta_b$.
For $\nu\in\Delta(S\times A)$ the smoothed measure $\nu^\sigma\in\Delta(S\times A)$ is $\nu^\sigma(E)\defeq\int U_\sigma(E_s\mid b)\,\nu(\dd s,\dd b)$ with $E_s\defeq\{a:(s,a)\in E\}$ (\cref{app:coverage}).
The \emph{action modulus} of $h\in\calQ$ is $\omega_h(\sigma)\defeq\sup_{s\in S}\sup_{\varrho_A(a,b)\leq\sigma}|h(s,a)-h(s,b)|$, and $\omega_\calF(\sigma)\defeq\sup_{\theta,\theta'\in\Theta}\omega_{f_\theta-T^*f_{\theta'}}(\sigma)$.
\end{definition}

\begin{lemma}[Smoothing transfer]\label{lem:transfer}
For $\nu,\mu\in\Delta(S\times A)$, $h\in\calQ$ and $\sigma\geq0$, $\int|h|\dd\nu\leq\int|h|\dd\nu^\sigma+\omega_h(\sigma)$.
If $\nu^\sigma\leq C\mu$, then $\int|h|\dd\nu\leq\sqrt C\,\|h\|_{2,\mu}+\omega_h(\sigma)$.
\end{lemma}

\begin{proof}
$U_\sigma(\cdot\mid b)$ is concentrated on $\bar\dsB_\sigma(b)$, on which $|h(s,b)|\leq|h(s,a)|+\omega_h(\sigma)$; integrate in $a$, then in $\nu$, and apply Cauchy--Schwarz with $\int(\dd\nu^\sigma/\dd\mu)^2\dd\mu\leq C$.
\end{proof}

The transfer cannot be sharpened, up to a constant depending on $d_A$, when the only information on $h$ is an $L^2$ bound and a modulus of continuity (\cref{prop:tight} in \cref{app:coverage}).
At each state, \cref{lem:transfer} with the best window is an interpolation between the $L^2$ norm and the modulus; smoothing fixes one window for all states, which keeps the bound in the affine form of \cref{asm:transfer}.
The smoothing is part of the analysis, not of the algorithm, so its kernel may be signed.
A signed kernel of order~1 \citep[see][Definition~1.3]{tsybakov2009} reproduces affine functions also at the boundary of $A$, so it pays only for the curvature of the residual, through the affine modulus $\omega^{\mathrm{lin}}_\calF$ of \cref{def:smoothing2}, which gives \cref{cor:rates}(b$'$) (\cref{lem:transfer2} in \cref{app:smoothing2}).
With the localized residual of \cref{thm:fast-entropy}, the $L^1$ transfers of \citet{antos2007} and \citet[Theorem~2]{oprescu2024lowrank} give the exponent $\beta/(2(d_A+\beta))$ in $n$ ($\beta=1$ for the former), and our $L^2$ transfers give $\beta/(d_A+2\beta)$ (\cref{cor:rates}(b), (b$'$)), which no learner improves (\cref{thm:lower}).

\begin{assumption}[Smoothed adaptive coverage at scale $\sigma$]\label{asm:coverage}
For some $\sigma\geq0$ there is $\Cad(\sigma)<\infty$ such that, almost surely, $\Xi_k^\sigma\leq\Cad(\sigma)\,\bar\mu_k$ for every $k<K$, where $\Xi_k$ is as in \cref{asm:transfer}.
\end{assumption}

By \cref{lem:transfer}, \cref{asm:coverage} gives \cref{asm:transfer} with $C_{\mathrm{tr}}=\Cad(\sigma)$ and $\omega_{\mathrm{tr}}=\omega_\calF(\sigma)$ (\cref{lem:instances}(i)).
For $\sigma=0$ and a fixed design, a finite classical all-policy coefficient $C$ gives \cref{asm:coverage} with $\Cad(0)\leq C$.
In general, $\Cad(\sigma)$ is at most the smoothed all-policy coefficient (\cref{def:allpolicy}), because each component of $\Xi_k$ is a marginal of a deterministic policy started from $\rho$ (\cref{prop:propagation}).
Under the conditions of \cref{prop:sufficient} below, both are finite for every $\sigma>0$ (\cref{app:allpolicy}).
In \cref{ex:continuous} (\cref{app:example}), the data restart in a rewarded region with an action floor on part of $A$, and the other actions lead to an unrewarded region where the design has no mass; there $\Cad(\sigma)$ is finite for small $\sigma>0$ while the smoothed all-policy coefficient is infinite.

\textbf{Verifiable Coverage.} The coefficient $\Cad(\sigma)$ refers to greedy policies that are unknown when the data are collected.
The following conditions are policy-independent: the floor (b) can be ensured at collection time, and the ceiling (a) is an assumption on $\rho$ and $P$.
They are the adaptive analogues of the uniformly stochastic transitions of \citet{munos2008} and the behavior-density floor of \citet{antos2007}.

\begin{proposition}[Sufficient conditions]\label{prop:sufficient}
Suppose there are a finite measure $\lambda_S$ on $S$ and constants $D,c_\mu>0$ such that (a) $\rho\leq D\lambda_S$ and $P(\cdot\mid s,a)\leq D\lambda_S$ for all $(s,a)$, and (b) almost surely $\mu_{k,i}\geq c_\mu(\lambda_S\otimes\lambda_A)$ for all $k,i$.
Then \cref{asm:coverage} holds for every $\sigma>0$ and any choice of greedy selectors, with $\Cad(\sigma)\leq D/(c_\mu v_A(\sigma))$.
\end{proposition}

\Cref{prop:protocols} in \cref{app:coverage} gives two protocols that meet (b) when every behavior rule draws its action from a law at least $c_A\lambda_A$ given the history and the current state.
Restarts from a law at least $c_S\lambda_S$ give $c_\mu=c_Sc_A$, and thinned trajectories, where FQI regresses on every second transition of a single trajectory, give $c_\mu=c_A^2m$ if $\int P(\cdot\mid s,a)\,\lambda_A(\dd a)\geq m\lambda_S$ for all $s$.
By \cref{prop:second-moment}, (a) can be weakened to $L^2(\lambda_S)$-bounded $\lambda_S$-densities of $\rho$ and of the transitions.

\subsection{\texorpdfstring{Deterministic dynamics}{Deterministic dynamics}}\label{subsec:deterministic}

The ceiling (a) of \cref{prop:sufficient} fails for deterministic dynamics $P(\cdot\mid s,a)=\delta_{F(s,a)}$ unless $F$ takes finitely many values.
If the optimal policy steers the states to a goal state, $\Xi_k$ has an atom there, which smoothing in the action keeps and a design with an atomless state marginal does not cover.
Smoothing in the state as well removes the ceiling, because the smoothed measure then has a bounded density for every evaluation measure, atoms included; the price is a modulus in the state.

\begin{definition}[State--action smoothing]\label{def:smoothingSA}
Let $S\subset\dsR^{d_S}$ and $A\subset\dsR^{d_A}$ be convex bodies whose metrics $\varrho_S$ and $\varrho_A$ are induced by norms, with $\lambda_S$ and $\lambda_A$ the normalized Lebesgue measures.
For $\sigma_S,\sigma_A>0$ and $\nu\in\Delta(S\times A)$, let $\nu^{\sigma_S,\sigma_A}(E)\defeq\int U_{\sigma_S,\sigma_A}(E\mid s,b)\,\nu(\dd s,\dd b)$, where ${U_{\sigma_S,\sigma_A}(\cdot\mid s,b)}\defeq\Unif(\bar\dsB_{\sigma_S}(s)\times\bar\dsB_{\sigma_A}(b))$ with the closed balls of $S$ and of $A$, so that $U_{\sigma_S,\sigma_A}(\cdot\mid s,b)$ is a law on $S\times A$.
The \emph{state--action modulus} of $h\in\calQ$ is $\omega^{\mathrm{SA}}_h(\sigma_S,\sigma_A)\defeq\sup\{|h(s,a)-h(s',a')|:\varrho_S(s,s')\leq\sigma_S,\ \varrho_A(a,a')\leq\sigma_A\}$, and $\omega^{\mathrm{SA}}_\calF(\sigma_S,\sigma_A)\defeq\sup_{\theta,\theta'\in\Theta}\omega^{\mathrm{SA}}_{f_\theta-T^*f_{\theta'}}(\sigma_S,\sigma_A)$.
\end{definition}

\begin{lemma}[State--action smoothing transfer]\label{lem:transferSA}
In the setting of \cref{def:smoothingSA}, let ${\nu\in\Delta(S\times A)}$, $h\in\calQ$, $\sigma_S\in(0,\operatorname{diam}(S)]$ and $\sigma_A\in(0,\operatorname{diam}(A)]$.
Then $\int|h|\,\dd\nu\leq\int|h|\,\dd\nu^{\sigma_S,\sigma_A}+\omega^{\mathrm{SA}}_h(\sigma_S,\sigma_A)$ and $\nu^{\sigma_S,\sigma_A}\leq(\operatorname{diam}(S)/\sigma_S)^{d_S}(\operatorname{diam}(A)/\sigma_A)^{d_A}\,\lambda_S\otimes\lambda_A$.
So, under \cref{asm:model,asm:class,asm:adaptive-batches} and (b) of \cref{prop:sufficient} with these $\lambda_S$ and $\lambda_A$, \cref{asm:transfer} holds for every choice of selectors and every $\rho$ and $P$, with $C_{\mathrm{tr}}=c_\mu^{-1}(\operatorname{diam}(S)/\sigma_S)^{d_S}(\operatorname{diam}(A)/\sigma_A)^{d_A}$ and $\omega_{\mathrm{tr}}=\omega^{\mathrm{SA}}_\calF(\sigma_S,\sigma_A)$.
\end{lemma}

\subsection{\texorpdfstring{Class-based transfers}{Class-based transfers}}\label{subsec:class-transfers}

The transfer condition can also come from the structure of the class instead of smoothing, with $\omega_{\mathrm{tr}}=0$ (\cref{lem:instances}).
A norm equivalence in the action, ${\sup_{a\in A}h(s,a)^2\leq\Lambda\int h(s,a)^2\,\lambda'_A(\dd a)}$ for some $\lambda'_A\in\Delta(A)$, all $s$ and all residuals $h=f_\theta-T^*f_{\theta'}$, gives $C_{\mathrm{tr}}=\Lambda C'$ whenever $\Xi_{k,S}\otimes\lambda'_A\leq C'\bar\mu_k$ for the state marginals $\Xi_{k,S}$ of the $\Xi_k$.
Under (a) and (b) of \cref{prop:sufficient}, this holds for every choice of selectors, with $\lambda'_A=\lambda_A$ and $C'=D/c_\mu$ (\cref{lem:instances}(ii)).
Consider a linear class $\{\langle w,\phi\rangle\}$, with features $\phi$ of Euclidean norm at most $\kappa$, whose Bellman targets lie in the span of the features (\cref{cor:fast}).
On the event that $\lambda_{\min}(\int\phi\phi\T\dd\bar\mu_k)\geq\lambda_0$ for all $k<K$, \eqref{eq:transfer} holds for every choice of selectors with $C_{\mathrm{tr}}=\kappa^2/\lambda_0$ and $\omega_{\mathrm{tr}}=0$, and so does \eqref{eq:main-path} (\cref{lem:instances}(iii)).
This needs no condition on the state marginals, so it also covers trajectory data (\cref{prop:trajectory-gram}).

\subsection{\texorpdfstring{Rates and lower bounds}{Rates and lower bounds}}\label{subsec:rates}

Suppose that $\eapp[,k]\leq\eapp$ almost surely and $\eopt[,k]\leq\eopt$, where $\eapp$ is a design-uniform inherent Bellman error, for example $\sup_{\theta'}\inf_\theta\|f_\theta-T^*f_{\theta'}\|_\infty$, which is $0$ for a Bellman-complete class.
Then $\Ex{\|e_k\|_{2,\bar\mu_k}}\leq a_n$ for every $k$ with the \emph{localized choice}
\begin{equation*}
a_n\defeq3\eapp+2\eopt+11\Bres\sqrt{\bigl(\log N_\infty(B/n)+\log(6n)\bigr)/n}
\end{equation*}
if $B>0$ and $N_\infty(B/n)<\infty$ (\cref{thm:fast-entropy}).
The \emph{global choice} $a_n\defeq\eapp+\eopt+2\Bres n^{-1/4}$ also works for the classes of \cref{prop:hilbert}.
For every class whose sequential Rademacher complexity satisfies $\mathfrak M_n(\calF)\geq B$, the bound holds with the global choice plus a term of order $(\Bres\mathfrak M_n(\calF)\log^{3/2}n/n)^{1/2}$ (\cref{thm:residual,prop:contraction}).

\begin{corollary}[Rates]\label{cor:rates}
Let \cref{asm:model,asm:class,asm:adaptive-batches,asm:erm} and the conditions of \cref{prop:sufficient} hold, and let $a_n$ satisfy $\Ex{\|e_k\|_{2,\bar\mu_k}}\leq a_n$ for every $k<K$.
Then $\Ex{\|V^*-V^{\widehat\pi_K}\|_{1,\rho}}-8B\gamma^K/(1-\gamma)$ is at most:
\begin{enumerate}[label=(\alph*)]
\item $4(1-\gamma)^{-2}\sqrt{D|A|/c_\mu}\,a_n$, if $A$ is finite with the discrete metric and $\lambda_A$ uniform;
\item $8(1-\gamma)^{-2}L^{d_A/(d_A+2)}\bigl(\sqrt{D/c_\mu}\,\operatorname{diam}(A)^{d_A/2}a_n\bigr)^{2/(d_A+2)}$, if $A\subset\dsR^{d_A}$ is a convex body whose metric is induced by a norm on $\dsR^{d_A}$ (which also defines $\operatorname{diam}(A)$ and $L$) with normalized Lebesgue $\lambda_A$, $\omega_\calF(\sigma)\leq L\sigma$, and $n$ is large enough that $\sqrt{D/c_\mu}\,a_n\leq L\operatorname{diam}(A)$.
The modulus bound holds with $L=L_\calF+L_r+\gamma BL_P$ if every $f_\theta$ is $L_\calF$-Lipschitz in $a$, $r$ is $L_r$-Lipschitz in $a$, and $a\mapsto P(\cdot\mid s,a)$ is $L_P$-Lipschitz in total variation;
\item[(b$'$)] $24(1-\gamma)^{-2}L_\beta^{d_A/(d_A+2\beta)}\bigl(\sqrt{c^{(2)}}\,a_n\bigr)^{2\beta/(d_A+2\beta)}$, if $A$ is as in \cref{def:smoothing2} (a convex body that contains a ball of radius $r_A$) and the residuals are smooth in the action.
That is, the largest error $\omega^{\mathrm{lin}}_\calF(\sigma)\defeq\sup_{\theta,\theta'}\omega^{\mathrm{lin}}_{f_\theta-T^*f_{\theta'}}(\sigma)$ of the best affine approximation of a residual on an action ball of radius $\sigma$, with ${\omega^{\mathrm{lin}}_h(\sigma)\defeq\sup_{s,b}\inf_{v\in\dsR^{d_A}}\sup_{\|a-b\|\leq\sigma}|h(s,a)-h(s,b)-\langle v,a-b\rangle|}$, satisfies $\omega^{\mathrm{lin}}_\calF(\sigma)\leq L_\beta\sigma^\beta$ for all $\sigma\in(0,\operatorname{diam}(A)+r_A]$, for some $\beta\in(0,2]$ and $L_\beta>0$.
Here $c^{(2)}\defeq\frac13(2^{d_A+1}+1)(D/c_\mu)\bigl(\operatorname{diam}(A)(\operatorname{diam}(A)+r_A)/r_A\bigr)^{d_A}$ is the constant of \cref{lem:transfer2}, and $n$ is large enough that $\sqrt{c^{(2)}}\,a_n\leq L_\beta(\operatorname{diam}(A)+r_A)^{\beta+d_A/2}$;
\item $4(1-\gamma)^{-2}\sqrt{D\Lambda/c_\mu}\,a_n$, if the norm equivalence of \cref{subsec:class-transfers} holds with $\lambda'_A=\lambda_A$.
\end{enumerate}
\end{corollary}

With the localized choice, (a) and (c) decay as $n^{-1/2}$ and (b) as $n^{-1/(d_A+2)}$, up to logarithmic factors, whenever $\log N_\infty(B/n)$ is polylogarithmic in $n$, as for linear classes and networks of fixed architecture.
With the global choice, the exponents are halved.
Under the same condition, (b$'$) decays as $n^{-\beta/(d_A+2\beta)}$ up to logarithmic factors, which improves on (b) for $\beta>1$; residuals with Lipschitz gradients in the action have $\beta=2$ and the rate $n^{-2/(d_A+4)}$ (sufficient conditions in \cref{lem:modulus2}).
In (b) the dependence on the action dimension enters only through the transfer.
Part (c) shows that a norm equivalence in the action removes this dependence: for Bellman-complete linear classes with features $\phi$, the condition in (c) holds with $\Lambda=\sup_{s,a}\phi(s,a)\T G_s^{+}\phi(s,a)$, where $G_s=\int\phi\phi\T(s,\cdot)\dd\lambda_A$, a feature-space condition in the spirit of linear offline RL \citep{duan2020minimax,jin2021pessimism}.
Route (b) is needed when no norm equivalence holds uniformly in $s$, for instance for classes whose residuals can be narrow bumps in the action.
On the families of \cref{thm:lower}(a), the norm equivalence of (c) holds only with $\Lambda$ of order $N^{d_A/(d_A+2\beta)}$ (\cref{app:lower}), and (c) then gives the exponent of (b$'$).

Combined with \cref{thm:main}, \cref{lem:transferSA} gives rates without the ceiling (a) of \cref{prop:sufficient} (\cref{cor:ratesSA} in \cref{app:smoothingSA}): if $\omega^{\mathrm{SA}}_\calF(\sigma_S,\sigma_A)\leq L_S\sigma_S+L_A\sigma_A$, the policy loss decays as $n^{-1/(d_S+d_A+2)}$ up to logarithmic factors with the localized choice of $a_n$.
The modulus condition holds if the class and the reward are Lipschitz and the kernel $P$ is Lipschitz in the Wasserstein distance $W_1$, as for $s'=F(s,a)$ with a Lipschitz map $F$ (\cref{lem:modulusSA}).
For expansive dynamics, $Q^*$ need not be Lipschitz in the state, and then the approximation error of a Lipschitz class need not vanish (\cref{app:smoothingSA}).
The floor (b) is still needed: restarts provide it, and thinned trajectories only under the minorization of \cref{prop:protocols}(ii).

\begin{theorem}[Minimax lower bounds]\label{thm:lower}
Let $L,M_S>0$, $\beta\in(0,2]$, $\gamma\in(0,1)$, $d_A\geq1$ and $A\defeq[0,1]^{d_A}$.
A \emph{learner} observes $N$ transitions $(S_t,A_t,R_t,S'_t)$, where it draws $S_t$ and then $A_t$ from conditional laws given the past and its internal randomness.
It outputs a policy $\widehat\pi\in\Pi$ whose action laws depend measurably on the state, the data and the internal randomness.
The finite families $\{P_v:v\in\calV\}$ of \cref{thm:lower-full} consist of MDPs with Bernoulli rewards of mean $\frac12$ plus a bump at an unknown location, with expectations $\E_v$ and values $V^*$, $V^{\widehat\pi}$ in $P_v$.
On them, the following holds for every $N$ above a threshold that depends only on $L$, $\beta$, the dimensions and, in (b), $M_S$, with constants $c>0$ that depend only on $\beta$ and the dimensions (explicit values in \cref{thm:lower-full}).
\begin{enumerate}[label=(\alph*)]
\item \emph{Actions.} On the family with one state $s_0$, $\rho=\delta_{s_0}$ and the bump in the action, every learner satisfies ${\max_{v\in\calV}\Ep{v}{\|V^*-V^{\widehat\pi}\|_{1,\rho}}\geq c\,(1-\gamma)^{-1}L^{d_A/(d_A+2\beta)}N^{-\beta/(d_A+2\beta)}}$.
\item \emph{States and actions, restart data.} For $\beta=1$, $d_S\geq1$, $S\defeq[0,1]^{d_S}$ and $\rho$ the uniform law on $S$, consider the family with deterministic transitions to the center of $S$ and the bump in state and action.
Every learner whose conditional law of each $S_t$ has $\rho$-density at most $M_S$ satisfies ${\max_{v\in\calV}\Ep{v}{\|V^*-V^{\widehat\pi}\|_{1,\rho}}\geq c\,\gamma(1-\gamma)^{-1}(L^{d_S+d_A}/(M_SN))^{1/(d_S+d_A+2)}}$.
\end{enumerate}
For $N$ large enough, both families satisfy the conditions of \cref{cor:rates}(b$'$) in (a) and of \cref{cor:ratesSA} in (b), with a finite Bellman-complete class that contains $Q^*$ (\cref{lem:lower-conditions}).
\end{theorem}

The proof of (a) hides a bump of height $L\sigma^\beta$ in one of about $\sigma^{-d_A}$ disjoint balls of radius $\sigma$, each of which must be sampled about $(L\sigma^\beta)^{-2}$ times to tell whether it holds the bump (\cref{lem:divergence}), and missing it costs $L\sigma^\beta/(1-\gamma)$.
In (b), the bump also has radius $\sigma$ in the state, around the center $s_0$ of $S$, and states with $\rho$-density at most $M_S$ fall that close to $s_0$ with probability at most $M_S(2\sigma)^{d_S}$.
The construction of (a) is the classical one for continuum-armed bandits \citep{kleinberg2004nearly,kleinberg2008multi,bubeck2011xarmed}, and its rate is known for the simple regret of noisy optimization with adaptive queries \citep{wang2018optimization,locatelli2018adaptivity}.
The denominator $d_S+d_A+2$ of (b) is that of Lipschitz contextual bandits \citep{lu2010contextual,slivkins2014contextual}, whose lower bounds spread the contexts evenly, as the density bound spreads the states here.
Part (a) holds for every way of collecting the data, and part (b) only for learners whose states have bounded densities, as with restarts: on its family, a learner that follows a trajectory reaches $s_0$ after one step and avoids $d_S$.
However, our density-ratio transfers fail for consecutive transitions of a trajectory that continues across rounds, whose designs are carried by finitely many states: the smoothed measures are atomless in the state under atomless dynamics (\cref{lem:obstructions}(ii)) and have densities under state--action smoothing (\cref{lem:transferSA}).
With $K$ of order $(1-\gamma)^{-1}\log N$ rounds of $n=N/K$ transitions, \cref{cor:rates}(b$'$) and \cref{cor:ratesSA} match the bounds of (a) and (b), respectively, in $N$ and $L$, up to logarithmic factors and factors that do not depend on $N$ or $L$.
The classes of \cref{lem:lower-conditions} grow with $N$; whether one class that does not depend on $N$ forces these exponents is open.
A simulation of the families of (a) gives the same exponent for the regression step of FQI (\cref{app:threshold}).

\section{Discussion}\label{sec:discussion}

\textbf{What Is New and What Is Classical.} The fixed-point theory (\cref{thm:fixedpoints}), the selection theorem, the propagation argument (\cref{prop:propagation}) and sequential symmetrization are classical \citep{bertsekas1978,munos2007,rakhlin2015ptrf}.
So are the cover argument behind \cref{thm:fast-entropy}, for well-specified classes \citep{agarwal2012contextual,russo2013,jin2021bellman}, and the construction behind \cref{thm:lower} \citep{kleinberg2004nearly,kleinberg2008multi,bubeck2011xarmed}.
What is new is their combination for greedy FQI with adaptive fresh batches on Borel spaces.
The regression freezes the target by conditioning and localizes under misspecification, and the transfer condition on the realized residuals is met by smoothing under a density floor of the data and, for action smoothing, bounded transition densities.
\citet{maran2024projection} attain the exponent of row three of \cref{tab:rates} with a generative model and smooth Bellman operators and give a matching lower bound; here FQI attains the optimal exponents from adaptive data, and only $d_A$ enters them under bounded transition densities.

\textbf{Remaining Problems.} \emph{Restarts.} Under deterministic dynamics the rates need a density floor in the state (\cref{subsec:rates}); the Gram-matrix transfer needs none (\cref{subsec:class-transfers}), but it requires a linear class that is complete for the dynamics at hand.
\emph{Exact maximization.} An approximate maximizer perturbs the labels and the greedy step, and an additive error per round in \cref{thm:residual} and \cref{prop:propagation} is the natural target.
\emph{Fixed exploration floors.} The floors of \cref{prop:protocols} are fixed, so the regret of the behavior is linear in $N$ for nontrivial problems (\cref{app:regret}), and the bound on $\Cad(\sigma)$ in \cref{prop:sufficient} grows as the floor shrinks.
The proof of \cref{thm:main} applies \eqref{eq:transfer} round by round and goes through with round-dependent constants, so a floor that decays over the rounds only changes the constant of each round.
Whether some schedule keeps a final-policy bound and makes the regret of the behavior sublinear is open.
\emph{Horizon.} The upper bounds carry $(1-\gamma)^{-2}$ from the propagation step and further powers of $(1-\gamma)^{-1}$ through $\Bres$, while the lower bounds carry $(1-\gamma)^{-1}$.
\emph{Size-free classes.} \Cref{thm:fast-entropy} needs a finite cover at scale $B/n$, so size-free classes such as the width-free networks of \cref{prop:nn} keep the global rate $n^{-1/4}$ up to logarithmic factors.
Classes of polynomial metric entropy, with $\log N_\infty(\epsilon)$ of order $\epsilon^{-q}$, get the rate $n^{-1/(2(1+q))}$ of \cref{prop:fast-entropy-scale}, which improves on $n^{-1/4}$ only for $q<1$.
Offset sequential complexities \citep{rakhlin2014online,liang2015offset}, combined with the frozen target, are a natural route to localized rates for both.
\emph{Smoothness above two.} The kernel of \cref{lem:transfer2} reproduces affine functions, and \cref{thm:lower}(a) is stated for $\beta\leq2$; kernels of higher order \citep[Definition~1.3]{tsybakov2009} or an interpolation of order $\beta$ at each state are the natural candidates for $\beta>2$.

\section*{AI Disclosure}\label{sec:ai-disclosure}
The problem formulation and the main line of the analysis are the authors' own.
This covers the setting and its well-posedness results in \cref{sec:setting} and the analysis of Bellman regression under adaptive data that yields \cref{thm:residual}.
It also covers the two obstructions of \cref{lem:obstructions} and action smoothing (\cref{def:smoothing,lem:transfer}), which together with smoothed adaptive coverage (\cref{asm:coverage}) give the performance bound of \cref{thm:main}.
Each result listed below builds on this framework.
We used Claude (Anthropic), a large language model assistant, to check proofs and to draft parts of the text.
Claude proposed and proved several results, which the authors verified:
\emph{(i)} the Gram-matrix transfer for trajectory data in \cref{prop:trajectory-gram};
\emph{(ii)} the transfer condition \cref{asm:transfer}, \cref{thm:main} in this form and \cref{lem:instances};
\emph{(iii)} \cref{def:allpolicy}, the choice of $B$ in \cref{ex:continuous}, the weaker hypothesis $W_0\kappa\leq B$ of \cref{prop:hilbert} and the exponent in \cref{prop:nn};
\emph{(iv)} the fast rates of \cref{thm:fast-entropy,prop:fast-entropy-scale,lem:entropy,cor:fast};
\emph{(v)} the second-order smoothing of \cref{def:smoothing2,lem:kernel2,lem:transfer2,lem:modulus2}, with the boundary argument for the signed kernel, part (b$'$) of \cref{cor:rates} and \cref{prop:second-moment};
\emph{(vi)} the state--action smoothing of \cref{def:smoothingSA,lem:transferSA,lem:modulusSA,cor:ratesSA};
\emph{(vii)} the minimax lower bounds of \cref{thm:lower,thm:lower-full,lem:divergence,lem:lower-conditions}, with the matching upper bounds and the order of $\Lambda$ on their families in \cref{app:lower};
\emph{(viii)} the exponents of the $L^1$ transfers under the localized residual and the per-state interpolation after \cref{lem:transfer}.
Claude also proposed the corrected statements of \cref{lem:wellposed-full} and \cref{prop:protocols}, and wrote the code of the experiment in \cref{app:threshold}.
\Cref{fig:roadmap} in \cref{appsec:dependencies} shows how these results and those of the authors depend on each other.

\bibliography{references}

\begin{thebibliography}{59}
\providecommand{\natexlab}[1]{#1}
\providecommand{\url}[1]{\texttt{#1}}
\expandafter\ifx\csname urlstyle\endcsname\relax
  \providecommand{\doi}[1]{doi: #1}\else
  \providecommand{\doi}{doi: \begingroup \urlstyle{rm}\Url}\fi

\bibitem[Abbasi-Yadkori et~al.(2011)Abbasi-Yadkori, P{\'a}l, and Szepesv{\'a}ri]{abbasiyadkori2011improved}
Yasin Abbasi-Yadkori, D{\'a}vid P{\'a}l, and Csaba Szepesv{\'a}ri.
\newblock Improved algorithms for linear stochastic bandits.
\newblock In \emph{Advances in Neural Information Processing Systems (NeurIPS)}, 2011.

\bibitem[Agarwal et~al.(2012)Agarwal, Dud{\'\i}k, Kale, Langford, and Schapire]{agarwal2012contextual}
Alekh Agarwal, Miroslav Dud{\'\i}k, Satyen Kale, John Langford, and Robert~E. Schapire.
\newblock Contextual bandit learning with predictable rewards.
\newblock In \emph{International Conference on Artificial Intelligence and Statistics (AISTATS)}, 2012.

\bibitem[Antos et~al.(2007)Antos, Szepesv{\'a}ri, and Munos]{antos2007}
Andr{\'a}s Antos, Csaba Szepesv{\'a}ri, and R{\'e}mi Munos.
\newblock Fitted {Q}-iteration in continuous action-space {MDP}s.
\newblock In \emph{Advances in Neural Information Processing Systems (NeurIPS)}, 2007.

\bibitem[Antos et~al.(2008)Antos, Szepesv{\'a}ri, and Munos]{antos2008}
Andr{\'a}s Antos, Csaba Szepesv{\'a}ri, and R{\'e}mi Munos.
\newblock Learning near-optimal policies with {Bellman}-residual minimization based fitted policy iteration and a single sample path.
\newblock \emph{Machine Learning}, 2008.

\bibitem[Bertsekas and Shreve(1978)]{bertsekas1978}
Dimitri~P. Bertsekas and Steven~E. Shreve.
\newblock \emph{Stochastic Optimal Control: The Discrete-Time Case}.
\newblock Academic Press, 1978.

\bibitem[Bibaut et~al.(2021)Bibaut, Kallus, Dimakopoulou, Chambaz, and van~der Laan]{bibaut2021risk}
Aur{\'e}lien Bibaut, Nathan Kallus, Maria Dimakopoulou, Antoine Chambaz, and Mark van~der Laan.
\newblock Risk minimization from adaptively collected data: Guarantees for supervised and policy learning.
\newblock In \emph{Advances in Neural Information Processing Systems (NeurIPS)}, 2021.

\bibitem[Bubeck et~al.(2011)Bubeck, Munos, Stoltz, and Szepesv{\'a}ri]{bubeck2011xarmed}
S{\'e}bastien Bubeck, R{\'e}mi Munos, Gilles Stoltz, and Csaba Szepesv{\'a}ri.
\newblock {$\mathcal{X}$}-armed bandits.
\newblock \emph{Journal of Machine Learning Research}, 2011.

\bibitem[Cao and Krishnamurthy(2020)]{cao2020provably}
Tongyi Cao and Akshay Krishnamurthy.
\newblock Provably adaptive reinforcement learning in metric spaces.
\newblock In \emph{Advances in Neural Information Processing Systems (NeurIPS)}, 2020.

\bibitem[Chen and Jiang(2019)]{chen2019}
Jinglin Chen and Nan Jiang.
\newblock Information-theoretic considerations in batch reinforcement learning.
\newblock In \emph{International Conference on Machine Learning (ICML)}, 2019.

\bibitem[Du et~al.(2021)Du, Kakade, Lee, Lovett, Mahajan, Sun, and Wang]{du2021bilinear}
Simon Du, Sham Kakade, Jason Lee, Shachar Lovett, Gaurav Mahajan, Wen Sun, and Ruosong Wang.
\newblock Bilinear classes: A structural framework for provable generalization in {RL}.
\newblock In \emph{International Conference on Machine Learning (ICML)}, 2021.

\bibitem[Duan and Wainwright(2024)]{duan2024taming}
Yaqi Duan and Martin~J. Wainwright.
\newblock Taming ``data-hungry'' reinforcement learning? {S}tability in continuous state-action spaces.
\newblock In \emph{Advances in Neural Information Processing Systems (NeurIPS)}, 2024.

\bibitem[Duan et~al.(2020)Duan, Jia, and Wang]{duan2020minimax}
Yaqi Duan, Zeyu Jia, and Mengdi Wang.
\newblock Minimax-optimal off-policy evaluation with linear function approximation.
\newblock In \emph{International Conference on Machine Learning (ICML)}, 2020.

\bibitem[Duan et~al.(2021)Duan, Jin, and Li]{duan2021}
Yaqi Duan, Chi Jin, and Zhiyuan Li.
\newblock Risk bounds and {Rademacher} complexity in batch reinforcement learning.
\newblock In \emph{International Conference on Machine Learning (ICML)}, 2021.

\bibitem[Fan et~al.(2020)Fan, Wang, Xie, and Yang]{fan2020deepq}
Jianqing Fan, Zhaoran Wang, Yuchen Xie, and Zhuoran Yang.
\newblock A theoretical analysis of deep {Q}-learning.
\newblock In \emph{Learning for Dynamics and Control (L4DC)}, 2020.

\bibitem[Foster et~al.(2021)Foster, Kakade, Qian, and Rakhlin]{foster2021dec}
Dylan~J. Foster, Sham~M. Kakade, Jian Qian, and Alexander Rakhlin.
\newblock The statistical complexity of interactive decision making.
\newblock \emph{arXiv preprint arXiv:2112.13487}, 2021.

\bibitem[Foster et~al.(2022)Foster, Krishnamurthy, Simchi-Levi, and Xu]{foster2022impossibility}
Dylan~J. Foster, Akshay Krishnamurthy, David Simchi-Levi, and Yunzong Xu.
\newblock Offline reinforcement learning: Fundamental barriers for value function approximation.
\newblock In \emph{Conference on Learning Theory (COLT)}, 2022.

\bibitem[Freedman(1975)]{freedman1975tail}
David~A. Freedman.
\newblock On tail probabilities for martingales.
\newblock \emph{The Annals of Probability}, 1975.

\bibitem[Fujimoto et~al.(2018)Fujimoto, {van Hoof}, and Meger]{fujimoto2018addressing}
Scott Fujimoto, Herke {van Hoof}, and David Meger.
\newblock Addressing function approximation error in actor-critic methods.
\newblock In \emph{International Conference on Machine Learning (ICML)}, 2018.

\bibitem[Haarnoja et~al.(2018)Haarnoja, Zhou, Abbeel, and Levine]{haarnoja2018soft}
Tuomas Haarnoja, Aurick Zhou, Pieter Abbeel, and Sergey Levine.
\newblock Soft actor-critic: Off-policy maximum entropy deep reinforcement learning with a stochastic actor.
\newblock In \emph{International Conference on Machine Learning (ICML)}, 2018.

\bibitem[Hadad et~al.(2021)Hadad, Hirshberg, Zhan, Wager, and Athey]{hadad2021confidence}
Vitor Hadad, David~A. Hirshberg, Ruohan Zhan, Stefan Wager, and Susan Athey.
\newblock Confidence intervals for policy evaluation in adaptive experiments.
\newblock \emph{Proceedings of the National Academy of Sciences}, 2021.

\bibitem[Hern{\'a}ndez-Lerma and Lasserre(1996)]{hernandez1996}
On{\'e}simo Hern{\'a}ndez-Lerma and Jean~B. Lasserre.
\newblock \emph{Discrete-Time {Markov} Control Processes: Basic Optimality Criteria}.
\newblock Springer, 1996.

\bibitem[Jiang et~al.(2017)Jiang, Krishnamurthy, Agarwal, Langford, and Schapire]{jiang2017contextual}
Nan Jiang, Akshay Krishnamurthy, Alekh Agarwal, John Langford, and Robert~E. Schapire.
\newblock Contextual decision processes with low {Bellman} rank are {PAC}-learnable.
\newblock In \emph{International Conference on Machine Learning (ICML)}, 2017.

\bibitem[Jin et~al.(2021{\natexlab{a}})Jin, Liu, and Miryoosefi]{jin2021bellman}
Chi Jin, Qinghua Liu, and Sobhan Miryoosefi.
\newblock {Bellman} eluder dimension: New rich classes of {RL} problems, and sample-efficient algorithms.
\newblock In \emph{Advances in Neural Information Processing Systems (NeurIPS)}, 2021{\natexlab{a}}.

\bibitem[Jin et~al.(2021{\natexlab{b}})Jin, Yang, and Wang]{jin2021pessimism}
Ying Jin, Zhuoran Yang, and Zhaoran Wang.
\newblock Is pessimism provably efficient for offline {RL}?
\newblock In \emph{International Conference on Machine Learning (ICML)}, 2021{\natexlab{b}}.

\bibitem[Kallus and Uehara(2020)]{kallus2020doubly}
Nathan Kallus and Masatoshi Uehara.
\newblock Doubly robust off-policy value and gradient estimation for deterministic policies.
\newblock In \emph{Advances in Neural Information Processing Systems (NeurIPS)}, 2020.

\bibitem[Kallus and Zhou(2018)]{kallus2018}
Nathan Kallus and Angela Zhou.
\newblock Policy evaluation and optimization with continuous treatments.
\newblock In \emph{International Conference on Artificial Intelligence and Statistics (AISTATS)}, 2018.

\bibitem[Kleinberg(2004)]{kleinberg2004nearly}
Robert Kleinberg.
\newblock Nearly tight bounds for the continuum-armed bandit problem.
\newblock In \emph{Advances in Neural Information Processing Systems (NeurIPS)}, 2004.

\bibitem[Kleinberg et~al.(2008)Kleinberg, Slivkins, and Upfal]{kleinberg2008multi}
Robert Kleinberg, Aleksandrs Slivkins, and Eli Upfal.
\newblock Multi-armed bandits in metric spaces.
\newblock In \emph{ACM Symposium on Theory of Computing (STOC)}, 2008.

\bibitem[Krishnamurthy et~al.(2020)Krishnamurthy, Langford, Slivkins, and Zhang]{krishnamurthy2020}
Akshay Krishnamurthy, John Langford, Aleksandrs Slivkins, and Chicheng Zhang.
\newblock Contextual bandits with continuous actions: Smoothing, zooming, and adapting.
\newblock \emph{Journal of Machine Learning Research}, 2020.

\bibitem[Lai and Wei(1982)]{lai1982least}
Tze~Leung Lai and Ching~Zong Wei.
\newblock Least squares estimates in stochastic regression models with applications to identification and control of dynamic systems.
\newblock \emph{The Annals of Statistics}, 1982.

\bibitem[Lazaric et~al.(2012)Lazaric, Ghavamzadeh, and Munos]{lazaric2012}
Alessandro Lazaric, Mohammad Ghavamzadeh, and R{\'e}mi Munos.
\newblock Finite-sample analysis of least-squares policy iteration.
\newblock \emph{Journal of Machine Learning Research}, 2012.

\bibitem[Liang et~al.(2015)Liang, Rakhlin, and Sridharan]{liang2015offset}
Tengyuan Liang, Alexander Rakhlin, and Karthik Sridharan.
\newblock Learning with square loss: Localization through offset {Rademacher} complexity.
\newblock In \emph{Conference on Learning Theory (COLT)}, 2015.

\bibitem[Locatelli and Carpentier(2018)]{locatelli2018adaptivity}
Andrea Locatelli and Alexandra Carpentier.
\newblock Adaptivity to smoothness in {$\mathcal{X}$}-armed bandits.
\newblock In \emph{Conference on Learning Theory (COLT)}, Proceedings of Machine Learning Research, 2018.

\bibitem[Lu et~al.(2010)Lu, P{\'a}l, and P{\'a}l]{lu2010contextual}
Tyler Lu, D{\'a}vid P{\'a}l, and Martin P{\'a}l.
\newblock Contextual multi-armed bandits.
\newblock In \emph{International Conference on Artificial Intelligence and Statistics (AISTATS)}, 2010.

\bibitem[Madhow et~al.(2025)Madhow, Qiao, Yin, and Wang]{madhow2025rates}
Sunil Madhow, Dan Qiao, Ming Yin, and Yu-Xiang Wang.
\newblock Rates for offline reinforcement learning with adaptively collected data.
\newblock In \emph{Learning for Dynamics and Control (L4DC)}, 2025.

\bibitem[Maran et~al.(2024)Maran, Metelli, Papini, and Restelli]{maran2024projection}
Davide Maran, Alberto~Maria Metelli, Matteo Papini, and Marcello Restelli.
\newblock Projection by convolution: Optimal sample complexity for reinforcement learning in continuous-space {MDP}s.
\newblock In \emph{Conference on Learning Theory (COLT)}, 2024.

\bibitem[Munos(2003)]{munos2003}
R{\'e}mi Munos.
\newblock Error bounds for approximate policy iteration.
\newblock In \emph{International Conference on Machine Learning (ICML)}, 2003.

\bibitem[Munos(2007)]{munos2007}
R{\'e}mi Munos.
\newblock Performance bounds in ${L}^p$-norm for approximate value iteration.
\newblock \emph{SIAM Journal on Control and Optimization}, 2007.

\bibitem[Munos and Szepesv{\'a}ri(2008)]{munos2008}
R{\'e}mi Munos and Csaba Szepesv{\'a}ri.
\newblock Finite-time bounds for fitted value iteration.
\newblock \emph{Journal of Machine Learning Research}, 2008.

\bibitem[Nguyen-Tang et~al.(2022)Nguyen-Tang, Gupta, Tran-The, and Venkatesh]{nguyentang2022}
Thanh Nguyen-Tang, Sunil Gupta, Hung Tran-The, and Svetha Venkatesh.
\newblock On sample complexity of offline reinforcement learning with deep {ReLU} networks in {Besov} spaces.
\newblock \emph{Transactions on Machine Learning Research}, 2022.

\bibitem[Oprescu et~al.(2024)Oprescu, Bennett, and Kallus]{oprescu2024lowrank}
Miruna Oprescu, Andrew Bennett, and Nathan Kallus.
\newblock Low-rank {MDP}s with continuous action spaces.
\newblock In \emph{International Conference on Artificial Intelligence and Statistics (AISTATS)}, 2024.

\bibitem[Pinelis(1994)]{pinelis1994}
Iosif Pinelis.
\newblock Optimum bounds for the distributions of martingales in {B}anach spaces.
\newblock \emph{The Annals of Probability}, 1994.

\bibitem[Rakhlin and Sridharan(2014)]{rakhlin2014online}
Alexander Rakhlin and Karthik Sridharan.
\newblock Online non-parametric regression.
\newblock In \emph{Conference on Learning Theory (COLT)}, 2014.

\bibitem[Rakhlin et~al.(2015{\natexlab{a}})Rakhlin, Sridharan, and Tewari]{rakhlin2015}
Alexander Rakhlin, Karthik Sridharan, and Ambuj Tewari.
\newblock Online learning via sequential complexities.
\newblock \emph{Journal of Machine Learning Research}, 2015{\natexlab{a}}.

\bibitem[Rakhlin et~al.(2015{\natexlab{b}})Rakhlin, Sridharan, and Tewari]{rakhlin2015ptrf}
Alexander Rakhlin, Karthik Sridharan, and Ambuj Tewari.
\newblock Sequential complexities and uniform martingale laws of large numbers.
\newblock \emph{Probability Theory and Related Fields}, 2015{\natexlab{b}}.

\bibitem[Russo and {Van Roy}(2013)]{russo2013}
Daniel Russo and Benjamin {Van Roy}.
\newblock Eluder dimension and the sample complexity of optimistic exploration.
\newblock In \emph{Advances in Neural Information Processing Systems (NeurIPS)}, 2013.

\bibitem[Sinclair et~al.(2019)Sinclair, Banerjee, and Yu]{sinclair2019adaptive}
Sean~R. Sinclair, Siddhartha Banerjee, and Christina~Lee Yu.
\newblock Adaptive discretization for episodic reinforcement learning in metric spaces.
\newblock \emph{Proceedings of the ACM on Measurement and Analysis of Computing Systems}, 2019.

\bibitem[Slivkins(2014)]{slivkins2014contextual}
Aleksandrs Slivkins.
\newblock Contextual bandits with similarity information.
\newblock \emph{Journal of Machine Learning Research}, 2014.

\bibitem[Song et~al.(2023)Song, Zhou, Sekhari, Bagnell, Krishnamurthy, and Sun]{song2023hybrid}
Yuda Song, Yifei Zhou, Ayush Sekhari, J.~Andrew Bagnell, Akshay Krishnamurthy, and Wen Sun.
\newblock Hybrid {RL}: Using both offline and online data can make {RL} efficient.
\newblock In \emph{International Conference on Learning Representations (ICLR)}, 2023.

\bibitem[Szarek(1976)]{szarek1976}
Stanis{\l}aw~J. Szarek.
\newblock On the best constants in the {K}hinchin inequality.
\newblock \emph{Studia Mathematica}, 1976.

\bibitem[Tsybakov(2009)]{tsybakov2009}
Alexandre~B. Tsybakov.
\newblock \emph{Introduction to Nonparametric Estimation}.
\newblock Springer Series in Statistics. Springer, 2009.

\bibitem[Uehara and Sun(2022)]{uehara2022pessimistic}
Masatoshi Uehara and Wen Sun.
\newblock Pessimistic model-based offline reinforcement learning under partial coverage.
\newblock In \emph{International Conference on Learning Representations (ICLR)}, 2022.

\bibitem[Wang et~al.(2020)Wang, Salakhutdinov, and Yang]{wang2020eluder}
Ruosong Wang, Russ~R. Salakhutdinov, and Lin Yang.
\newblock Reinforcement learning with general value function approximation: Provably efficient approach via bounded eluder dimension.
\newblock In \emph{Advances in Neural Information Processing Systems (NeurIPS)}, 2020.

\bibitem[Wang et~al.(2018)Wang, Balakrishnan, and Singh]{wang2018optimization}
Yining Wang, Sivaraman Balakrishnan, and Aarti Singh.
\newblock Optimization of smooth functions with noisy observations: Local minimax rates.
\newblock In \emph{Advances in Neural Information Processing Systems (NeurIPS)}, 2018.

\bibitem[Xie et~al.(2021)Xie, Cheng, Jiang, Mineiro, and Agarwal]{xie2021bellman}
Tengyang Xie, Ching-An Cheng, Nan Jiang, Paul Mineiro, and Alekh Agarwal.
\newblock {Bellman}-consistent pessimism for offline reinforcement learning.
\newblock In \emph{Advances in Neural Information Processing Systems (NeurIPS)}, 2021.

\bibitem[Zhan et~al.(2021)Zhan, Hadad, Hirshberg, and Athey]{zhan2021offpolicy}
Ruohan Zhan, Vitor Hadad, David~A. Hirshberg, and Susan Athey.
\newblock Off-policy evaluation via adaptive weighting with data from contextual bandits.
\newblock In \emph{ACM SIGKDD Conference on Knowledge Discovery and Data Mining (KDD)}, 2021.

\bibitem[Zhan et~al.(2024)Zhan, Ren, Athey, and Zhou]{zhan2023policy}
Ruohan Zhan, Zhimei Ren, Susan Athey, and Zhengyuan Zhou.
\newblock Policy learning with adaptively collected data.
\newblock \emph{Management Science}, 2024.

\bibitem[Zhou(2002)]{zhou2002}
Ding-Xuan Zhou.
\newblock The covering number in learning theory.
\newblock \emph{Journal of Complexity}, 2002.

\bibitem[Ziemann and Tu(2022)]{ziemann2022}
Ingvar Ziemann and Stephen Tu.
\newblock Learning with little mixing.
\newblock In \emph{Advances in Neural Information Processing Systems (NeurIPS)}, 2022.

\end{thebibliography}

\clearpage
\appendix
\crefalias{section}{appendix}
\crefalias{subsection}{subappendix}
\begin{center}
    \Huge\textsc{Appendix}
\end{center}

Throughout, the notation of \cref{sec:setting,sec:regression,sec:coverage} is used.
All inequalities between functions are pointwise.

\section{Proofs for \texorpdfstring{\cref{sec:setting}}{Section \ref{sec:setting}}}\label{app:setting}

\subsection{Well-posedness and fixed points}

\begin{lemma}[Well-posedness, full form]\label{lem:wellposed-full}
Under \cref{asm:model}:
(i) $\calQ_0$ is closed under addition, multiplication by nonnegative scalars and uniform limits.
So it is a closed convex cone in $(\calQ,\|\cdot\|_\infty)$ and a complete metric space, but not a linear subspace unless $S\times A$ is discrete, in which case $\calQ_0=\calQ$.
(ii) For $Q\in\calQ_0$ the maximum defining $V_Q$ is attained, $V_Q$ is USC and so Borel, and $Q$ has a greedy selector.
(iii) $T^\pi$ maps $\calQ$ into $\calQ$, $T^*$ maps $\calQ_0$ into $\calQ_0$, and $T^\pi Q\leq T^*Q$ for $Q\in\calQ_0$, with equality if $\pi$ is a greedy selector of $Q$.
(iv) $T^\pi$ and $T^*$ are monotone $\gamma$-contractions in $\|\cdot\|_\infty$ on $\calQ$ and on $\calQ_0$, respectively.
\end{lemma}

\begin{proofof}{\cref{lem:wellposed-full}}

(i) Sums and nonnegative multiples of USC functions are USC.
If $Q_m\to Q$ uniformly with $Q_m\in\calQ_0$ and $x_j\to x$, then $\limsup_jQ(x_j)\leq\limsup_jQ_m(x_j)+\|Q-Q_m\|_\infty\leq Q(x)+2\|Q-Q_m\|_\infty$ for every $m$, so $Q$ is USC.
As USC functions are Borel, $Q\in\calQ_0$.
A closed subset of a Banach space is complete.
If $S\times A$ is not discrete, $\calQ_0$ is not a linear subspace: for $F=\{x\}$ with $x$ not isolated, $\dsone_F\in\calQ_0$ while $-\dsone_F$ is not USC.
If $S\times A$ is discrete, every function on it is continuous, so $\calQ_0=\calQ$.

(ii) For fixed $s$, $a\mapsto Q(s,a)$ is USC on the compact set $A$ and attains its maximum.
Let $s_j\to s$ and $a_j\in\argmax Q(s_j,\cdot)$.
Choose a subsequence $(j_l)$ with $V_Q(s_{j_l})\to\limsup_jV_Q(s_j)$ and, by compactness, $a_{j_l}\to a$.
Joint USC gives $\limsup_jV_Q(s_j)=\lim_lQ(s_{j_l},a_{j_l})\leq Q(s,a)\leq V_Q(s)$.
A greedy selector is provided by \citet[Prop.~7.33]{bertsekas1978}, applied with $X=S$ (metrizable), $Y=A$ (compact metrizable), $D=S\times A$ (closed) and the lower semicontinuous function $-Q$; it yields a Borel $g:S\to A$ with $-Q(s,g(s))=\min_a(-Q(s,a))$.

(iii) Integrals of bounded Borel functions against Borel kernels are Borel \citep[Prop.~7.29]{bertsekas1978}, so $T^\pi Q\in\calQ$.
For $Q\in\calQ_0$, $V_Q$ is bounded and USC on the metrizable space $S$ (metric $\varrho_S$), so $u_m(s)\defeq\sup_{s'}[V_Q(s')-m\,\varrho_S(s,s')]$ defines bounded $m$-Lipschitz functions with $u_m\downarrow V_Q$ pointwise.
By weak continuity of $P$, each $(s,a)\mapsto\int u_m\dd P(\cdot\mid s,a)$ is continuous, and by monotone convergence $\int V_Q\dd P(\cdot\mid s,a)=\inf_m\int u_m\dd P(\cdot\mid s,a)$, an infimum of continuous functions, which is USC.
Adding the USC function $r$ gives $T^*Q\in\calQ_0$.
Finally, $\int Q(s',a')\pi(\dd a'\mid s')\leq V_Q(s')$, with equality if $\pi=\delta_{g(s')}$ for a greedy selector $g$.

(iv) Monotonicity follows from the positivity of $P^\pi$ and of $P$.
For contraction, $|T^\pi Q-T^\pi Q'|\leq\gamma P^\pi|Q-Q'|\leq\gamma\|Q-Q'\|_\infty$, and $|V_Q-V_{Q'}|\leq\|Q-Q'\|_\infty$ gives the same for $T^*$.
\end{proofof}

\begin{theorem}[Fixed points and optimality, full form]\label{thm:fixedpoints-full}
Under \cref{asm:model}, $T^\pi$ has the unique fixed point $Q^\pi=\sum_{t\geq0}\gamma^t(P^\pi)^tr\in\calQ$ and $T^*$ has a unique fixed point $Q^*\in\calQ_0$.
If $\pi^*$ is a greedy selector of $Q^*$, then $Q^\pi\leq Q^{\pi^*}=Q^*$ for every $\pi\in\Pi$; with $V^\pi(s)\defeq\int Q^\pi(s,a)\,\pi(\dd a\mid s)$ and $V^*\defeq V_{Q^*}$, $V^*=V^{\pi^*}=\max_{\pi\in\Pi}V^\pi$.
\end{theorem}

\begin{proofof}{\cref{thm:fixedpoints-full}}
By \cref{lem:wellposed}(i),(iii),(iv) and Banach's fixed-point theorem on the complete metric spaces $\calQ$ and $\calQ_0$, the fixed points exist and are unique.
The Neumann series $\sum_t\gamma^t(P^\pi)^tr$ converges in $\calQ$ and is a fixed point of $T^\pi$.
For $\pi\in\Pi$, \cref{lem:wellposed}(iii) gives $T^\pi Q^*\leq T^*Q^*=Q^*$, so $(T^\pi)^mQ^*\leq Q^*$ by monotonicity, and $Q^\pi=\lim_m(T^\pi)^mQ^*\leq Q^*$.
If $\pi^*$ is a greedy selector of $Q^*$, then $T^{\pi^*}Q^*=T^*Q^*=Q^*$, so $Q^{\pi^*}=Q^*$.
Finally, $V^\pi(s)=\int Q^\pi(s,a)\pi(\dd a\mid s)\leq\int Q^*(s,a)\pi(\dd a\mid s)\leq V_{Q^*}(s)=Q^*(s,\pi^*(s))=V^{\pi^*}(s)$.
\end{proofof}

\begin{remark}[Policy class]\label{rem:policyclass}
Every maximization in this paper is over a function in $\calQ_0$ (an iterate or $Q^*$), so \cref{lem:wellposed}(ii) provides Borel greedy selectors and all policies lie in $\Pi$.
Analytically or universally measurable policies \citep[Prop.~7.49--7.50]{bertsekas1978} would be needed only for $\epsilon$-optimal selection from functions outside $\calQ_0$, for instance a greedy step from $Q^\pi$ of an arbitrary Borel policy; we never take such a step.
\end{remark}

\subsection{Measurability of the algorithm}

\Cref{alg:fqi} states the procedure of \cref{sec:regression}; the following lemma makes each of its steps measurable.

\makeatletter
\begin{algocf@algorithm}[t]
\caption{FQI with adaptive fresh batches}\label{alg:fqi}
\KwIn{class $\calF$, initial $\widehat\theta_0$, rounds $K$, batch size $n$}
\For{$k=0,\dots,K-1$}{
  choose a behavior rule using the history $\calH_k$\;
  collect $Z_{k,1},\dots,Z_{k,n}$ (\cref{asm:adaptive-batches})\;
  $Y_{k,i}\gets R_{k,i}+\gamma V_{\widehat Q_k}(S'_{k,i})$\;
  $\widehat\theta_{k+1}\gets$ approx.\ $\argmin_\theta\widehat L_k(f_\theta)$\;
}
\KwOut{greedy policy $\widehat\pi_K=g(\widehat\theta_K,\cdot)$}
\end{algocf@algorithm}
\makeatother

\begin{lemma}[Measurability]\label{lem:measurability}
Let \cref{asm:model,asm:class} hold.
(a) $(\theta,s,a)\mapsto f_\theta(s,a)$ is USC, and so Borel, on $\Theta\times S\times A$.
(b) $V(\theta,s)\defeq V_{f_\theta}(s)$ is Borel on $\Theta\times S$, and $|V(\theta,s)-V(\theta',s)|\leq\|f_\theta-f_{\theta'}\|_\infty$.
(c) $(\theta,s,a)\mapsto(T^*f_\theta)(s,a)$ is Borel, and $\|T^*f_\theta-T^*f_{\theta'}\|_\infty\leq\gamma\|f_\theta-f_{\theta'}\|_\infty$.
(d) There is a Borel map $g:\Theta\times S\to A$ such that $g(\theta,\cdot)$ is a greedy selector of $f_\theta$ for every $\theta$.
(e) If $\Theta_0\subseteq\Theta$ is countable and dense and $\Phi:\calQ\to\dsR$ is $\|\cdot\|_\infty$-continuous, then $\sup_{\theta\in\Theta}\Phi(f_\theta)=\sup_{\theta\in\Theta_0}\Phi(f_\theta)$.
(f) For the Borel map $g$ of (d) and $\pi_\theta\defeq g(\theta,\cdot)$, the functions $(\theta,s,a)\mapsto Q^{\pi_\theta}(s,a)$ and $(\theta,s)\mapsto V^{\pi_\theta}(s)$ are Borel.
Consequently, under \cref{asm:adaptive-batches,asm:erm}, the labels $Y_{k,i}$, the losses $\widehat L_k(f_\theta)$ and $L_k(f_\theta)$, the quantities $\Delta_k$ and $\eapp[,k]$, the residual norms $\|e_k\|_{2,\bar\mu_k}$, and $\|V^*-V^{\widehat\pi_K}\|_{1,\rho}$ with $\widehat\pi_K=g(\widehat\theta_K,\cdot)$ are random variables.
\end{lemma}

\begin{proof}

(a) If $(\theta_j,s_j,a_j)\to(\theta,s,a)$, then $f_{\theta_j}(s_j,a_j)\leq f_\theta(s_j,a_j)+\|f_{\theta_j}-f_\theta\|_\infty$, and the right-hand side has $\limsup$ at most $f_\theta(s,a)$ by USC of $f_\theta$ and continuity of $\theta\mapsto f_\theta$.
The product of separable metrizable spaces is separable metrizable, and its Borel $\sigma$-algebra is the product $\sigma$-algebra.
(b) By (a) and the argument of \cref{lem:wellposed}(ii), $V$ is USC on $\Theta\times S$; the Lipschitz bound is immediate.
(c) Apply \citet[Prop.~7.29]{bertsekas1978} to the Borel kernel $(\theta,s,a)\mapsto P(\cdot\mid s,a)$ and the bounded Borel function $(\theta,s')\mapsto V(\theta,s')$; the bound follows from (b).
(d) Apply \citet[Prop.~7.33]{bertsekas1978} with $X=\Theta\times S$ and the lower semicontinuous function $-f_\theta(s,a)$.
(e) The map $\theta\mapsto\Phi(f_\theta)$ is continuous and $\Theta_0$ is dense.
(f) For Borel $u$ on $\Theta\times S\times A$, $(\theta,s,a)\mapsto\int u(\theta,s',g(\theta,s'))P(\dd s'\mid s,a)$ is Borel by \citet[Prop.~7.29]{bertsekas1978}; iterating shows that $(P^{\pi_\theta})^tr$ is Borel in $(\theta,s,a)$ for every $t$, and so are $Q^{\pi_\theta}=\sum_t\gamma^t(P^{\pi_\theta})^tr$ and $V^{\pi_\theta}(s)=Q^{\pi_\theta}(s,g(\theta,s))$.
The consequences follow by composing these Borel maps with the measurable maps $\widehat\theta_k$, $Z_{k,i}$ and the regular conditional distributions $\mu_{k,i}$, and by (e) with a countable dense $\Theta_0$.
\end{proof}

\subsection{Error propagation}

\begin{lemma}[Pointwise propagation, \citealp{munos2007}]\label{lem:errprop}
Let $Q_0,\dots,Q_K\in\calQ_0$, $e_k\defeq Q_{k+1}-T^*Q_k$, and let $\pi^*$ and $\pi_k$ be greedy selectors of $Q^*$ and $Q_k$.
Then
\begin{equation*}
|Q_K-Q^*|\leq\gamma^KA_K|Q_0-Q^*|+\sum_{k=0}^{K-1}\gamma^{K-1-k}A_{K,k}|e_k|,
\end{equation*}
with $A_K\defeq(P^{\pi^*})^K+P^{\pi_{K-1}}\cdots P^{\pi_0}$ and $A_{K,k}\defeq(P^{\pi^*})^{K-1-k}+P^{\pi_{K-1}}\cdots P^{\pi_{k+1}}$, where the empty product is the identity.
\end{lemma}

\begin{proof}

By \cref{lem:wellposed}(iii), $T^*Q_k\geq T^{\pi^*}Q_k$ and $T^*Q^*\geq T^{\pi_k}Q^*$, with $T^*Q_k=T^{\pi_k}Q_k$ and $T^*Q^*=T^{\pi^*}Q^*$.
So
\begin{equation*}
\gamma P^{\pi^*}(Q_k-Q^*)+e_k\leq Q_{k+1}-Q^*=T^*Q_k-T^*Q^*+e_k\leq\gamma P^{\pi_k}(Q_k-Q^*)+e_k .
\end{equation*}
Iterating the lower bound with the positive operator $\gamma P^{\pi^*}$, and the upper bound with $\gamma P^{\pi_k}$, gives
\begin{align*}
Q_K-Q^*&\geq\gamma^K(P^{\pi^*})^K(Q_0-Q^*)+\sum_k\gamma^{K-1-k}(P^{\pi^*})^{K-1-k}e_k,\\
Q_K-Q^*&\leq\gamma^KP^{\pi_{K-1}}\cdots P^{\pi_0}(Q_0-Q^*)+\sum_k\gamma^{K-1-k}P^{\pi_{K-1}}\cdots P^{\pi_{k+1}}e_k .
\end{align*}
Since $x\leq y\leq z$ implies $|y|\leq|x|+|z|$, and $|P^\pi u|\leq P^\pi|u|$, the claim follows.
\end{proof}

\begin{proofof}{\cref{prop:propagation}}

Recall that $\widehat\pi_K=\pi_K$.

\emph{Step 1 ($Q$-loss).} By the fixed-point identities and \cref{lem:wellposed}(iii),
\begin{align*}
Q^*-Q^{\pi_K}&=[T^{\pi^*}Q^*-T^{\pi^*}Q_K]+[T^{\pi^*}Q_K-T^{\pi_K}Q_K]\\
&\quad+[T^{\pi_K}Q_K-T^{\pi_K}Q^*]+[T^{\pi_K}Q^*-T^{\pi_K}Q^{\pi_K}]\\
&\leq\gamma(P^{\pi^*}+P^{\pi_K})|Q_K-Q^*|+\gamma P^{\pi_K}(Q^*-Q^{\pi_K}),
\end{align*}
because the second bracket is nonpositive ($T^{\pi^*}Q_K\leq T^*Q_K=T^{\pi_K}Q_K$).
Iterating and using $(\gamma P^{\pi_K})^H(Q^*-Q^{\pi_K})\to0$ uniformly,
$Q^*-Q^{\pi_K}\leq\sum_{h\geq0}\gamma^{h+1}(P^{\pi_K})^h(P^{\pi^*}+P^{\pi_K})|Q_K-Q^*|$.

\emph{Step 2 ($V$-loss).} For every $s$,
\begin{align*}
V^*(s)-V^{\pi_K}(s)&=[Q^*-Q_K](s,\pi^*(s))+\underbrace{[Q_K(s,\pi^*(s))-Q_K(s,\pi_K(s))]}_{\leq0}\\
&\quad+[Q_K-Q^*](s,\pi_K(s))+[Q^*-Q^{\pi_K}](s,\pi_K(s)).
\end{align*}
Since $V^*\geq V^{\pi_K}$ (\cref{thm:fixedpoints}), integrating against $\rho$ and using Step~1 gives
\begin{align*}
\|V^*-V^{\pi_K}\|_{1,\rho}&\leq\sum_{M\in\calM}c_M\int|Q_K-Q^*|\,\dd M,\\
\calM&\defeq\{\eta^*,\eta_K\}\cup\{\eta_K(P^{\pi_K})^hP^\pi:h\geq0,\ \pi\in\{\pi^*,\pi_K\}\},
\end{align*}
with $\eta^*\defeq\rho\otimes\delta_{\pi^*}$, $\eta_K\defeq\rho\otimes\delta_{\pi_K}$, $c_{\eta^*}=c_{\eta_K}=1$ and $c_{\eta_K(P^{\pi_K})^hP^\pi}=\gamma^{h+1}$.
We read $\calM$ as a family indexed by these expressions, so measures that coincide, such as $\eta^*$ and $\eta_K$ when $\pi_K=\pi^*$, count once for each index, each with its own weight.
The weights sum to $\sum_Mc_M=2+2\gamma/(1-\gamma)=2/(1-\gamma)$.

\emph{Step 3 (substitution).} By \cref{lem:errprop} and positivity,
\begin{equation*}
\int|Q_K-Q^*|\,\dd M\leq\gamma^K\int A_K|Q_0-Q^*|\,\dd M+\sum_k\gamma^{K-1-k}\int A_{K,k}|e_k|\,\dd M,
\end{equation*}
and $\int A_K|Q_0-Q^*|\,\dd M\leq2\|Q_0-Q^*\|_\infty$.
Define
\begin{equation}\label{eq:Xi}
\Xi_k\defeq\frac{1-\gamma}{4}\sum_{M\in\calM}c_M\Bigl(M(P^{\pi^*})^{K-1-k}+MP^{\pi_{K-1}}\cdots P^{\pi_{k+1}}\Bigr).
\end{equation}
Each summand is a probability measure and $\sum_M2c_M=4/(1-\gamma)$, so $\Xi_k\in\Delta(S\times A)$ and $\sum_Mc_M\int A_{K,k}|e_k|\,\dd M=\frac4{1-\gamma}\int|e_k|\dd\Xi_k$.
Summing gives \eqref{eq:propagation}.

\emph{Structure.} Every summand in \eqref{eq:Xi} is obtained from $\rho$ by alternately choosing actions with deterministic Borel selectors and applying $P$.
It is therefore the step-$t$ marginal of a deterministic nonstationary Markov policy started from $\rho$, of the form $\nu_S\otimes\delta_g$, where $g$ is the last selector applied.
That selector is $\pi^*$ for the first product with $K-1-k\geq1$, $\pi_{k+1}$ for the second product with $k\leq K-2$, and the last selector of $M$ ($\pi^*$ or $\pi_K$) otherwise.
Its state marginal is $\nu_S=\rho$ if no transition has been applied, and $\nu_S=\eta P$ for some $\eta$ otherwise.
\end{proofof}

\section{Proofs for \texorpdfstring{\cref{sec:regression}}{Section \ref{sec:regression}}}\label{app:regression}

\subsection{Sequential symmetrization with a frozen parameter}

Let $Z\defeq S\times A\times[-R_{\max},R_{\max}]\times S$.
A $Z$-valued tree of depth $n$ is a sequence $\mbz=(z_1,\dots,z_n)$ of maps $z_t:\{\pm1\}^{t-1}\to Z$, and $\epsilon=(\epsilon_1,\dots,\epsilon_n)$ denotes i.i.d.\ Rademacher signs \citep{rakhlin2015}.

\begin{definition}[Sequential Rademacher complexity]\label{def:seqrad}
For a class $\calG$ of functions $Z\to\dsR$, $\Rseq_n(\calG)\defeq\sup_{\mbz}\E_\epsilon\sup_{g\in\calG}\bigl|\sum_{t=1}^n\epsilon_t\,g(z_t(\epsilon_{1:t-1}))\bigr|$.
For $\theta'\in\Theta$ let $\calL_{\theta'}\defeq\{z\mapsto(f_\theta(s,a)-r-\gamma V_{f_{\theta'}}(s'))^2:\theta\in\Theta\}$, $\overline\Rseq_n\defeq\sup_{\theta'\in\Theta}\Rseq_n(\calL_{\theta'})$ and $\alpha_n\defeq2\overline\Rseq_n/n$.
\end{definition}

Restricted to trees whose nodes do not depend on the signs, the supremum defining $\Rseq_n$ is the empirical (unnormalized, absolute-value) Rademacher complexity, so $\Rseq_n$ dominates its i.i.d.\ counterpart.
The class $\calL_{\theta'}$ is indexed by the regression function only: the target parameter $\theta'$ is fixed.
The quantities $L_k$ and $\Delta_k$ are those of \cref{subsec:residual}; for a fixed $f$, $L_k(f)=\frac1n\sum_{i=1}^n\Ex{(f(S_{k,i},A_{k,i})-Y_{k,i})^2\mid\calH_{k,i-1}}$ almost surely.

\begin{theorem}[Uniform deviation under adaptive data]\label{thm:deviation}
Under \cref{asm:model,asm:class,asm:adaptive-batches}, $\Ex{\Delta_k}\leq\alpha_n$ for every $k$.
Consequently, $\Pr(\max_{k<K}\Delta_k>K\alpha_n/\delta)\leq\delta$ for every $\delta\in(0,1)$.
\end{theorem}

\emph{Idea of proof.} Conditionally on $\calH_k$ the target $\widehat Q_k$ is fixed, so $\Delta_k$ is the supremum of a martingale-difference process indexed by $\theta$ over the loss class $\calL_{\widehat\theta_k}$.
Decoupled tangent sequences and the sequential symmetrization of \citet[Thm.~2 and Lemma~9]{rakhlin2015ptrf}, extended to a general filtration and a target parameter that is measurable at time $0$, then give $\Ex{\Delta_k}\leq\alpha_n$.

\begin{lemma}\label{lem:symmetrization}
Let $\calG_0\subseteq\calG_1\subseteq\dots\subseteq\calG_n$ be $\sigma$-fields on $(\Omega,\mathfrak F,\Pr)$ and let $Z_i$ be $\calG_i$-measurable random elements of a Borel space $Z$.
Let $\xi$ be a $\calG_0$-measurable random element of a Borel space $X$, and let $p_i$ be a regular conditional distribution of $Z_i$ given $\calG_{i-1}$.
Let $\Theta_0$ be countable, and let $(x,\theta,z)\mapsto\ell^x_\theta(z)$ be Borel with $|\ell|\leq M$.
Then
\begin{equation*}
\E\sup_{\theta\in\Theta_0}\Bigl|\sum_{i=1}^n\Bigl(\ell^\xi_\theta(Z_i)-\int\ell^\xi_\theta\,\dd p_i\Bigr)\Bigr|\leq2\sup_{x\in X}\Rseq_n\bigl(\{\ell^x_\theta:\theta\in\Theta_0\}\bigr).
\end{equation*}
\end{lemma}

\begin{proof}

We follow \citet[Thm.~2 and Lemma~9]{rakhlin2015ptrf}, adding a parameter that is fixed at time $0$ and a general filtration.

\emph{Step 1 (decoupling).} On $\Omega'\defeq\Omega\times Z^n$ with $\Pr'(\dd\omega,\dd z')\defeq\Pr(\dd\omega)\bigotimes_ip_i(\omega,\dd z'_i)$, let $Z'_i$ be the coordinates.
Conditionally on $\mathfrak F$ the $Z'_i$ are independent with laws $p_i$, so $\E'[\ell^\xi_\theta(Z'_i)\mid\mathfrak F]=\int\ell^\xi_\theta\dd p_i$ because $\xi$ is $\mathfrak F$-measurable.
So the sum in the claim equals $\E'[\sum_i(\ell^\xi_\theta(Z_i)-\ell^\xi_\theta(Z'_i))\mid\mathfrak F]$.
Since $|\E'[X_\theta\mid\mathfrak F]|\leq\E'[\sup_{\theta'}|X_{\theta'}|\mid\mathfrak F]$ for each of the countably many $\theta$, the left-hand side is at most $\E'\sup_\theta|\sum_i(\ell^\xi_\theta(Z_i)-\ell^\xi_\theta(Z'_i))|$.

\emph{Step 2 (exchangeability).} Let $\calK_{i-1}\defeq\sigma(\xi,Z_{1:i-1},Z'_{1:i-1})$ and $\calG'\defeq\calG_{i-1}\vee\sigma(Z'_{1:i-1})\supseteq\calK_{i-1}$.
Given $\mathfrak F$, the law of $Z'_{1:i-1}$ is $\bigotimes_{j<i}p_j$, which is $\calG_{i-1}$-measurable, so $Z'_{1:i-1}$ is conditionally independent of $\mathfrak F$ given $\calG_{i-1}$.
It follows that $\Law(Z_i\mid\calG')=\Law(Z_i\mid\calG_{i-1})=p_i$.
The variable $Z'_i$ satisfies $\Law(Z'_i\mid\calG'\vee\sigma(Z_i))=p_i$, because it is distributed as $p_i$ given $\mathfrak F\vee\sigma(Z'_{1:i-1})$.
So $(Z_i,Z'_i)$ has conditional law $p_i\otimes p_i$ given $\calG'$, and is exchangeable conditionally on the smaller $\sigma$-field $\calK_{i-1}$.

\emph{Step 3 (backward induction).} For $0\leq j\leq n$ and $u=(z_{1:j},z'_{1:j})$ let
\begin{equation*}
\Psi_j(x,u)\defeq\sup_{w_{j+1},w'_{j+1}}\E_{\epsilon_{j+1}}\cdots\sup_{w_n,w'_n}\E_{\epsilon_n}\sup_{\theta\in\Theta_0}\Bigl|\sum_{i\leq j}(\ell^x_\theta(z_i)-\ell^x_\theta(z'_i))+\sum_{i>j}\epsilon_i(\ell^x_\theta(w_i)-\ell^x_\theta(w'_i))\Bigr|,
\end{equation*}
so that $\Psi_n$ is Borel.
Suprema over a Borel space and finite averages preserve upper semianalyticity \citep[Prop.~7.47 and Lemma~7.30(4)]{bertsekas1978}, so each $\Psi_j$ is upper semianalytic and therefore universally measurable, and the conditional expectations below are defined (on the completed space).
Exchanging $(z_j,z'_j)$ in $\Psi_j$ multiplies the $j$-th term by $-1$.
By Step~2,
\begin{align*}
\E'[\Psi_j(\xi,U_j)\mid\calK_{j-1}]&=\E'\Bigl[\tfrac12\bigl(\Psi_j(\xi,U_{j-1},Z_j,Z'_j)+\Psi_j(\xi,U_{j-1},Z'_j,Z_j)\bigr)\Bigm|\calK_{j-1}\Bigr]\\
&\leq\Psi_{j-1}(\xi,U_{j-1}),
\end{align*}
where $U_j\defeq(Z_{1:j},Z'_{1:j})$.
Taking expectations and chaining from $j=n$ down to $j=1$ yields $\E'\Psi_n(\xi,U_n)\leq\sup_x\Psi_0(x)$.

\emph{Step 4 (trees).} The alternating expression $\Psi_0(x)$ is the supremum over pairs of $Z$-valued trees $(\mbw,\mbw')$ of $\E_\epsilon\sup_\theta|\sum_i\epsilon_i(\ell^x_\theta(w_i(\epsilon))-\ell^x_\theta(w'_i(\epsilon)))|$: choose near-optimal nodes for each sign prefix.
By the triangle inequality this is at most $2\Rseq_n(\{\ell^x_\theta\})$.
\end{proof}

\begin{proofof}{\cref{thm:deviation}}

Fix $k$ and apply \cref{lem:symmetrization} with $\calG_i\defeq\calH_{k,i}$, $Z_i\defeq Z_{k,i}$, $\xi\defeq\widehat\theta_k$, a countable dense $\Theta_0\subseteq\Theta$, and $\ell^x_\theta(z)\defeq(f_\theta(s,a)-r-\gamma V(x,s'))^2$.
This function is Borel by \cref{lem:measurability} and bounded by $\Bres^2$.
Since $\widehat\theta_k$ is $\calH_{k,i-1}$-measurable, $\int\ell^{\widehat\theta_k}_\theta\dd p_i=\Ex{(f_\theta(S_{k,i},A_{k,i})-Y_{k,i})^2\mid\calH_{k,i-1}}$.
So $n\,(L_k(f_\theta)-\widehat L_k(f_\theta))=\sum_i(\int\ell\,\dd p_i-\ell(Z_{k,i}))$, and by \cref{lem:measurability}(e) the supremum over $\Theta$ equals the supremum over $\Theta_0$.
\Cref{lem:symmetrization} gives $n\Ex{\Delta_k}\leq2\sup_{\theta'}\Rseq_n(\calL_{\theta'})$.
The tail bound follows from Markov's inequality and a union bound over the $K$ rounds.
\end{proofof}

\subsection{Hilbert-space classes}

 For Hilbert-space classes the deviation is the norm of a vector-valued martingale and has sub-Gaussian tails:

\begin{proposition}[Hilbert-space classes]\label{prop:hilbert}
Let \cref{asm:model,asm:adaptive-batches} hold, and let $\calF=\{\langle w,\psi(\cdot)\rangle_{\dsH}:w\in W\}$, where $\dsH$ is a separable Hilbert space, $\psi:S\times A\to\dsH$ is continuous with $\sup\|\psi\|\leq\kappa$, and $W$ is a closed subset of the ball of radius $W_0$ with $W_0\kappa\leq B$ and $B\geq R_{\max}/(1-\gamma)$.
Then $\overline\Rseq_n\leq\Bres^2\sqrt n$, so $\alpha_n\leq2\Bres^2/\sqrt n$.
For all $k$ and $\delta\in(0,1)$, we also have $\Pr\bigl(\Delta_k>2\Bres^2\sqrt{2\log(6/\delta)/n}\bigr)\leq\delta$.
\end{proposition}

This covers linear classes ($\dsH=\dsR^d$) and RKHS balls ($\psi(x)=k(x,\cdot)$), which satisfy \cref{asm:class} with every $B\geq\max\{W_0\kappa,R_{\max}/(1-\gamma)\}$.
The squared loss is linear in $(w\otimes w,w,1)$, so no contraction inequality is needed, and the tail bound is Pinelis' inequality for Hilbert-space martingales \citep[Thm.~3.5]{pinelis1994}.
General classes, including neural nets, go through the sequential contraction inequality, at a cost of $O(\log^{3/2}n)$ per Lipschitz composition \citep[Cor.~5, Prop.~15]{rakhlin2015} (\cref{app:classes}).

\begin{proofof}{\cref{prop:hilbert}}

The class satisfies \cref{asm:class} with $\Theta=W$, since $\|f_w-f_{w'}\|_\infty\leq\kappa\|w-w'\|$ and $f_w$ is continuous.
Write $x=(s,a)$ and $y(z)=r+\gamma V_{f_{\theta'}}(s')$, so that $|y|\leq R_{\max}+\gamma B$ and $B+R_{\max}+\gamma B=\Bres$.
With $\otimes$ denoting the Hilbert--Schmidt tensor product,
\begin{equation*}
(\langle w,\psi(x)\rangle-y)^2=\langle w\otimes w,\psi(x)\otimes\psi(x)\rangle-2y\langle w,\psi(x)\rangle+y^2,\qquad\|\psi(x)\otimes\psi(x)\|=\|\psi(x)\|^2\leq\kappa^2 .
\end{equation*}

\emph{Complexity.} For a tree $\mbz$ with $x_t=x_t(\epsilon)$ and $y_t=y(z_t(\epsilon))$,
\begin{equation*}
\sup_{w\in W}\Bigl|\sum_t\epsilon_t\ell_w(z_t)\Bigr|\leq W_0^2\Bigl\|\sum_t\epsilon_t\psi(x_t)\otimes\psi(x_t)\Bigr\|+2W_0\Bigl\|\sum_t\epsilon_ty_t\psi(x_t)\Bigr\|+\Bigl|\sum_t\epsilon_ty_t^2\Bigr| .
\end{equation*}
Each vector $v_t$ multiplying $\epsilon_t$ depends only on $\epsilon_{1:t-1}$, so $\E\|\sum_t\epsilon_tv_t\|\leq(\sum_t\E\|v_t\|^2)^{1/2}$ (the cross terms vanish).
So $\Rseq_n(\calL_{\theta'})\leq\sqrt n\,(W_0\kappa+R_{\max}+\gamma B)^2\leq\Bres^2\sqrt n$, since $W_0\kappa\leq B$.

\emph{Tail.} With $\psi_i\defeq\psi(S_{k,i},A_{k,i})$ and $Y_i\defeq Y_{k,i}$, define the martingale differences (with respect to $(\calH_{k,i})_i$) $d^{(1)}_i\defeq\Ex{\psi_i\otimes\psi_i\mid\calH_{k,i-1}}-\psi_i\otimes\psi_i$, $d^{(2)}_i\defeq\Ex{Y_i\psi_i\mid\calH_{k,i-1}}-Y_i\psi_i$ and $d^{(3)}_i\defeq\Ex{Y_i^2\mid\calH_{k,i-1}}-Y_i^2$.
These are conditional expectations of Bochner-integrable variables in the separable Hilbert spaces $\dsH\otimes\dsH$, $\dsH$ and $\dsR$, and $Y_i$ involves $\widehat\theta_k$, which is $\calH_{k,0}$-measurable.
Then $n(L_k(f_w)-\widehat L_k(f_w))=\langle w\otimes w,\sum_id^{(1)}_i\rangle-2\langle w,\sum_id^{(2)}_i\rangle+\sum_id^{(3)}_i$, so
\begin{equation*}
n\Delta_k\leq W_0^2\Bigl\|\sum_id^{(1)}_i\Bigr\|+2W_0\Bigl\|\sum_id^{(2)}_i\Bigr\|+\Bigl|\sum_id^{(3)}_i\Bigr| .
\end{equation*}
Here $\|d^{(1)}_i\|\leq2\kappa^2$, $\|d^{(2)}_i\|\leq2\kappa(R_{\max}+\gamma B)$ and $|d^{(3)}_i|\leq2(R_{\max}+\gamma B)^2$.
Pinelis' inequality for martingales in a Hilbert space \citep[Thm.~3.5, with $D=1$]{pinelis1994} states that $\Pr(\|\sum_{i\leq n}d_i\|\geq u)\leq2\exp(-u^2/(2nb^2))$ when $\|d_i\|\leq b$.
Applying it to each of the three sums with confidence $\delta/3$ and taking a union bound gives, with probability at least $1-\delta$,
$n\Delta_k\leq2\sqrt{2n\log(6/\delta)}\,(W_0\kappa+R_{\max}+\gamma B)^2\leq2\Bres^2\sqrt{2n\log(6/\delta)}$.
\end{proofof}

\subsection{The per-round residual}

The conditional-mean identity \eqref{eq:condmean} is the key step.
It gives $L_k(f)=\|f-h_k\|_{2,\bar\mu_k}^2+v_k$ with $v_k$ independent of $f$; combined with \cref{asm:erm} this bounds $\|e_k\|_{2,\bar\mu_k}^2$ by $\eapp[,k]^2+2\Delta_k+\eopt[,k]^2$.
For fresh i.i.d.\ batches from a fixed $\mu$, $\bar\mu_k=\mu$ and $\Rseq_n$ can be replaced by the expected i.i.d.\ Rademacher complexity.

The regular conditional distributions $p_{k,i}$ in the definition of $L_k$ (\cref{subsec:residual}) exist because $(S_{k,i},A_{k,i},Y_{k,i})$ takes values in a Borel space.
Since $|V_{f_\theta}|\leq B$ for every $\theta$, we have $|Y_{k,i}|\leq R_{\max}+\gamma B$, so a modification on a null set makes $p_{k,i}$ concentrated on $S\times A\times[-R_{\max}-\gamma B,R_{\max}+\gamma B]$.
For $|u|,|v|\leq B$ and $|y|\leq R_{\max}+\gamma B$ we have $|(u-y)^2-(v-y)^2|=|u-v|\,|u+v-2y|\leq2\Bres|u-v|$, which gives the Lipschitz bound on $L_k$ stated there.
So $\theta\mapsto L_k(f_\theta)$ is continuous for every outcome and measurable for every $\theta$, and therefore jointly measurable, and $L_k(\widehat Q_{k+1})$ is a random variable.

\begin{proofof}{\cref{thm:residual}}

Fix $k$ and write $\calG\defeq\calH_{k,i-1}\vee\sigma(S_{k,i},A_{k,i})$.
Since $\widehat\theta_k$ is $\calG$-measurable, $V$ is jointly Borel (\cref{lem:measurability}) and $S'_{k,i}\mid\calG\sim P(\cdot\mid S_{k,i},A_{k,i})$, we have $\Ex{V(\widehat\theta_k,S'_{k,i})\mid\calG}=\int V(\widehat\theta_k,s')P(\dd s'\mid S_{k,i},A_{k,i})$.
With $\Ex{R_{k,i}\mid\calG}=r(S_{k,i},A_{k,i})$ this gives \eqref{eq:condmean}.
Consequently, for every $\theta$,
\begin{equation*}
\Ex{(f_\theta(S_{k,i},A_{k,i})-Y_{k,i})^2\mid\calG}=(f_\theta-h_k)^2(S_{k,i},A_{k,i})+\Ex{(Y_{k,i}-h_k(S_{k,i},A_{k,i}))^2\mid\calG},
\end{equation*}
and conditioning on $\calH_{k,i-1}$ and averaging over $i$ gives $L_k(f_\theta)=\|f_\theta-h_k\|_{2,\bar\mu_k}^2+v_k$ with $v_k$ independent of $\theta$.
This holds almost surely for all $\theta$ simultaneously: first for a countable dense subset of $\Theta$, then because both sides are continuous in $\theta$ for every outcome (\cref{asm:class} and the definition of $L_k$ through regular conditional distributions in \cref{subsec:residual}).
So it holds for $\theta=\widehat\theta_{k+1}$.
For every $\theta$, using \cref{asm:erm} and twice the definition of $\Delta_k$,
\begin{align*}
\|e_k\|_{2,\bar\mu_k}^2=L_k(\widehat Q_{k+1})-v_k&\leq\widehat L_k(\widehat Q_{k+1})+\Delta_k-v_k\leq\widehat L_k(f_\theta)+\eopt[,k]^2+\Delta_k-v_k\\
&\leq\|f_\theta-h_k\|_{2,\bar\mu_k}^2+2\Delta_k+\eopt[,k]^2 .
\end{align*}
Taking the infimum over $\theta$ and using $\sqrt{a^2+b+c^2}\leq a+\sqrt b+c$ proves \eqref{eq:residual}.
The expectation bound follows from $\E\sqrt{2\Delta_k}\leq\sqrt{2\E\Delta_k}\leq\sqrt{2\alpha_n}$.
The high-probability statements follow from \cref{thm:deviation}, and from \cref{prop:hilbert} with confidence $\delta/K$ in each round and a union bound.
\end{proofof}

\subsection{\texorpdfstring{Fast rates for classes of finite metric entropy}{Fast rates for classes of finite metric entropy}}\label{app:fast}

\emph{Why the global route stops at $n^{-1/4}$.} On the tree with all nodes equal to $z$ the losses are nonnegative, so $\Rseq_n(\calL_{\theta'})\geq\sup_{\ell\in\calL_{\theta'}}\ell(z)\,\E|\sum_t\epsilon_t|\geq\sup_{\ell\in\calL_{\theta'}}\ell(z)\sqrt{n/2}$ (Khinchin's inequality for $p=1$ with the sharp constant $1/\sqrt2$ of \citet{szarek1976}); so $\alpha_n\gtrsim n^{-1/2}$ as soon as some loss is nonzero somewhere.

\emph{Related arguments.} This is the argument of \citet[Lemma~1 and Prop.~2]{russo2013} for least squares on adaptively collected data, and \citet[Lemma~39]{jin2021bellman} compare empirical and conditional Bellman errors over a cover in a similar way.
Both assume a well-specified class (realizability and completeness, respectively), while here misspecification enters through $\eapp[,k]$.
For a finite realizable class, \citet[Lemmas~4.1 and 4.2]{agarwal2012contextual} give the same argument with Freedman's inequality.
\citet{ziemann2022} prove fast rates on dependent data under realizability and a hypercontractivity condition on the trajectory, neither of which is needed here.

\begin{proposition}[Fast rates at a general scale]\label{prop:fast-entropy-scale}
Let \cref{asm:model,asm:class,asm:adaptive-batches,asm:erm} hold, and let $N_\infty(\epsilon)$ be as in \cref{thm:fast-entropy}.
For every $\epsilon\in(0,\Bres]$ with $N_\infty(\epsilon)<\infty$ and every $\delta\in(0,1)$, with probability at least $1-\delta$, simultaneously for all $k<K$,
\begin{equation}\label{eq:fast-entropy-scale}
\|e_k\|_{2,\bar\mu_k}^2\leq6\,\eapp[,k]^2+3\,\eopt[,k]^2+60\,\Bres^2\,\frac{\log N_\infty(\epsilon)+\log(6K/\delta)}n+53\,\Bres\,\epsilon .
\end{equation}
\end{proposition}
The proof is part of the proof of \cref{thm:fast-entropy} below.
If $\log N_\infty(\epsilon)$ is of order $\epsilon^{-q}$ with $q>0$, then \eqref{eq:fast-entropy-scale} with $\epsilon$ of order $n^{-1/(1+q)}$ gives the rate $n^{-1/(2(1+q))}$, which improves on $n^{-1/4}$ only for $q<1$.

The first tool compares the empirical and the design norm of a single function that is fixed at time $0$, under an adaptive design.

\begin{lemma}[Relative deviations under an adaptive design]\label{lem:reldev}
Let $\calG_0\subseteq\dots\subseteq\calG_n$ be $\sigma$-fields, let $X_i$ be a $\calG_i$-measurable random element of $S\times A$ with regular conditional distribution $\mu_i$ given $\calG_{i-1}$, and let $\bar\mu\defeq\frac1n\sum_i\mu_i$ and $\|g\|_n^2\defeq\frac1n\sum_ig(X_i)^2$.
Let $g$ be a random function on $S\times A$, measurable with respect to $\calG_0$ and the Borel $\sigma$-algebra, with $g^2\leq b$ for a constant $b>0$.
For every $\delta\in(0,1)$, each of the following two bounds holds with probability at least $1-\delta$:
\begin{equation*}
\|g\|_{2,\bar\mu}^2\leq\tfrac32\|g\|_n^2+\frac{3b\log(1/\delta)}n,\qquad\|g\|_n^2\leq1.3\,\|g\|_{2,\bar\mu}^2+\frac{2b\log(1/\delta)}n.
\end{equation*}
\end{lemma}

\begin{proof}

Let $q_i\defeq\int g^2\dd\mu_i$.
Since $g$ depends on $\omega$ only through $\calG_0$, $q_i=\Ex{g(X_i)^2\mid\calG_{i-1}}$ almost surely, and $\|g\|^2_{2,\bar\mu}=\frac1n\sum_iq_i$.
Fix a sign and let $D_i\defeq\pm(q_i-g(X_i)^2)$.
Then $\Ex{D_i\mid\calG_{i-1}}=0$, $D_i\leq b$ because $0\leq g^2\leq b$, and $\Ex{D_i^2\mid\calG_{i-1}}\leq\Ex{g(X_i)^4\mid\calG_{i-1}}\leq bq_i$.
Since $x\mapsto(e^x-1-x)/x^2$ is nondecreasing, $e^x\leq1+x+4c_1x^2$ for $x\leq\frac12$, where $c_1\defeq e^{1/2}-\frac32<0.15$.
With $\lambda\defeq1/(2b)$ this gives $\Ex{e^{\lambda D_i}\mid\calG_{i-1}}\leq1+4c_1\lambda^2bq_i\leq\exp(c_1q_i/b)$, so $M_j\defeq\exp\bigl(\sum_{i\leq j}(\lambda D_i-c_1q_i/b)\bigr)$ is a nonnegative supermartingale with $M_0=1$ \citep{freedman1975tail}.
By Markov's inequality $\Pr(M_n\geq1/\delta)\leq\delta$, and on the complementary event
\begin{equation*}
\pm\bigl(\|g\|^2_{2,\bar\mu}-\|g\|_n^2\bigr)=\frac1n\sum_iD_i\leq2c_1\|g\|^2_{2,\bar\mu}+\frac{2b\log(1/\delta)}n .
\end{equation*}
As $2c_1<0.3$, the sign $+$ gives $\|g\|^2_{2,\bar\mu}\leq\frac1{0.7}\bigl(\|g\|_n^2+2b\log(1/\delta)/n\bigr)\leq\frac32\|g\|_n^2+3b\log(1/\delta)/n$, and the sign $-$ gives $\|g\|_n^2\leq1.3\|g\|^2_{2,\bar\mu}+2b\log(1/\delta)/n$.
\end{proof}

\begin{proofof}{\cref{thm:fast-entropy}}

Fix $k$ and $\epsilon\in(0,\Bres]$ with $N_\infty(\epsilon)<\infty$, and let $\Theta_\epsilon$ be a set as in the statement with $N_\infty(\epsilon)$ elements.
We write $X_i\defeq(S_{k,i},A_{k,i})$, $h\defeq h_k$ and $\xi_i\defeq Y_{k,i}-h(X_i)$, and $\|g\|_n^2\defeq\frac1n\sum_ig(X_i)^2$ as in \cref{lem:reldev}.
Since $\|V_{f_\theta}\|_\infty\leq B$ for every $\theta$, we have $|Y_{k,i}|\leq R_{\max}+\gamma B$ and $|h|\leq R_{\max}+\gamma B$, so $|\xi_i|\leq2(R_{\max}+\gamma B)$ and $|f_\theta-h|\leq\Bres$.
As $\widehat\theta_k$ is $\calH_{k,0}$-measurable, \cref{lem:measurability}(c) shows that $h$ is measurable with respect to $\calH_{k,0}$ and the Borel $\sigma$-algebra of $S\times A$, and so is $f_\theta-h$ for every fixed $\theta$.
We call such functions \emph{fixed at time $0$}.

\emph{Noise of a single function.} Let $\calJ_{i-1}\defeq\calH_{k,i-1}\vee\sigma(X_i)$ for $1\leq i\leq n$ and $\calJ_n\defeq\calH_{k,n}$.
This is a filtration, $\xi_i$ is $\calJ_i$-measurable, and $g(X_i)$ is $\calJ_{i-1}$-measurable for every $g$ fixed at time $0$.
By \eqref{eq:condmean}, $\Ex{\xi_i\mid\calJ_{i-1}}=0$, and given $\calJ_{i-1}$, $\xi_i$ lies in the interval $[-R_{\max}-\gamma B-h(X_i),R_{\max}+\gamma B-h(X_i)]$ of length $2(R_{\max}+\gamma B)$, whose end points are $\calJ_{i-1}$-measurable.
Let $g$ be bounded and fixed at time $0$, and let $s\in\{-1,1\}$ and $\lambda,u>0$.
The conditional form of Hoeffding's lemma gives $\Ex{\exp(s\lambda g(X_i)\xi_i)\mid\calJ_{i-1}}\leq\exp(\lambda^2(R_{\max}+\gamma B)^2g(X_i)^2/2)$, so $M_j\defeq\exp\bigl(\sum_{i\leq j}(s\lambda\xi_ig(X_i)-\lambda^2(R_{\max}+\gamma B)^2g(X_i)^2/2)\bigr)$ is a nonnegative supermartingale with respect to $(\calJ_j)_{j=0}^n$ with $M_0=1$.
Markov's inequality gives $\Pr(M_n\geq e^u)\leq e^{-u}$, that is,
\begin{equation}\label{eq:offset-one}
\Pr\Bigl(s\sum_{i=1}^n\xi_ig(X_i)\geq\frac{\lambda (R_{\max}+\gamma B)^2}2\sum_{i=1}^ng(X_i)^2+\frac u\lambda\Bigr)\leq e^{-u}.
\end{equation}
This is the one-function step of \citet[Lemma~1]{russo2013}, here with a general filtration and a function that is fixed at time $0$.

\emph{Design norm of a single function.} We apply \cref{lem:reldev} with $\calG_i\defeq\calH_{k,i}$, $\mu_i\defeq\mu_{k,i}$ and $b\defeq\Bres^2$, so that $\bar\mu=\bar\mu_k$.
For every $g$ fixed at time $0$ with $|g|\leq\Bres$ and every $\delta'\in(0,1)$, each of the following two bounds then holds with probability at least $1-\delta'$:
\begin{equation}\label{eq:isometry-one}
\|g\|_{2,\bar\mu_k}^2\leq\tfrac32\|g\|_n^2+\frac{3\Bres^2\log(1/\delta')}n,\qquad\|g\|_n^2\leq1.3\,\|g\|_{2,\bar\mu_k}^2+\frac{2\Bres^2\log(1/\delta')}n.
\end{equation}

\emph{Three events.} Let $\ell_0\defeq\log N_\infty(\epsilon)+\log(6K/\delta)$.
Let $E_1$ be the event that $s\sum_i\xi_i(f_{\theta^\circ}-h)(X_i)\leq\frac18\sum_i(f_{\theta^\circ}-h)(X_i)^2+4(R_{\max}+\gamma B)^2\ell_0$ for all $\theta^\circ\in\Theta_\epsilon$ and $s\in\{-1,1\}$.
The set $\Theta_\epsilon$ does not depend on the data, so every $f_{\theta^\circ}-h$ is fixed at time $0$.
If $R_{\max}+\gamma B=0$, then $\xi_i=0$ for every $i$ and $E_1$ holds trivially.
Otherwise, \eqref{eq:offset-one} with $\lambda=1/(4(R_{\max}+\gamma B)^2)$ and $u=\ell_0=\log(6KN_\infty(\epsilon)/\delta)$, together with a union bound over the $2N_\infty(\epsilon)$ pairs $(\theta^\circ,s)$, gives $\Pr(E_1^c)\leq\delta/(3K)$.
Let $E_2$ and $E_3$ be the events that, for all $\theta^\circ\in\Theta_\epsilon$, the first and the second bound of \eqref{eq:isometry-one}, respectively, hold for $g=f_{\theta^\circ}-h$ with $\log(1/\delta')$ replaced by $\ell_0$.
Since $\log(3KN_\infty(\epsilon)/\delta)\leq\ell_0$, \eqref{eq:isometry-one} with $\delta'=\delta/(3KN_\infty(\epsilon))$ and a union bound over $\Theta_\epsilon$ give $\Pr(E_2^c)\leq\delta/(3K)$ and $\Pr(E_3^c)\leq\delta/(3K)$.
So $E_1\cap E_2\cap E_3$ has probability at least $1-\delta/K$, and we work on this event from now on.

\emph{Basic inequality.} Fix $\theta\in\Theta$, and choose $\widehat\theta^\circ,\theta^{\circ\circ}\in\Theta_\epsilon$ with $\|\widehat Q_{k+1}-f_{\widehat\theta^\circ}\|_\infty\leq\epsilon$ and $\|f_\theta-f_{\theta^{\circ\circ}}\|_\infty\leq\epsilon$.
Recall that $e_k=\widehat Q_{k+1}-h$.
Since $\widehat L_k(f)=\|f-h\|_n^2-\frac2n\sum_i\xi_i(f-h)(X_i)+\frac1n\sum_i\xi_i^2$ for every $f$, \cref{asm:erm} gives
\begin{equation}\label{eq:basic-entropy}
\|e_k\|_n^2\leq \|f_\theta-h\|_n^2+\frac2n\sum_i\xi_ie_k(X_i)-\frac2n\sum_i\xi_i(f_\theta-h)(X_i)+\eopt[,k]^2 .
\end{equation}

\emph{Noise terms.} We have $e_k=(f_{\widehat\theta^\circ}-h)+(\widehat Q_{k+1}-f_{\widehat\theta^\circ})$ and $\|f_{\widehat\theta^\circ}-h\|_n\leq\|e_k\|_n+\epsilon$, so
\begin{align*}
\frac2n\sum_i\xi_ie_k(X_i)&\leq\frac2n\sum_i\xi_i(f_{\widehat\theta^\circ}-h)(X_i)+4(R_{\max}+\gamma B)\epsilon &\prstep{(bound on $|\xi_i|$)}\\
&\leq\frac14\|f_{\widehat\theta^\circ}-h\|_n^2\\
&\qquad+\frac{8(R_{\max}+\gamma B)^2\ell_0}n+4(R_{\max}+\gamma B)\epsilon &\prstep{($E_1$ with $s=1$)}\\
&\leq\frac12\|e_k\|_n^2+\frac{\epsilon^2}2\\
&\qquad+\frac{8(R_{\max}+\gamma B)^2\ell_0}n+4(R_{\max}+\gamma B)\epsilon. &\prstep{($(x+y)^2\leq2x^2+2y^2$)}
\end{align*}
In the same way, with $E_1$ for $s=-1$ and the net point $\theta^{\circ\circ}$,
\begin{equation*}
-\frac2n\sum_i\xi_i(f_\theta-h)(X_i)\leq\frac12\|f_\theta-h\|_n^2+\frac{\epsilon^2}2+\frac{8(R_{\max}+\gamma B)^2\ell_0}n+4(R_{\max}+\gamma B)\epsilon .
\end{equation*}
Inserting both bounds into \eqref{eq:basic-entropy} and solving for $\|e_k\|_n^2$ gives
\begin{equation}\label{eq:entropy-empirical}
\|e_k\|_n^2\leq3\|f_\theta-h\|_n^2+2\eopt[,k]^2+2\epsilon^2+\frac{32(R_{\max}+\gamma B)^2\ell_0}n+16(R_{\max}+\gamma B)\epsilon .
\end{equation}
The offset term $\frac18\sum_i(f_{\theta^\circ}-h)(X_i)^2$ of $E_1$ is absorbed into $\|e_k\|_n^2$ and $\|f_\theta-h\|_n^2$, and this is what localizes the bound.

\emph{From the empirical to the design norm.} For all $\theta',\theta''\in\Theta$, pointwise, $|(f_{\theta'}-h)^2-(f_{\theta''}-h)^2|=|f_{\theta'}-f_{\theta''}|\,|f_{\theta'}+f_{\theta''}-2h|\leq2\Bres\|f_{\theta'}-f_{\theta''}\|_\infty$.
For the comparator, this gives
\begin{align*}
\|f_\theta-h\|_n^2&\leq\|f_{\theta^{\circ\circ}}-h\|_n^2+2\epsilon\Bres &\prstep{(pointwise bound)}\\
&\leq1.3\,\|f_{\theta^{\circ\circ}}-h\|_{2,\bar\mu_k}^2+\frac{2\Bres^2\ell_0}n+2\epsilon\Bres &\prstep{($E_3$)}\\
&\leq1.3\,\|f_\theta-h\|_{2,\bar\mu_k}^2+\frac{2\Bres^2\ell_0}n+4.6\,\epsilon\Bres, &\prstep{(pointwise bound)}
\end{align*}
and for the estimator,
\begin{align*}
\|e_k\|_{2,\bar\mu_k}^2&\leq\|f_{\widehat\theta^\circ}-h\|_{2,\bar\mu_k}^2+2\epsilon\Bres &\prstep{(pointwise bound)}\\
&\leq\tfrac32\|f_{\widehat\theta^\circ}-h\|_n^2+\frac{3\Bres^2\ell_0}n+2\epsilon\Bres &\prstep{($E_2$)}\\
&\leq\tfrac32\|e_k\|_n^2+\frac{3\Bres^2\ell_0}n+5\epsilon\Bres. &\prstep{(pointwise bound)}
\end{align*}

\emph{Combination.} We insert \eqref{eq:entropy-empirical} and then the bound on $\|f_\theta-h\|_n^2$ into the last display, and use $R_{\max}+\gamma B\leq\Bres$ and $\epsilon^2\leq\epsilon\Bres$:
\begin{align*}
\|e_k\|_{2,\bar\mu_k}^2&\leq4.5\,\|f_\theta-h\|_n^2+3\eopt[,k]^2+3\epsilon^2+5\epsilon\Bres\\
&\qquad+\frac{(48(R_{\max}+\gamma B)^2+3\Bres^2)\ell_0}n+24(R_{\max}+\gamma B)\epsilon &\prstep{(by \eqref{eq:entropy-empirical})}\\
&\leq5.85\,\|f_\theta-h\|_{2,\bar\mu_k}^2+3\eopt[,k]^2+\frac{60\Bres^2\ell_0}n+52.7\,\epsilon\Bres . &\prstep{(comparator bound)}
\end{align*}
This holds for every $\theta\in\Theta$ on the same event, and $\inf_\theta\|f_\theta-h\|_{2,\bar\mu_k}=\eapp[,k]$.
Taking the infimum over $\theta$, with $5.85\leq6$ and $52.7\leq53$, gives \eqref{eq:fast-entropy-scale} for round $k$ with probability at least $1-\delta/K$, and a union bound over $k<K$ proves \eqref{eq:fast-entropy-scale}.

\emph{The scale $\epsilon=B/n$.} Then $\epsilon\leq\Bres$, and $52.7\,\epsilon\Bres\leq52.7\,\Bres^2/n\leq30\,\Bres^2\ell_0/n$, because $\ell_0\geq\log6\geq52.7/30$.
So, on the same events, $\|e_k\|_{2,\bar\mu_k}^2\leq5.85\,\eapp[,k]^2+3\eopt[,k]^2+90\,\Bres^2\ell_0/n$.
Taking square roots, with $\sqrt{x+y+z}\leq\sqrt x+\sqrt y+\sqrt z$, $\sqrt{5.85}\leq3$, $\sqrt3\leq2$ and $\sqrt{90}\leq10$, gives \eqref{eq:fast-entropy}.

\emph{Expectation.} Fix $k$ and run the argument for round $k$ alone, with confidence $\delta/3$ for each event and $\delta=1/n$; the argument holds verbatim for $\delta=1$.
Then \eqref{eq:fast-entropy} holds for this $k$ with $\log(6K/\delta)$ replaced by $\log(6n)$, outside an event of probability at most $1/n$, while $\|e_k\|_{2,\bar\mu_k}\leq\|e_k\|_\infty\leq\Bres$ everywhere.
So
\begin{equation*}
\Ex{\|e_k\|_{2,\bar\mu_k}}\leq3\Ex{\eapp[,k]}+2\eopt[,k]+10\Bres\sqrt{\frac{\log N_\infty(B/n)+\log(6n)}n}+\frac{\Bres}n,
\end{equation*}
and $1/n\leq\sqrt{(\log N_\infty(B/n)+\log(6n))/n}$, which holds because $\log(6n)\geq1$, gives the constant $11$.

Compared with \citet[Lemma~1 and Prop.~2]{russo2013}, the target $h_k$ is fixed at time $0$, so no union bound over targets is needed.
The class may also be misspecified, and the error is transferred to the design norm, as in \citet[Lemma~39]{jin2021bellman}.
\end{proofof}

The next lemma gives the instances of \cref{thm:fast-entropy} in \cref{subsec:fast}.
\begin{lemma}[Metric entropy of Lipschitz parametrizations]\label{lem:entropy}
(i) Let $\Theta$ be a subset of the Euclidean ball of radius $W_\Theta$ in $\dsR^p$, and suppose that $\|f_\theta-f_{\theta'}\|_\infty\leq L_\Theta\|\theta-\theta'\|$ for all $\theta,\theta'\in\Theta$, with $L_\Theta>0$.
Then $\log N_\infty(\epsilon)\leq p\log(1+2W_\Theta L_\Theta/\epsilon)$ for every $\epsilon>0$.
(ii) Let $S\subseteq\dsR^{d_S}$ and $A\subseteq\dsR^{d_A}$ carry metrics induced by norms, with $S\times A\subseteq[-X_\infty,X_\infty]^{m_0}$ and $m_0\defeq d_S+d_A$.
For $m_1,\dots,m_{J-1}\in\dsN$, $m_J\defeq1$ and $B_1,\dots,B_J>0$, let $\Theta$ be the set of $\theta=(W_1,\dots,W_J)$ with $W_\ell\in\dsR^{m_\ell\times m_{\ell-1}}$ and $\|W_\ell\|_{\infty\to\infty}\leq B_\ell$, where $\|M\|_{\infty\to\infty}$ is the largest $\ell_1$-norm of a row of $M$, let $f_\theta(x)\defeq W_J\,\mathrm{relu}(W_{J-1}\cdots\mathrm{relu}(W_1x))$, and let $p\defeq\sum_\ell m_\ell m_{\ell-1}$.
Then $\sup_\theta\|f_\theta\|_\infty\leq X_\infty\prod_\ell B_\ell$, and for every $B\geq X_\infty\prod_\ell B_\ell$ and $n\geq1$,
\begin{equation*}
\log N_\infty(B/n)\leq p\log\Bigl(1+2n\Bigl(\sum_\ell m_\ell B_\ell^2\Bigr)^{1/2}\Bigl(\sum_\ell m_{\ell-1}B_\ell^{-2}\Bigr)^{1/2}\Bigr).
\end{equation*}
If moreover $B\geq R_{\max}/(1-\gamma)$, then $\calF$ satisfies \cref{asm:class}.
\end{lemma}

\begin{proof}

\emph{(i)} Call a set $T\subseteq\Theta$ separated if its points are at Euclidean distance at least $\epsilon/L_\Theta$ from each other.
The open Euclidean balls of radius $\epsilon/(2L_\Theta)$ around the points of a separated set are disjoint and lie in the ball of radius $W_\Theta+\epsilon/(2L_\Theta)$ around $0$, so comparing volumes gives $|T|\leq(1+2W_\Theta L_\Theta/\epsilon)^p$.
So there is a separated set $T$ of maximal size.
By maximality, every $\theta\in\Theta$ has some $t\in T$ with $\|\theta-t\|<\epsilon/L_\Theta$, and then $\|f_\theta-f_t\|_\infty\leq L_\Theta\|\theta-t\|<\epsilon$.
So $T$ is a set as in \cref{thm:fast-entropy}, and $\log N_\infty(\epsilon)\leq\log|T|\leq p\log(1+2W_\Theta L_\Theta/\epsilon)$.

\emph{(ii), size.} Let $a_0\defeq x$, $a_\ell\defeq\mathrm{relu}(W_\ell a_{\ell-1})$ for $1\leq\ell<J$, and note $f_\theta(x)=W_Ja_{J-1}$.
We have $\|Mv\|_\infty\leq\|M\|_{\infty\to\infty}\|v\|_\infty$, and $\mathrm{relu}$ is $1$-Lipschitz in each coordinate with $\mathrm{relu}(0)=0$, so $\|a_\ell\|_\infty\leq B_\ell\|a_{\ell-1}\|_\infty$.
With $\|x\|_\infty\leq X_\infty$, this gives $\|a_\ell\|_\infty\leq X_\infty\prod_{j\leq\ell}B_j$ and $\sup_\theta\|f_\theta\|_\infty\leq X_\infty\prod_\ell B_\ell$.

\emph{(ii), Lipschitz constant.} Let $\theta,\theta'\in\Theta$ have layer outputs $a_\ell$ and $a'_\ell$.
Since $\|W'_\ell\|_{\infty\to\infty}\leq B_\ell$,
\begin{equation*}
\|a_\ell-a'_\ell\|_\infty\leq\|W_\ell-W'_\ell\|_{\infty\to\infty}\|a_{\ell-1}\|_\infty+B_\ell\|a_{\ell-1}-a'_{\ell-1}\|_\infty ,
\end{equation*}
and the same holds for the output layer without $\mathrm{relu}$.
Induction over $\ell$ with the bound on $\|a_{\ell-1}\|_\infty$ gives $\|f_\theta-f_{\theta'}\|_\infty\leq X_\infty\sum_\ell\bigl(\prod_{j\neq\ell}B_j\bigr)\|W_\ell-W'_\ell\|_{\infty\to\infty}$.
For $M\in\dsR^{m_\ell\times m_{\ell-1}}$, the Cauchy--Schwarz inequality in each row gives $\|M\|_{\infty\to\infty}\leq\sqrt{m_{\ell-1}}\|M\|_F$, and the Cauchy--Schwarz inequality over $\ell$ then gives
\begin{align*}
\|f_\theta-f_{\theta'}\|_\infty&\leq X_\infty\Bigl(\prod_jB_j\Bigr)\sum_\ell\frac{\sqrt{m_{\ell-1}}}{B_\ell}\|W_\ell-W'_\ell\|_F\leq L_\Theta\|\theta-\theta'\|,\qquad\text{where}\\
L_\Theta&\defeq X_\infty\Bigl(\prod_jB_j\Bigr)\Bigl(\sum_\ell m_{\ell-1}B_\ell^{-2}\Bigr)^{1/2},
\end{align*}
for the Euclidean norm of $\theta\in\dsR^p$.

\emph{(ii), radius and entropy.} Each row of $W_\ell$ has Euclidean norm at most its $\ell_1$-norm, so $\|\theta\|^2=\sum_\ell\|W_\ell\|_F^2\leq\sum_\ell m_\ell B_\ell^2\eqdef W_\Theta^2$.
Part (i) with $\epsilon=B/n$ applies, and $X_\infty\prod_jB_j\leq B$ gives $2W_\Theta L_\Theta n/B\leq2n(\sum_\ell m_\ell B_\ell^2)^{1/2}(\sum_\ell m_{\ell-1}B_\ell^{-2})^{1/2}$, which is the claimed bound.

\emph{(ii), \cref{asm:class}.} The set $\Theta$ is closed in $\dsR^p$, so it is a Borel space, and $\theta\mapsto f_\theta$ is Lipschitz into $(\calQ,\|\cdot\|_\infty)$.
Each $f_\theta$ is continuous and bounded on $S\times A$, so $f_\theta\in\calQ_0$, and $\sup_\theta\|f_\theta\|_\infty\leq X_\infty\prod_\ell B_\ell\leq B$ with $B\geq R_{\max}/(1-\gamma)$.
\end{proof}

\begin{corollary}[Fast rates for linear classes]\label{cor:fast}
Let $\calF=\{f_w=\langle w,\phi(\cdot)\rangle:w\in W\}$, where $\phi:S\times A\to\dsR^d$ is continuous with $\sup\|\phi\|\leq\kappa$, $W$ is a closed subset of the Euclidean ball of radius $W_0$, and $\sup_{w\in W}\|f_w\|_\infty\leq B$ with $B>0$ and $B\geq R_{\max}/(1-\gamma)$, and let $\iota\defeq\log(3n\max\{1,\kappa W_0/B\})$.
Under \cref{asm:model,asm:adaptive-batches,asm:erm}, for every $\delta\in(0,1)$, with probability at least $1-\delta$, \eqref{eq:fast-entropy} holds for all $k<K$ with $d\iota$ in place of $\log N_\infty(B/n)$, and ${\Ex{\|e_k\|_{2,\bar\mu_k}}\leq3\Ex{\eapp[,k]}+2\eopt[,k]+18\Bres\sqrt{d\iota/n}}$ for every $k$.
\end{corollary}

\begin{proofof}{\cref{cor:fast}}

\Cref{asm:class} holds with $\Theta=W$, because $W$ is closed in $\dsR^d$, each $f_w$ is continuous and $\|f_w-f_{w'}\|_\infty\leq\kappa\|w-w'\|$.
If $\kappa=0$, then $\calF=\{0\}$ and $N_\infty(B/n)=1$.
Otherwise, \cref{lem:entropy}(i) with $p=d$, $W_\Theta=W_0$ and $L_\Theta=\kappa$ gives $\log N_\infty(B/n)\leq d\log(1+2n\kappa W_0/B)\leq d\iota$, because $1+2nx\leq3n\max\{1,x\}$ for all $x\geq0$ and $n\geq1$.
So \cref{thm:fast-entropy} gives the high-probability bound.
For the expectation bound, $\iota\geq\log(3n)\geq\log3$ gives $\log(6n)=\log2+\log(3n)\leq(1+\log2/\log3)\,\iota$, so $d\iota+\log(6n)\leq(2+\log2/\log3)\,d\iota$, and $11\sqrt{2+\log2/\log3}\leq18$.
\end{proofof}

\section{Proofs for \texorpdfstring{\cref{sec:coverage}}{Section \ref{sec:coverage}}}\label{app:coverage}

\begin{proofof}{\cref{lem:obstructions}}

(i) By \cref{prop:propagation}, $\Xi_k$ is a countable mixture of measures $\nu_S\otimes\delta_g$ with $g\in\{\pi^*,\pi_{k+1},\widehat\pi_K\}$.
So $\Xi_k(\Gamma)=1$ for the Borel set $\Gamma$ formed by the union of the three graphs, whereas $\bar\mu_k(\Gamma)=0$.
The same argument applies to any marginal of a deterministic policy.
(ii) Smoothing acts on the action only, so the state marginal of $\Xi_k^\sigma$ equals that of $\Xi_k$.
This is a mixture of $\rho$ and of measures $\eta P$, which are atomless because $\eta P(\{s\})=\int P(\{s\}\mid x,a)\eta(\dd x,\dd a)=0$.
If the state marginal of $\bar\mu_k$ is carried by a countable set $N$, then $\Xi_k^\sigma((S\setminus N)\times A)=1$ while $\bar\mu_k((S\setminus N)\times A)=0$.
For a single trajectory, $S_{k,i}=S'_{k,i-1}$ is $\calH_{k,i-1}$-measurable for $i\geq2$, and $S_{k,1}$ is $\calH_k$-measurable when the trajectory continues from the previous round.
The state marginal of $\mu_{k,i}$ is then $\delta_{S_{k,i}}$, and $\bar\mu_k$ is carried by $\{S_{k,1},\dots,S_{k,n}\}$.
\end{proofof}

For uncountable $A$, no design makes the all-policy coefficient finite: $\mu(S\times\{c\})>0$ for at most countably many $c\in A$, while the step-$0$ marginal $\rho\otimes\delta_c$ of the constant policy $c$ is carried by $S\times\{c\}$.

\emph{Action smoothing.} Two facts used in \cref{def:smoothing} follow from compactness and Tonelli's theorem.
First, cover $A$ by finitely many open balls $\dsB_{\sigma/2}(b_j)$, $j\leq j_0$.
Each has positive $\lambda_A$-mass by full support, and $\dsB_{\sigma/2}(b_j)\subseteq\bar\dsB_\sigma(b)$ whenever $b\in\dsB_{\sigma/2}(b_j)$, so $v_A(\sigma)\geq\min_{j\leq j_0}\lambda_A(\dsB_{\sigma/2}(b_j))>0$.
Second, for a Borel $E\subseteq S\times A$ the map $(s,b,a)\mapsto\dsone_E(s,a)\dsone\{\varrho_A(a,b)\leq\sigma\}$ is Borel, so $(s,b)\mapsto\lambda_A(E_s\cap\bar\dsB_\sigma(b))$ is Borel by Tonelli's theorem.
The denominator $b\mapsto\lambda_A(\bar\dsB_\sigma(b))$ is upper semicontinuous and positive, so $(s,b)\mapsto U_\sigma(E_s\mid b)$ is Borel (for $\sigma=0$ it is $\dsone_E(s,b)$), and $\nu^\sigma$ is a probability measure by monotone convergence.

\begin{proposition}[Tightness of the transfer]\label{prop:tight}
Let $A=[0,1]^{d_A}$ with the $\ell_\infty$ metric and $\lambda_A$ the Lebesgue measure, $S=\{s_0\}$, $\mu=\delta_{s_0}\otimes\lambda_A$ and $\nu=\delta_{(s_0,a_0)}$ with $a_0$ the center of $A$.
For every $L>0$ and $\zeta\in(0,\zeta_0(L)]$ there is an $L$-Lipschitz $h$ with $\|h\|_{2,\mu}=\zeta$ and $\int|h|\dd\nu=c_{d_A}\zeta^{2/(d_A+2)}L^{d_A/(d_A+2)}$, whereas $\inf_{\sigma}\bigl[\|\dd\nu^\sigma/\dd\mu\|_\infty^{1/2}\zeta+L\sigma\bigr]\leq2\zeta^{2/(d_A+2)}L^{d_A/(d_A+2)}$; here $c_{d_A}>0$ depends only on $d_A$.
\end{proposition}

\begin{proofof}{\cref{prop:tight}}

Write $\|\cdot\|$ for the $\ell_\infty$ norm, so that $\lambda_A(\bar\dsB_t(a_0))=(2t)^{d_A}$ for $t\leq\frac12$.
For $H>0$ let $h(a)\defeq(H-L\|a-a_0\|)_+$, which is $L$-Lipschitz and supported in $\bar\dsB_{H/L}(a_0)\subseteq A$ when $H\leq L/2$.
Using that $\|a-a_0\|$ has density $d_A2^{d_A}t^{d_A-1}$ on $[0,\frac12]$,
\begin{align*}
\|h\|_{2,\mu}^2=\int_0^{H/L}(H-Lt)^2d_A2^{d_A}t^{d_A-1}\dd t&=d_A2^{d_A}H^{d_A+2}L^{-d_A}\int_0^1(1-x)^2x^{d_A-1}\dd x\\
&=\frac{2^{d_A+1}}{(d_A+1)(d_A+2)}H^{d_A+2}L^{-d_A}.
\end{align*}
Setting $\|h\|_{2,\mu}=\zeta$ gives $\int|h|\dd\nu=h(a_0)=H=c_{d_A}\,\zeta^{2/(d_A+2)}L^{d_A/(d_A+2)}$ with $c_{d_A}=((d_A+1)(d_A+2)/2^{d_A+1})^{1/(d_A+2)}$.
The requirement $H\leq L/2$ holds for $\zeta\leq\zeta_0(L)\defeq(2^{d_A+1}/((d_A+1)(d_A+2)))^{1/2}(L/2)^{(d_A+2)/2}L^{-d_A/2}$, and $\zeta_0(L)\leq L$.
On the other hand, $\bar\dsB_\sigma(a_0)$ contains a cube of side $\min\{2\sigma,1\}\geq\sigma$ for $\sigma\leq1$, so $\|\dd\nu^\sigma/\dd\mu\|_\infty=\lambda_A(\bar\dsB_\sigma(a_0))^{-1}\leq\sigma^{-d_A}$.
The choice $\sigma=(\zeta/L)^{2/(d_A+2)}\leq1$ gives $\sigma^{-d_A/2}\zeta+L\sigma=2\zeta^{2/(d_A+2)}L^{d_A/(d_A+2)}$.
\end{proofof}

\begin{proofof}{\cref{thm:main}}

All quantities in the statement are random variables by \cref{lem:measurability}, and $C_{\mathrm{tr}}$, $\omega_{\mathrm{tr}}$ are deterministic.
Fix an outcome in the event of \cref{asm:transfer}, and fix the greedy selectors of \cref{asm:transfer} for this outcome.
\Cref{prop:propagation} holds for every choice of selectors, so we may apply it to $Q_k=\widehat Q_k\in\calF\subseteq\calQ_0$ with these selectors and $\widehat\pi_K=g(\widehat\theta_K,\cdot)$; the selectors need not depend measurably on the outcome.
Inserting \eqref{eq:transfer} for every $k<K$ into \eqref{eq:propagation}, and using $\|\widehat Q_0-Q^*\|_\infty\leq2B$, which holds because $\|\widehat Q_0\|_\infty\leq B$ and $\|Q^*\|_\infty\leq R_{\max}/(1-\gamma)\leq B$, gives \eqref{eq:main-path}.
The argument uses \cref{asm:transfer} only at the fixed outcome, so \eqref{eq:main-path} holds at every outcome at which \eqref{eq:transfer} holds for every $k<K$ and some greedy selectors as in \cref{asm:transfer}.
\Cref{lem:instances}(iii) and \cref{prop:trajectory-gram} use it in this form.

The event of \cref{asm:transfer} has probability one, so \eqref{eq:main-path} holds almost surely, and taking expectations with $\Ex{\|e_k\|_{2,\bar\mu_k}}\leq\veps_k$ gives \eqref{eq:main}.

Almost surely on $E$, we bound each bracket in \eqref{eq:main-path} by its maximum over $k$ and use $\sum_{k<K}\gamma^{K-1-k}=(1-\gamma^K)/(1-\gamma)$, which gives the last bound.
The admissible $\veps_k$ and $\bar\veps_k$ stated after the theorem are those of \cref{thm:residual}, with $\E\sqrt{2\Delta_k}\leq\sqrt{2\alpha_n}$, of \cref{prop:hilbert}, and of \cref{thm:fast-entropy}.
\end{proofof}

\begin{proofof}{\cref{prop:sufficient}}

Every component $\nu=\nu_S\otimes\delta_g$ of $\Xi_k$ satisfies $\nu_S\leq D\lambda_S$: either $\nu_S=\rho$, or $\nu_S=\eta P\leq D\lambda_S$ by (a).
For a Borel $E$, $U_\sigma(E_s\mid b)\leq\lambda_A(E_s)/v_A(\sigma)$, so
\begin{equation*}
\nu^\sigma(E)\leq\frac1{v_A(\sigma)}\int\lambda_A(E_s)\,\nu_S(\dd s)\leq\frac D{v_A(\sigma)}(\lambda_S\otimes\lambda_A)(E)\leq\frac D{c_\mu v_A(\sigma)}\bar\mu_k(E),
\end{equation*}
using (b) averaged over $i$.
Since smoothing is linear, the same bound holds for the mixture $\Xi_k^\sigma$.
No property of the selectors was used.
\end{proofof}

\begin{proposition}[Sampling protocols]\label{prop:protocols}
Suppose that, almost surely, $\Pr(A_{k,i}\in F\mid\calH_{k,i-1}\vee\sigma(S_{k,i}))\geq c_A\lambda_A(F)$ for all $k$, $i$ and Borel $F\subseteq A$ (the \emph{behavior floor}).
This holds when $A_{k,i}$ is drawn from a law at least $c_A\lambda_A$ given a $\sigma$-field that contains $\calH_{k,i-1}\vee\sigma(S_{k,i})$, for instance from a mixture with weight $c_A$ on $\lambda_A$ whose coin and draw are independent of that $\sigma$-field.
Then (b) of \cref{prop:sufficient} holds in each of the following two cases.
(i) \emph{Restarts.} Conditionally on $\calH_{k,i-1}$, $S_{k,i}$ is drawn from a law at least $c_S\lambda_S$.
Then $c_\mu=c_Sc_A$.
(ii) \emph{Thinned trajectories.} The data form one trajectory $S_0,A_0,R_0,S_1,\dots$.
The transitions at odd times are used in order, $n$ per round, and those at even times are skipped, including the transition at time $0$ and one transition between rounds.
The $\sigma$-field $\calH_{0,0}$ contains $S_0$ but not $A_0$, $\calH_{k,i}$ is the history up to the next state $S_{t_i+1}$ of the $i$-th used transition of round $k$, at time $t_i$, and $\calH_{k,0}=\calH_{k-1,n}$ for $k\geq1$.
Every skipped action satisfies $\Pr(A_{t_i+1}\in F\mid\calH_{k,i})\geq c_A\lambda_A(F)$, and $\Pr(A_0\in F\mid\calH_{0,0})\geq c_A\lambda_A(F)$.
Finally, $\int P(\cdot\mid s,a)\lambda_A(\dd a)\geq m\lambda_S$ for all $s$.
Then $c_\mu=c_A^2m$.
\end{proposition}

\begin{proofof}{\cref{prop:protocols}}

For a Borel set $E$ and $\calG\defeq\calH_{k,i-1}\vee\sigma(S_{k,i})$, the behavior floor gives $\Ex{\dsone_E(S_{k,i},A_{k,i})\mid\calG}\geq c_A\lambda_A(E_{S_{k,i}})$.

(i) So $\mu_{k,i}(E)\geq c_A\Ex{\lambda_A(E_{S_{k,i}})\mid\calH_{k,i-1}}\geq c_Ac_S(\lambda_S\otimes\lambda_A)(E)$.

(ii) Let $t_i$ be the time of the $i$-th used transition, so that $t_{i+1}=t_i+2$.
The $\sigma$-field $\calH_{k,i}$ contains the history up to $S_{t_i+1}$ but not $A_{t_i+1}$, and the skipped action $A_{t_i+1}$ has conditional law at least $c_A\lambda_A$ given $\calH_{k,i}$.
The Markov property given the history up to $A_{t_i+1}$, which contains $\calH_{k,i}\vee\sigma(A_{t_i+1})$, and the tower property give
\begin{equation*}
\Law(S_{t_{i+1}}\mid\calH_{k,i})=\Ex{P(\cdot\mid S_{t_i+1},A_{t_i+1})\mid\calH_{k,i}}\geq c_A\int P(\cdot\mid S_{t_i+1},a)\lambda_A(\dd a)\geq c_Am\lambda_S .
\end{equation*}
Combined with the behavior floor given $\calH_{k,i}\vee\sigma(S_{t_{i+1}})$, this gives
\begin{equation*}
\mu_{k,i+1}(E)\geq c_A\Ex{\lambda_A(E_{S_{t_{i+1}}})\mid\calH_{k,i}}\geq c_A^2m(\lambda_S\otimes\lambda_A)(E).
\end{equation*}
The first transition of round $k\geq1$ follows the skipped transition between rounds, since $\calH_{k,0}=\calH_{k-1,n}$, and that of round $0$ follows the skipped transition at time $0$, with $\Law(S_1\mid\calH_{0,0})\geq c_Am\lambda_S$ by the same argument with the floor on $A_0$.
\Cref{asm:adaptive-batches} holds with this filtration.
$Z_{k,i}$ is $\calH_{k,i}$-measurable, and conditionally on $\calH_{k,i-1}\vee\sigma(S_{t_i},A_{t_i})$ the reward and next state are drawn from the MDP: by the Markov property this holds given the full history up to $A_{t_i}$, and the tower property passes it to the smaller $\sigma$-field.
\end{proofof}

The proof of \cref{lem:transfer} uses $\nu^\sigma\leq C\mu$ only through $\int(\dd\nu^\sigma/\dd\mu)^2\dd\mu\leq C$.
So (a) of \cref{prop:sufficient} can be weakened to a second-moment condition.

\begin{proposition}[Second-moment coverage]\label{prop:second-moment}
Let \cref{asm:model,asm:class,asm:adaptive-batches} and (b) of \cref{prop:sufficient} hold, and suppose that $\rho$ and all $P(\cdot\mid s,a)$ have $\lambda_S$-densities, jointly Borel with $L^2(\lambda_S)$-norms at most $D_2$.
Then, for every $\sigma>0$, \cref{asm:transfer} holds for every choice of selectors with $C_{\mathrm{tr}}=D_2^2/(c_\mu v_A(\sigma))$ and $\omega_{\mathrm{tr}}=\omega_\calF(\sigma)$.
\end{proposition}

\begin{proofof}{\cref{prop:second-moment}}

Let $p(s'\mid s,a)$ be a jointly Borel version of the $\lambda_S$-densities of the kernels $P(\cdot\mid s,a)$.
Fix an outcome in the almost-sure event of (b), a round $k$, a choice of selectors and $\sigma>0$.
By \cref{prop:propagation}, $\Xi_k$ is a countable convex combination of measures $\nu_S\otimes\delta_g$ with $\nu_S=\rho$ or $\nu_S=\eta P$ for some $\eta\in\Delta(S\times A)$.
In the first case, the $\lambda_S$-density $p_S$ of $\nu_S$ satisfies $\|p_S\|_{2,\lambda_S}\leq D_2$ by assumption.
In the second case, $\nu_S$ has the $\lambda_S$-density $p_S(s')\defeq\int p(s'\mid s,a)\,\eta(\dd s,\dd a)$, which is Borel by Tonelli's theorem, and Minkowski's integral inequality gives $\|p_S\|_{2,\lambda_S}\leq\int\|p(\cdot\mid s,a)\|_{2,\lambda_S}\,\eta(\dd s,\dd a)\leq D_2$.
The smoothed component $(\nu_S\otimes\delta_g)^\sigma$ has the density $q(s,a)\defeq p_S(s)\dsone\{\varrho_A(a,g(s))\leq\sigma\}/\lambda_A(\bar\dsB_\sigma(g(s)))$ with respect to $\lambda_S\otimes\lambda_A$, and
\begin{equation*}
\int q^2\,\dd(\lambda_S\otimes\lambda_A)=\int\frac{p_S(s)^2}{\lambda_A(\bar\dsB_\sigma(g(s)))}\,\lambda_S(\dd s)\leq\frac{D_2^2}{v_A(\sigma)} .
\end{equation*}
Minkowski's inequality passes this bound to the density $q_\Xi$ of the countable convex combination $\Xi_k^\sigma$, so $\int q_\Xi^2\,\dd(\lambda_S\otimes\lambda_A)\leq D_2^2/v_A(\sigma)$.

Let $\bar\mu_k=p_\mu\,(\lambda_S\otimes\lambda_A)+\bar\mu_k^\perp$ be the Lebesgue decomposition, where $\bar\mu_k^\perp$ is concentrated on a $(\lambda_S\otimes\lambda_A)$-null set $N$.
By (b), averaged over $i$, $\int_Ep_\mu\,\dd(\lambda_S\otimes\lambda_A)=\bar\mu_k(E)\geq c_\mu(\lambda_S\otimes\lambda_A)(E)$ for every Borel $E\subseteq N^c$, so $p_\mu\geq c_\mu$ almost everywhere.
The measure $\Xi_k^\sigma$ has a density with respect to $\lambda_S\otimes\lambda_A$, so $\Xi_k^\sigma(N)=0$, and $\Xi_k^\sigma\ll\bar\mu_k$ with $\dd\Xi_k^\sigma/\dd\bar\mu_k=q_\Xi/p_\mu$ on $N^c$.
So
\begin{equation*}
\int\Bigl(\frac{\dd\Xi_k^\sigma}{\dd\bar\mu_k}\Bigr)^2\dd\bar\mu_k=\int_{N^c}\frac{q_\Xi^2}{p_\mu}\,\dd(\lambda_S\otimes\lambda_A)\leq\frac1{c_\mu}\int q_\Xi^2\,\dd(\lambda_S\otimes\lambda_A)\leq\frac{D_2^2}{c_\mu v_A(\sigma)} .
\end{equation*}
The proof of \cref{lem:transfer} uses $\nu^\sigma\leq C\mu$ only through this second moment, so it gives \eqref{eq:transfer} for $e_k$ with $C_{\mathrm{tr}}=D_2^2/(c_\mu v_A(\sigma))$ and $\omega_{\mathrm{tr}}=\omega_\calF(\sigma)$, for every $k<K$ and every choice of selectors.
\end{proofof}

\begin{lemma}[Lipschitz modulus]\label{lem:modulus}
If every $f_\theta$ is $L_\calF$-Lipschitz in $a$, $r$ is $L_r$-Lipschitz in $a$, and $\|P(\cdot\mid s,a)-P(\cdot\mid s,b)\|_{\mathrm{TV}}\leq L_P\varrho_A(a,b)$ with $\|\nu\|_{\mathrm{TV}}\defeq\sup_{|u|\leq1}|\int u\,\dd\nu|$, then $\omega_\calF(\sigma)\leq(L_\calF+L_r+\gamma BL_P)\sigma$.
\end{lemma}

\begin{proof}

For $h=f_\theta-T^*f_{\theta'}$ and $\varrho_A(a,b)\leq\sigma$, $|h(s,a)-h(s,b)|\leq L_\calF\sigma+L_r\sigma+\gamma|\int V_{f_{\theta'}}\dd(P(\cdot\mid s,a)-P(\cdot\mid s,b))|$, and $\|V_{f_{\theta'}}\|_\infty\leq B$.
\end{proof}

\begin{lemma}[Norm-equivalence transfer]\label{lem:nikolskii}
Let $\lambda'_A\in\Delta(A)$ and suppose that $\sup_ah(s,a)^2\leq\Lambda\int h(s,a)^2\lambda'_A(\dd a)$ for all $s$.
If $\nu\in\Delta(S\times A)$ has state marginal $\nu_S$ and $\nu_S\otimes\lambda'_A\leq C'\mu$, then $\int|h|\dd\nu\leq\sqrt{\Lambda C'}\,\|h\|_{2,\mu}$.
\end{lemma}

\begin{proof}

The condition gives $|h(s,a)|\leq\sqrt\Lambda\,(\int h(s,\cdot)^2\,\dd\lambda'_A)^{1/2}$ for all $(s,a)$, and the right-hand side is a Borel function of $s$ by Tonelli's theorem.
Integrating against $\nu$, whose state marginal is $\nu_S$, and using Jensen's inequality, we obtain
\begin{align*}
\int|h|\,\dd\nu&\leq\sqrt\Lambda\int\Bigl(\int h(s,\cdot)^2\,\dd\lambda'_A\Bigr)^{1/2}\nu_S(\dd s)\\
&\leq\sqrt\Lambda\Bigl(\int h^2\,\dd(\nu_S\otimes\lambda'_A)\Bigr)^{1/2},
\end{align*}
and $\int h^2\dd(\nu_S\otimes\lambda'_A)\leq C'\|h\|_{2,\mu}^2$ because $\nu_S\otimes\lambda'_A\leq C'\mu$.
\end{proof}

\begin{lemma}[Instances of the transfer condition]\label{lem:instances}
Let \cref{asm:model,asm:class,asm:adaptive-batches} hold.
(i) \emph{Smoothing.} If \cref{asm:coverage} holds at scale $\sigma$, then \cref{asm:transfer} holds with $C_{\mathrm{tr}}=\Cad(\sigma)$ and $\omega_{\mathrm{tr}}=\omega_\calF(\sigma)$, for the selectors of \cref{asm:coverage}.
(ii) \emph{Norm equivalence in the action.} Suppose that there are $\lambda'_A\in\Delta(A)$ and $\Lambda,C'>0$ such that $\sup_{a\in A}h(s,a)^2\leq\Lambda\int h(s,a)^2\,\lambda'_A(\dd a)$ for all $s\in S$ and all $h=f_\theta-T^*f_{\theta'}$ with $\theta,\theta'\in\Theta$, and that, almost surely, for some greedy selectors as in \cref{asm:transfer} and every $k<K$, the state marginal $\Xi_{k,S}$ of $\Xi_k$ satisfies $\Xi_{k,S}\otimes\lambda'_A\leq C'\bar\mu_k$.
Then \cref{asm:transfer} holds with $C_{\mathrm{tr}}=\Lambda C'$ and $\omega_{\mathrm{tr}}=0$, for these selectors.
Under (a) and (b) of \cref{prop:sufficient}, the condition on $\Xi_{k,S}$ holds for every choice of selectors, with $\lambda'_A=\lambda_A$ and $C'=D/c_\mu$.
(iii) \emph{Gram-matrix coverage for linear classes.} Let $\calF$ be a class as in \cref{cor:fast} such that $T^*f_w\in\{\langle v,\phi(\cdot)\rangle:v\in\dsR^d\}$ for every $w\in W$, and let $\lambda_0>0$ and $\Sigma_k\defeq\int\phi\phi\T\dd\bar\mu_k$.
Then, on the event that $\lambda_{\min}(\Sigma_k)\geq\lambda_0$ for all $k<K$, \eqref{eq:transfer} holds for every $k<K$ and every choice of selectors, with $C_{\mathrm{tr}}=\kappa^2/\lambda_0$ and $\omega_{\mathrm{tr}}=0$, and so does \eqref{eq:main-path}.
\end{lemma}

\begin{proofof}{\cref{lem:instances}}

Every residual $e_k=f_{\widehat\theta_{k+1}}-T^*f_{\widehat\theta_k}$ is of the form $f_\theta-T^*f_{\theta'}$ with $\theta,\theta'\in\Theta$, so $\omega_{e_k}(\sigma)\leq\omega_\calF(\sigma)$, and $e_k$ satisfies every condition that (ii) imposes on all such differences.

\emph{(i)} On the almost-sure event of \cref{asm:coverage}, \cref{lem:transfer} with $\nu=\Xi_k$, $\mu=\bar\mu_k$, $h=e_k$ and $C=\Cad(\sigma)$ gives $\int|e_k|\,\dd\Xi_k\leq\sqrt{\Cad(\sigma)}\,\|e_k\|_{2,\bar\mu_k}+\omega_\calF(\sigma)$ for every $k<K$, for the selectors of \cref{asm:coverage}.

\emph{(ii)} By the condition on the residuals, $|e_k(s,a)|\leq\sqrt\Lambda\,(\int e_k(s,\cdot)^2\,\dd\lambda'_A)^{1/2}$ for all $(s,a)$, and the right-hand side is a Borel function of $s$ by Tonelli's theorem.
Integrating against $\Xi_k$, whose state marginal is $\Xi_{k,S}$, and using Jensen's inequality, we obtain on the almost-sure event of the condition
\begin{align*}
\int|e_k|\,\dd\Xi_k&\leq\sqrt\Lambda\int\Bigl(\int e_k(s,\cdot)^2\,\dd\lambda'_A\Bigr)^{1/2}\Xi_{k,S}(\dd s)\\
&\leq\sqrt\Lambda\Bigl(\int e_k^2\,\dd(\Xi_{k,S}\otimes\lambda'_A)\Bigr)^{1/2} &\prstep{(Jensen)}\\
&\leq\sqrt{\Lambda C'}\,\|e_k\|_{2,\bar\mu_k}. &\prstep{($\Xi_{k,S}\otimes\lambda'_A\leq C'\bar\mu_k$)}
\end{align*}
This is the argument of \cref{lem:nikolskii}.
Under (a) and (b) of \cref{prop:sufficient}, each component $\nu_S\otimes\delta_g$ of $\Xi_k$ has $\nu_S=\rho$ or $\nu_S=\eta P$, so $\nu_S\leq D\lambda_S$ by (a), and so does the countable mixture $\Xi_{k,S}$.
With (b), averaged over $i$, this gives $\Xi_{k,S}\otimes\lambda_A\leq D(\lambda_S\otimes\lambda_A)\leq(D/c_\mu)\bar\mu_k$ almost surely, for every choice of selectors.

\emph{(iii)} The class satisfies \cref{asm:class} with $\Theta=W$.
By assumption, $T^*f_{\widehat\theta_k}=\langle v',\phi\rangle$ for some $v'\in\dsR^d$, so $e_k=\langle v_k,\phi\rangle$ with $v_k\defeq\widehat\theta_{k+1}-v'$.
For every $\nu\in\Delta(S\times A)$, the Cauchy--Schwarz inequality and $\|\phi\|\leq\kappa$ give $\int|e_k|\,\dd\nu\leq\kappa\|v_k\|$, while $\|e_k\|_{2,\bar\mu_k}^2=v_k\T\Sigma_kv_k\geq\lambda_{\min}(\Sigma_k)\|v_k\|^2$.
So, on the event that $\lambda_{\min}(\Sigma_k)\geq\lambda_0$ for all $k<K$, $\int|e_k|\,\dd\Xi_k\leq\kappa\lambda_0^{-1/2}\|e_k\|_{2,\bar\mu_k}$ for every $k<K$ and every choice of selectors, and \eqref{eq:main-path} follows there, because the proof of \cref{thm:main} is pathwise.
This set is an event, because $\Sigma_k=\frac1n\sum_i\Ex{\phi\phi\T(S_{k,i},A_{k,i})\mid\calH_{k,i-1}}$ is a random matrix and $\lambda_{\min}$ is continuous.
\end{proofof}

\begin{proofof}{\cref{cor:rates}}

\emph{Choices of $a_n$.} Each choice of $a_n$ in \cref{subsec:rates} satisfies $\Ex{\|e_k\|_{2,\bar\mu_k}}\leq a_n$ for every $k$.
For the localized choice, this is the expectation bound of \cref{thm:fast-entropy}.
For the global choice and the classes of \cref{prop:hilbert}, $\alpha_n\leq2\Bres^2/\sqrt n$, so $\sqrt{2\alpha_n}\leq2\Bres n^{-1/4}$, and \cref{thm:residual} applies.
For the global choice and a class with $\mathfrak M_n(\calF)\geq B$, \cref{prop:contraction} gives $\alpha_n=2\overline\Rseq_n/n\leq2\Bres^2/\sqrt n+64\Bres c_n\mathfrak M_n(\calF)/n$, so $\sqrt{2\alpha_n}\leq2\Bres n^{-1/4}+8\sqrt2\,(\Bres c_n\mathfrak M_n(\calF)/n)^{1/2}$ by $\sqrt{x+y}\leq\sqrt x+\sqrt y$, and \cref{thm:residual} applies.

\emph{Transfer.} If \cref{asm:transfer} holds, then \eqref{eq:main} with $\veps_k=a_n$ and $\sum_k\gamma^{K-1-k}\leq1/(1-\gamma)$ gives
\begin{equation*}
\Ex{\|V^*-V^{\widehat\pi_K}\|_{1,\rho}}-\frac{8B\gamma^K}{1-\gamma}\leq\frac{4}{(1-\gamma)^2}\bigl(\sqrt{C_{\mathrm{tr}}}\,a_n+\omega_{\mathrm{tr}}\bigr).
\end{equation*}
In (a) and (b), \cref{prop:sufficient} gives \cref{asm:coverage} for every $\sigma>0$, so \cref{lem:instances}(i) applies with $C_{\mathrm{tr}}=\Cad(\sigma)$ and $\omega_{\mathrm{tr}}=\omega_\calF(\sigma)$, which we bound below.
In (c), \cref{lem:instances}(ii) with $\lambda'_A=\lambda_A$ and $C'=D/c_\mu$ gives $C_{\mathrm{tr}}=\Lambda D/c_\mu$ and $\omega_{\mathrm{tr}}=0$.
(a) With the discrete metric and $\sigma<1$, $\bar\dsB_\sigma(b)=\{b\}$, so $v_A(\sigma)=1/|A|$, $\omega_\calF(\sigma)=0$ and $\Cad(\sigma)\leq D|A|/c_\mu$ by \cref{prop:sufficient}.
(b) For a convex body and $b\in A$, the homothety $a\mapsto b+\frac{\sigma}{\operatorname{diam}(A)}(a-b)$ maps $A$ into $A\cap\bar\dsB_\sigma(b)$.
So $v_A(\sigma)\geq(\sigma/\operatorname{diam}(A))^{d_A}$ for $\sigma\leq\operatorname{diam}(A)$, and $\sqrt{\Cad(\sigma)}\,a_n\leq\sqrt{D/c_\mu}\operatorname{diam}(A)^{d_A/2}a_n\,\sigma^{-d_A/2}$.
The choice $\sigma=\bigl(\sqrt{D/c_\mu}\operatorname{diam}(A)^{d_A/2}a_n/L\bigr)^{2/(d_A+2)}$ makes this bound and $L\sigma$ both equal to $L^{d_A/(d_A+2)}\bigl(\sqrt{D/c_\mu}\operatorname{diam}(A)^{d_A/2}a_n\bigr)^{2/(d_A+2)}$, which gives (b).
The modulus bound is \cref{lem:modulus}.
(b$'$) is proved in \cref{app:smoothing2}, after \cref{lem:transfer2}.
(c) Each component of $\Xi_k$ has state marginal $\nu_S\leq D\lambda_S$, so $\nu_S\otimes\lambda_A\leq(D/c_\mu)\bar\mu_k$, and \cref{lem:nikolskii} gives $\int|e_k|\dd\Xi_k\leq\sqrt{\Lambda D/c_\mu}\|e_k\|_{2,\bar\mu_k}$.
The rest is as in (a).
For linear classes, if $h(s,\cdot)=\langle\vartheta,\phi(s,\cdot)\rangle$, Cauchy--Schwarz in the metric $G_s$ gives $h(s,a)^2\leq(\vartheta\T G_s\vartheta)(\phi(s,a)\T G_s^+\phi(s,a))$ for $\phi(s,a)$ in the range of $G_s$ (which contains all $a$ by continuity and full support), and $\vartheta\T G_s\vartheta=\int h(s,\cdot)^2\dd\lambda_A$.
\end{proofof}

\subsection{\texorpdfstring{Smoothed all-policy coefficient}{Smoothed all-policy coefficient}}\label{app:allpolicy}

\begin{definition}[Smoothed all-policy coefficient]\label{def:allpolicy}
A \emph{deterministic Markov policy} is a sequence $\mbg=(g_0,g_1,\dots)$ of Borel maps $g_t\colon S\to A$.
Its step-$t$ state--action marginal from $\rho$ is defined by $\nu^{\mbg}_0\defeq\rho\otimes\delta_{g_0}$ and $\nu^{\mbg}_{t+1}\defeq(\nu^{\mbg}_tP)\otimes\delta_{g_{t+1}}$.
For $\sigma\geq0$, the \emph{smoothed all-policy coefficient} $C_{\mathrm{all}}(\sigma)\in[0,\infty]$ of the design $(\bar\mu_k)_{k<K}$ is the infimum of the constants $C\geq0$ for which there is an event of probability one on which $(\nu^{\mbg}_t)^\sigma\leq C\,\bar\mu_k$ for all $k<K$, all $t\geq0$ and all deterministic Markov policies $\mbg$.
As usual, $\inf\emptyset\defeq\infty$.
\end{definition}
The definition uses one event for all $\mbg$ and $t$, so it does not require an uncountable intersection of events to be measurable, and $C_{\mathrm{all}}(\sigma)$ is a deterministic constant, like $\Cad(\sigma)$.
If $C_{\mathrm{all}}(\sigma)<\infty$, the infimum is attained, as intersecting the events of a sequence of admissible constants that decreases to $C_{\mathrm{all}}(\sigma)$ shows.
Then \cref{asm:coverage} holds with $\Cad(\sigma)=C_{\mathrm{all}}(\sigma)$ for every choice of selectors: by \cref{prop:propagation}, every component of $\Xi_k$ is some $\nu^{\mbg}_t$, smoothing is linear, and a countable convex combination of measures that are each at most $C\bar\mu_k$ is at most $C\bar\mu_k$.
For $\sigma=0$ and a fixed design $\bar\mu_k=\mu$, $C_{\mathrm{all}}(0)=\sup_{\mbg,t}\|\dd\nu^{\mbg}_t/\dd\mu\|_\infty$ is the coefficient $C$ of \cref{sec:setting}.
The proof of \cref{prop:sufficient} applies to every $\nu^{\mbg}_t$, so $C_{\mathrm{all}}(\sigma)\leq D/(c_\mu v_A(\sigma))$ under its conditions.
The design of \cref{ex:continuous} violates them, and there $\Cad(\sigma)<\infty$ for $\sigma\in(0,\frac14]$, while $C_{\mathrm{all}}(\sigma)=\infty$ for every $\sigma\geq0$.
Every greedy selector of an element of the class of that example is optimal, so the example separates the two coefficients, not the losses of FQI.

\subsection{\texorpdfstring{Trajectory data with a Gram-matrix transfer}{Trajectory data with a Gram-matrix transfer}}\label{app:trajectory-gram}

\begin{proposition}[Trajectory data with a Gram-matrix transfer]\label{prop:trajectory-gram}
Let $S=A=[0,1]$, $u(x)\defeq(1-x,x)\T$ and $\phi(s,a)\defeq u(s)\otimes u(a)\in[0,1]^4$.
Fix $m\in(0,1]$, probability measures $\nu_0,\nu_1$ on $[0,1]$ with Lebesgue densities at least $m$, and $\theta_r\in[0,1]^4$.
Let $P(\cdot\mid s,a)\defeq(1-a)\nu_0+a\nu_1$, $r\defeq\langle\theta_r,\phi\rangle$, $\rho\defeq\Unif[0,1]$, and $\calF\defeq\{\langle w,\phi\rangle:w\in W\}$ with $W\defeq[0,(1-\gamma)^{-1}]^4$ and $B\defeq(1-\gamma)^{-1}$.
The data form one trajectory that continues across rounds, i.e., $S_{k,1}=S'_{k-1,n}$ for $k\geq1$ and $S_{0,1}\sim\rho$ is $\calH_0$-measurable, and every transition is used.
Given $\calH_{k,i-1}$, the action $A_{k,i}$ has a law at least $c_A\lambda_A$, where $\lambda_A$ is the Lebesgue measure and $c_A\in(0,1]$.
The regression is solved exactly, $\eopt[,k]=0$.
Then:
(i) \cref{asm:model,asm:class,asm:adaptive-batches} hold, $\calF$ is Bellman complete, and every residual is linear, $e_k=\langle v_k,\phi\rangle$ with $v_k\in\dsR^4$.
(ii) For every $k$ and every $\sigma\geq0$, no multiple of $\bar\mu_k$ dominates $\Xi_k^\sigma$, so \cref{asm:coverage} fails at every scale.
(iii) With $\Sigma_k\defeq\int\phi\phi\T\dd\bar\mu_k$, $\int|e_k|\dd\nu\leq\lambda_{\min}(\Sigma_k)^{-1/2}\|e_k\|_{2,\bar\mu_k}$ for every $\nu\in\Delta(S\times A)$ whenever $\lambda_{\min}(\Sigma_k)>0$, and $\Pr(\lambda_{\min}(\Sigma_k)<c_Am/72)\leq6e^{-nm^2/288}$.
(iv) For every $\delta\in(0,1)$, with probability at least $1-\delta-6Ke^{-nm^2/288}$,
\begin{equation*}
\|V^*-V^{\widehat\pi_K}\|_{1,\rho}\leq\frac{480\sqrt2}{(1-\gamma)^3\sqrt{c_Am}}\sqrt{\frac{4\log(6n)+\log(6K/\delta)}{n}}+\frac{8\gamma^K}{(1-\gamma)^2}.
\end{equation*}
\end{proposition}
\begin{proofof}{\cref{prop:trajectory-gram}}
\emph{(i)} The features $\phi$ are continuous with $\sum_j\phi_j=1$ and $\|\phi\|=\|u(s)\|\,\|u(a)\|\leq1$, so $P$ is weakly continuous, $r$ is continuous and $R_{\max}=1$.
Also $|f_w-f_{w'}|\leq\|w-w'\|$ and $0\leq f_w\leq\max_jw_j\leq B$, so \cref{asm:class} holds with $\Theta=W$.
We index $\phi$ by pairs $(i,j)$ of a state and an action component, so that $\phi_{ij}(s,a)=u_i(s)u_j(a)$ and $P(\cdot\mid s,a)=\sum_ju_j(a)\nu_j$.
Since $\sum_iu_i(s)=1$, we have $T^*f_w=\langle w',\phi\rangle$ with $w'_{ij}=(\theta_r)_{ij}+\gamma\int V_{f_w}\,\dd\nu_j$, and $0\leq V_{f_w}\leq(1-\gamma)^{-1}$ gives $w'\in W$.
So $\calF$ is Bellman complete, and every residual is linear in $\phi$.
Since the trajectory continues, $S_{k,i}$ is $\calH_{k,i-1}$-measurable, and so $\mu_{k,i}=\delta_{S_{k,i}}\otimes\pi_{k,i}$, where $\pi_{k,i}\geq c_A\lambda_A$ is the conditional law of $A_{k,i}$.

\emph{(ii)} The state marginal of $\bar\mu_k$ is carried by $\{S_{k,1},\dots,S_{k,n}\}$, while $\rho$ and every $P(\cdot\mid s,a)$ have Lebesgue densities, so they are atomless and \cref{lem:obstructions}(ii) applies.

\emph{(iii)} For $v\in\dsR^4$ and $\nu\in\Delta(S\times A)$, the Cauchy--Schwarz inequality and $\|\phi\|\leq1$ give $\int|\langle v,\phi\rangle|\,\dd\nu\leq\|v\|\leq\lambda_{\min}(\Sigma_k)^{-1/2}(v\T\Sigma_kv)^{1/2}$, and $v_k\T\Sigma_kv_k=\|e_k\|_{2,\bar\mu_k}^2$.
For the eigenvalue, $\phi\phi\T=u(s)u(s)\T\otimes u(a)u(a)\T$ and the action floor give $\Sigma_k\succeq c_A\widehat M_k\otimes M$, where $\widehat M_k\defeq\frac1n\sum_iu(S_{k,i})u(S_{k,i})\T$, and $M\defeq\int_0^1u(x)u(x)\T\,\dd x=\begin{psmallmatrix}1/3&1/6\\1/6&1/3\end{psmallmatrix}$ has the eigenvalues $\frac12$ and $\frac16$.
The eigenvalues of a Kronecker product of positive semidefinite matrices are the products of their eigenvalues, so $\lambda_{\min}(\Sigma_k)\geq\frac{c_A}6\lambda_{\min}(\widehat M_k)$.
The $\sigma$-field just before $S_{k,i}$ is drawn is $\calH_{k-1,n-1}$ for $i=1$ and $k\geq1$, trivial for $S_{0,1}$, and $\calH_{k,i-2}$ for $i\geq2$.
Given this $\sigma$-field, $S_{k,i}$ has the law $\rho$ or a mixture of the kernels $P(\cdot\mid s,a)$, which is at least $m\,\Unif[0,1]$.
So the average of the conditional means of $u(S_{k,i})u(S_{k,i})\T$ is at least $mM$ in the positive semidefinite order, and its smallest eigenvalue is at least $m/6$.
The three distinct entries of $u(S_{k,i})u(S_{k,i})\T$ lie in $[0,1]$.
So the Azuma--Hoeffding inequality with $t=m/24$ and a union bound show that, with probability at least $1-6e^{-2nt^2}=1-6e^{-nm^2/288}$, each entry of $\widehat M_k$ deviates by at most $t$ from the average of its conditional means.
Then the deviation matrix has Frobenius norm, and so operator norm, at most $2t$, and Weyl's inequality gives $\lambda_{\min}(\widehat M_k)\geq m/6-2t=m/12$, so that $\lambda_{\min}(\Sigma_k)\geq c_Am/72$.

\emph{(iv)} By (iii) and a union bound over $k<K$, the event of \cref{lem:instances}(iii) with $\kappa=1$ and $\lambda_0=c_Am/72$ has probability at least $1-6Ke^{-nm^2/288}$.
\Cref{cor:fast} applies with $d=4$, $\kappa=1$, $W_0=2(1-\gamma)^{-1}$, $\iota=\log(6n)$, $\Bres=2(1-\gamma)^{-1}$ and $\eapp[,k]=\eopt[,k]=0$, so, with probability at least $1-\delta$, $\|e_k\|_{2,\bar\mu_k}\leq20(1-\gamma)^{-1}\sqrt{(4\log(6n)+\log(6K/\delta))/n}$ for all $k<K$.
On the intersection of this event with the event of \cref{lem:instances}(iii), \eqref{eq:main-path} with $C_{\mathrm{tr}}=72/(c_Am)$ and $\omega_{\mathrm{tr}}=0$ gives the claim, since $4\cdot\sqrt{72}\cdot20=480\sqrt2$ and $8B\gamma^K/(1-\gamma)=8\gamma^K/(1-\gamma)^2$.
The loss is at most $2B$ off the event of \cref{lem:instances}(iii), so the expectation bound \eqref{eq:main} holds as well, with the extra term $2B\cdot6Ke^{-nm^2/288}=12Ke^{-nm^2/288}/(1-\gamma)$.
The transfer of (iii) holds for every $\nu$, so it is an all-policy statement in the Gram sense, and the atomic design is no obstruction for this class.
\end{proofof}

\subsection{\texorpdfstring{Affine-reproducing smoothing}{Affine-reproducing smoothing}}\label{app:smoothing2}

\Cref{prop:tight} concerns residuals that are only Lipschitz in the action.
The smoothing is part of the analysis, not of the algorithm, so its kernel may be signed, and a kernel that reproduces affine functions pays only for the curvature of the residual.
\begin{definition}[Affine-reproducing smoothing]\label{def:smoothing2}
Let $A\subset\dsR^{d_A}$ be a convex body whose metric is induced by a norm $\|\cdot\|$, let $\lambda_A$ be the normalized Lebesgue measure on $A$, and fix $o\in A$ and $r_A>0$ with $\bar\dsB_{r_A}(o)\subseteq A$.
For $\sigma\in(0,\operatorname{diam}(A)+r_A]$ and $b\in A$, let $\bar\sigma\defeq\sigma/(\operatorname{diam}(A)+r_A)$, let $\dsB^{(1)}(b)$ and $\dsB^{(2)}(b)$ be the closed balls of radii $\bar\sigma r_A/2$ and $\bar\sigma r_A$ centered at $b+\frac{\bar\sigma}2(o-b)$ and $b+\bar\sigma(o-b)$, let $U^{(2)}_\sigma(\cdot\mid b)\defeq\frac23\Unif(\dsB^{(1)}(b))+\frac13\Unif(\dsB^{(2)}(b))$, and define $\nu^{\sigma,2}$ as $\nu^\sigma$ in \cref{def:smoothing}, with $U^{(2)}_\sigma$ in place of $U_\sigma$.
The \emph{affine modulus} of $h\in\calQ$ is
\begin{equation*}
\omega^{\mathrm{lin}}_h(\sigma)\defeq\sup_{s\in S,\,b\in A}\ \inf_{v\in\dsR^{d_A}}\ \sup_{a\in A,\ \|a-b\|\leq\sigma}\bigl|h(s,a)-h(s,b)-\langle v,a-b\rangle\bigr|,
\end{equation*}
and $\omega^{\mathrm{lin}}_\calF(\sigma)\defeq\sup_{\theta,\theta'\in\Theta}\omega^{\mathrm{lin}}_{f_\theta-T^*f_{\theta'}}(\sigma)$.
\end{definition}

\begin{lemma}[Affine-reproducing smoothing transfer]\label{lem:transfer2}
In the setting of \cref{def:smoothing2}, let $h\in\calQ$, $\sigma\in(0,\operatorname{diam}(A)+r_A]$ and ${\nu,\mu\in\Delta(S\times A)}$.
Then $\int|h|\,\dd\nu\leq3\int|h|\,\dd\nu^{\sigma,2}+3\,\omega^{\mathrm{lin}}_h(\sigma)$, and if $\nu^{\sigma,2}\leq C\mu$, then $\int|h|\,\dd\nu\leq3\sqrt C\,\|h\|_{2,\mu}+3\,\omega^{\mathrm{lin}}_h(\sigma)$.
Under (a) and (b) of \cref{prop:sufficient}, almost surely, $\Xi_k^{\sigma,2}\leq c^{(2)}\sigma^{-d_A}\,\bar\mu_k$ for every $k<K$ and every choice of selectors, where
\begin{equation*}
c^{(2)}\defeq\frac{2^{d_A+1}+1}{3}\,\frac{D}{c_\mu}\Bigl(\frac{\operatorname{diam}(A)(\operatorname{diam}(A)+r_A)}{r_A}\Bigr)^{d_A}.
\end{equation*}
\end{lemma}
The signed kernel $2\Unif(\dsB^{(1)}(b))-\Unif(\dsB^{(2)}(b))$ has mass one and mean $b$, so it reproduces affine functions of the action: it is a kernel of order~1 in the multivariate form of \citet[Definition~1.3]{tsybakov2009}.
Its total variation is at most $3U^{(2)}_\sigma(\cdot\mid b)$ (\cref{lem:kernel2}).
The kernel has to be signed at the boundary: if $b$ lies on the boundary of $A$, there is a nonzero $u\in\dsR^{d_A}$ with $\langle u,a-b\rangle\leq0$ for all $a\in A$, so every probability measure on $A$ with mean $b$ is carried by the Lebesgue-null set $\{a\in A:\langle u,a-b\rangle=0\}$ and has no density.
Taking $v=0$ shows $\omega^{\mathrm{lin}}_h\leq\omega_h$, and $\omega^{\mathrm{lin}}_h(\sigma)$ is of order $\sigma^2$ for residuals with Lipschitz gradients in the action.

\begin{lemma}[Affine-reproducing kernel]\label{lem:kernel2}
In the setting of \cref{def:smoothing2}, let $\sigma\in(0,\operatorname{diam}(A)+r_A]$ and $b\in A$, let $c_1(b)\defeq b+\frac{\bar\sigma}2(o-b)$ and $c_2(b)\defeq b+\bar\sigma(o-b)$ be the centers of $\dsB^{(1)}(b)$ and $\dsB^{(2)}(b)$, and let $U^{\pm}_\sigma(\cdot\mid b)\defeq2\Unif(\dsB^{(1)}(b))-\Unif(\dsB^{(2)}(b))$.
Then
(i) $\dsB^{(1)}(b)\cup\dsB^{(2)}(b)\subseteq A\cap\bar\dsB_\sigma(b)$;
(ii) $U^{\pm}_\sigma(A\mid b)=1$ and $\int(a-b)\,U^{\pm}_\sigma(\dd a\mid b)=0$;
(iii) $|U^{\pm}_\sigma(\cdot\mid b)|\leq3U^{(2)}_\sigma(\cdot\mid b)$, where $|U^{\pm}_\sigma(\cdot\mid b)|$ is the total-variation measure of $U^{\pm}_\sigma(\cdot\mid b)$;
(iv) $U^{(2)}_\sigma(\cdot\mid b)\leq\frac{2^{d_A+1}+1}3\bigl(\operatorname{diam}(A)(\operatorname{diam}(A)+r_A)/(r_A\sigma)\bigr)^{d_A}\lambda_A$;
(v) for every Borel set $E\subseteq S\times A$, the map $(s,b)\mapsto U^{(2)}_\sigma(E_s\mid b)$ is Borel, and $\nu^{\sigma,2}\in\Delta(S\times A)$ for every $\nu\in\Delta(S\times A)$.
\end{lemma}

\begin{proof}
\emph{(i)} The ball $\dsB^{(2)}(b)$ equals $(1-\bar\sigma)b+\bar\sigma\bar\dsB_{r_A}(o)$, and $\bar\sigma\in(0,1]$, so it lies in $A$ by convexity, since $b\in A$ and $\bar\dsB_{r_A}(o)\subseteq A$.
Every $a\in\dsB^{(2)}(b)$ satisfies $\|a-b\|\leq\|a-c_2(b)\|+\|c_2(b)-b\|\leq\bar\sigma r_A+\bar\sigma\|o-b\|\leq\bar\sigma(r_A+\operatorname{diam}(A))=\sigma$, because $o,b\in A$.
The same argument with $\bar\sigma/2$ in place of $\bar\sigma$ applies to $\dsB^{(1)}(b)$.

\emph{(ii)} The reflection $a\mapsto2c-a$ at the center $c$ of a closed norm ball preserves the Lebesgue measure and maps the ball onto itself, so the mean of the uniform law on the ball is $c$.
By (i), both balls lie in $A$, so $U^{\pm}_\sigma(A\mid b)=2-1=1$, and $\int(a-b)\,U^{\pm}_\sigma(\dd a\mid b)=2(c_1(b)-b)-(c_2(b)-b)=\bar\sigma(o-b)-\bar\sigma(o-b)=0$.

\emph{(iii)} The total-variation measure of a difference of two measures is at most their sum, so $|U^{\pm}_\sigma(\cdot\mid b)|\leq2\Unif(\dsB^{(1)}(b))+\Unif(\dsB^{(2)}(b))=3U^{(2)}_\sigma(\cdot\mid b)$.

\emph{(iv)} Let $\bar\dsB_t(c)\subseteq A$ be a closed ball of radius $t>0$.
Then $c\in A$, so $A\subseteq\bar\dsB_{\operatorname{diam}(A)}(c)$, and since the Lebesgue measure of a ball of a fixed norm scales with the $d_A$-th power of its radius, $\lambda_A(\bar\dsB_t(c))\geq(t/\operatorname{diam}(A))^{d_A}$.
So the uniform law on $\bar\dsB_t(c)$, which is $\lambda_A(\,\cdot\cap\bar\dsB_t(c))/\lambda_A(\bar\dsB_t(c))$, is at most $(\operatorname{diam}(A)/t)^{d_A}\lambda_A$.
The radii of $\dsB^{(1)}(b)$ and $\dsB^{(2)}(b)$ are $\bar\sigma r_A/2$ and $\bar\sigma r_A$, and $\operatorname{diam}(A)/(\bar\sigma r_A)=\operatorname{diam}(A)(\operatorname{diam}(A)+r_A)/(r_A\sigma)$, so
\begin{equation*}
U^{(2)}_\sigma(\cdot\mid b)\leq\Bigl(\frac23\,2^{d_A}+\frac13\Bigr)\Bigl(\frac{\operatorname{diam}(A)(\operatorname{diam}(A)+r_A)}{r_A\sigma}\Bigr)^{d_A}\lambda_A,
\end{equation*}
which is (iv).

\emph{(v)} The centers $c_1(b)$ and $c_2(b)$ are continuous in $b$, so $(s,b,a)\mapsto\dsone_E(s,a)\dsone\{\|a-c_j(b)\|\leq t_j\}$ is Borel for $j=1,2$, with $t_1\defeq\bar\sigma r_A/2$ and $t_2\defeq\bar\sigma r_A$.
By Tonelli's theorem, $(s,b)\mapsto\lambda_A(E_s\cap\dsB^{(j)}(b))$ is Borel, and the normalizing constants $\lambda_A(\dsB^{(j)}(b))$ do not depend on $b$, because the balls lie in $A$ and the Lebesgue measure is invariant under translations.
So $(s,b)\mapsto U^{(2)}_\sigma(E_s\mid b)$ is Borel, $\nu^{\sigma,2}(S\times A)=1$, and $\nu^{\sigma,2}$ is countably additive by monotone convergence.
\end{proof}

\begin{proofof}{\cref{lem:transfer2}}
Fix $(s,b)\in S\times A$ and $v\in\dsR^{d_A}$, and let $R_v(a)\defeq h(s,a)-h(s,b)-\langle v,a-b\rangle$ for $a\in A$.
By \cref{lem:kernel2}(ii), $\int h(s,a)\,U^{\pm}_\sigma(\dd a\mid b)=h(s,b)+\int R_v(a)\,U^{\pm}_\sigma(\dd a\mid b)$.
By \cref{lem:kernel2}(i) and (iii), the total-variation measure of $U^{\pm}_\sigma(\cdot\mid b)$ is concentrated on $A\cap\bar\dsB_\sigma(b)$ and has mass at most $3$, so
\begin{equation*}
|h(s,b)|\leq3\int|h(s,a)|\,U^{(2)}_\sigma(\dd a\mid b)+3\sup_{a\in A,\ \|a-b\|\leq\sigma}|R_v(a)| .
\end{equation*}
The left-hand side does not depend on $v$, so we may take the infimum over $v$ and bound it by $\omega^{\mathrm{lin}}_h(\sigma)$, which does not depend on $(s,b)$.
Integrating against $\nu(\dd s,\dd b)$ gives the first claim, since $\int\int|h(s,a)|\,U^{(2)}_\sigma(\dd a\mid b)\,\nu(\dd s,\dd b)=\int|h|\,\dd\nu^{\sigma,2}$ by the definition of $\nu^{\sigma,2}$ and \cref{lem:kernel2}(v).

If $\nu^{\sigma,2}\leq C\mu$, then $\nu^{\sigma,2}\ll\mu$ with $\dd\nu^{\sigma,2}/\dd\mu\leq C$, so $\int(\dd\nu^{\sigma,2}/\dd\mu)^2\,\dd\mu\leq C$, and the Cauchy--Schwarz inequality gives $\int|h|\,\dd\nu^{\sigma,2}\leq\sqrt C\,\|h\|_{2,\mu}$, as in the proof of \cref{lem:transfer}.
Together with the first claim, this is the second claim.

For the last claim, fix an outcome in the almost-sure event of (b), a round $k$ and a choice of selectors.
By \cref{prop:propagation}, $\Xi_k$ is a countable convex combination of measures $\nu_S\otimes\delta_g$ with $\nu_S\leq D\lambda_S$, by (a) (proof of \cref{prop:sufficient}).
For a Borel set $E\subseteq S\times A$,
\begin{align*}
(\nu_S\otimes\delta_g)^{\sigma,2}(E)&=\int U^{(2)}_\sigma(E_s\mid g(s))\,\nu_S(\dd s)\\
&\leq\frac{c^{(2)}c_\mu}{D\sigma^{d_A}}\int\lambda_A(E_s)\,\nu_S(\dd s) &\prstep{(\cref{lem:kernel2}(iv), definition of $c^{(2)}$)}\\
&\leq c^{(2)}\sigma^{-d_A}c_\mu(\lambda_S\otimes\lambda_A)(E) &\prstep{($\nu_S\leq D\lambda_S$)}\\
&\leq c^{(2)}\sigma^{-d_A}\bar\mu_k(E). &\prstep{((b), averaged over $i$)}
\end{align*}
The map $\nu\mapsto\nu^{\sigma,2}$ is linear and commutes with countable convex combinations, so the same bound holds for $\Xi_k^{\sigma,2}$.
\end{proofof}

\begin{proofof}{\cref{cor:rates}(b$'$)}
Fix $\sigma\in(0,\operatorname{diam}(A)+r_A]$.
Every residual $e_k=f_{\widehat\theta_{k+1}}-T^*f_{\widehat\theta_k}$ satisfies $\omega^{\mathrm{lin}}_{e_k}(\sigma)\leq\omega^{\mathrm{lin}}_\calF(\sigma)\leq L_\beta\sigma^\beta$.
On the almost-sure event of the last claim of \cref{lem:transfer2}, its second claim with $\nu=\Xi_k$, $\mu=\bar\mu_k$ and $C=c^{(2)}\sigma^{-d_A}$ gives $\int|e_k|\,\dd\Xi_k\leq3\sqrt{c^{(2)}}\,\sigma^{-d_A/2}\|e_k\|_{2,\bar\mu_k}+3L_\beta\sigma^\beta$ for every $k<K$ and every choice of selectors.
So \cref{asm:transfer} holds with $C_{\mathrm{tr}}=9c^{(2)}\sigma^{-d_A}$ and $\omega_{\mathrm{tr}}=3L_\beta\sigma^\beta$, and \eqref{eq:main} with $\veps_k=a_n$ and $\sum_k\gamma^{K-1-k}\leq1/(1-\gamma)$ gives
\begin{equation*}
\Ex{\|V^*-V^{\widehat\pi_K}\|_{1,\rho}}-\frac{8B\gamma^K}{1-\gamma}\leq\frac{12}{(1-\gamma)^2}\Bigl(\sqrt{c^{(2)}}\,\sigma^{-d_A/2}a_n+L_\beta\sigma^\beta\Bigr).
\end{equation*}
Let $\sigma_n\defeq(\sqrt{c^{(2)}}\,a_n/L_\beta)^{2/(d_A+2\beta)}$.
The condition on $n$ in (b$'$) is equivalent to $\sigma_n\leq\operatorname{diam}(A)+r_A$, so we may take $\sigma=\sigma_n$, and then both terms in the bracket equal $L_\beta^{d_A/(d_A+2\beta)}(\sqrt{c^{(2)}}\,a_n)^{2\beta/(d_A+2\beta)}$, which gives the factor $24$.
\end{proofof}

\begin{lemma}[Second-order action modulus]\label{lem:modulus2}
In the setting of \cref{def:smoothing2}, let $\|\cdot\|_*$ be the dual norm.
Suppose that, for all $s\in S$ and $\theta\in\Theta$, the functions $f_\theta(s,\cdot)$ and $r(s,\cdot)$ have continuously differentiable extensions to an open set containing $A$, with $\|\nabla_af_\theta(s,a)-\nabla_af_\theta(s,b)\|_*\leq L_{\calF,2}\|a-b\|$ and $\|\nabla_ar(s,a)-\nabla_ar(s,b)\|_*\leq L_{r,2}\|a-b\|$ for $a,b\in A$.
Suppose also that there are finite signed measures $\partial_1P(\cdot\mid s,b),\dots,\partial_{d_A}P(\cdot\mid s,b)$ on $S$ such that
\begin{equation*}
\Bigl\|P(\cdot\mid s,a)-P(\cdot\mid s,b)-\sum_{j=1}^{d_A}(a_j-b_j)\,\partial_jP(\cdot\mid s,b)\Bigr\|_{\mathrm{TV}}\leq\tfrac12L_{P,2}\|a-b\|^2\qquad\text{for }a,b\in A,
\end{equation*}
with $\|\cdot\|_{\mathrm{TV}}$ as in \cref{lem:modulus}.
Then $\omega^{\mathrm{lin}}_\calF(\sigma)\leq\frac12(L_{\calF,2}+L_{r,2}+\gamma BL_{P,2})\sigma^2$ for all $\sigma>0$.
\end{lemma}

\begin{proof}
Let $u$ be continuously differentiable on an open set containing $A$, with $\|\nabla u(x)-\nabla u(y)\|_*\leq L\|x-y\|$ for $x,y\in A$.
For $a,b\in A$ the segment from $b$ to $a$ lies in $A$ by convexity, and
\begin{align*}
|u(a)-u(b)-\langle\nabla u(b),a-b\rangle|&=\Bigl|\int_0^1\langle\nabla u(b+t(a-b))-\nabla u(b),a-b\rangle\,\dd t\Bigr|\\
&\leq\int_0^1Lt\,\|a-b\|^2\,\dd t=\frac L2\|a-b\|^2 .
\end{align*}
Fix $\theta,\theta'\in\Theta$, $s\in S$ and $b\in A$, let $h\defeq f_\theta-T^*f_{\theta'}$, so that $h(s,a)=f_\theta(s,a)-r(s,a)-\gamma\int V_{f_{\theta'}}\,\dd P(\cdot\mid s,a)$, and let
\begin{equation*}
v\defeq\nabla_af_\theta(s,b)-\nabla_ar(s,b)-\gamma\Bigl(\int V_{f_{\theta'}}\,\dd\partial_jP(\cdot\mid s,b)\Bigr)_{j=1}^{d_A}.
\end{equation*}
For $a\in A$, $h(s,a)-h(s,b)-\langle v,a-b\rangle$ is the sum of three terms: the first-order Taylor remainders of $f_\theta(s,\cdot)$ and of $-r(s,\cdot)$ at $b$, which are at most $\frac12L_{\calF,2}\|a-b\|^2$ and $\frac12L_{r,2}\|a-b\|^2$ in absolute value, and $-\gamma\int V_{f_{\theta'}}\,\dd m_{a,b}$, where $m_{a,b}$ is the signed measure in the condition on $P$.
Since $|V_{f_{\theta'}}|\leq B$, the last term is at most $\gamma B\|m_{a,b}\|_{\mathrm{TV}}\leq\frac12\gamma BL_{P,2}\|a-b\|^2$ in absolute value.
Taking the supremum over $a\in A$ with $\|a-b\|\leq\sigma$, and then over $s$, $b$, $\theta$ and $\theta'$, gives the claim.

The condition on $P$ holds, for instance, if $P(\dd s'\mid s,a)=p(s'\mid s,a)\,\lambda_S(\dd s')$ with smooth densities.
That is, for every $s$ and $\lambda_S$-almost every $s'$, $p(s'\mid s,\cdot)$ has a twice continuously differentiable extension to an open set containing $A$, with $\int\|\nabla_ap(s'\mid s,b)\|_*\,\lambda_S(\dd s')<\infty$ and $\int\sup_{x\in A}\|\nabla_a^2p(s'\mid s,x)\|\,\lambda_S(\dd s')\leq L_{P,2}$ for all $s$ and $b$, with the operator norm from $\|\cdot\|$ to $\|\cdot\|_*$ on Hessians.
Then $\partial_jP(\dd s'\mid s,b)\defeq\partial_{a_j}p(s'\mid s,b)\,\lambda_S(\dd s')$ are finite signed measures.
Taylor's formula with integral remainder gives, for $\lambda_S$-almost every $s'$,
\begin{align*}
&p(s'\mid s,a)-p(s'\mid s,b)-\langle\nabla_ap(s'\mid s,b),a-b\rangle\\
&\qquad=\int_0^1(1-t)\,\bigl\langle\nabla_a^2p(s'\mid s,b+t(a-b))(a-b),a-b\bigr\rangle\,\dd t,
\end{align*}
whose absolute value is at most $\frac12\sup_{x\in A}\|\nabla_a^2p(s'\mid s,x)\|\,\|a-b\|^2$.
Integrating in $s'$ bounds the total variation of $m_{a,b}$ by $\frac12L_{P,2}\|a-b\|^2$.
\end{proof}

\subsection{\texorpdfstring{State--action smoothing}{State--action smoothing}}\label{app:smoothingSA}

\begin{proofof}{\cref{lem:transferSA}}
With closed balls in $S$ and in $A$, $U_{\sigma_S,\sigma_A}(E\mid s,b)=(\lambda_S\otimes\lambda_A)(E\cap(\bar\dsB_{\sigma_S}(s)\times\bar\dsB_{\sigma_A}(b)))/(\lambda_S(\bar\dsB_{\sigma_S}(s))\,\lambda_A(\bar\dsB_{\sigma_A}(b)))$ for Borel sets $E\subseteq S\times A$.

\emph{Volumes.} Let $s\in S$ and $\sigma_S\in(0,\operatorname{diam}(S)]$.
The homothety $x\mapsto s+\frac{\sigma_S}{\operatorname{diam}(S)}(x-s)$ maps $S$ into $\bar\dsB_{\sigma_S}(s)$, by convexity and since $\|x-s\|\leq\operatorname{diam}(S)$ for $x\in S$, and it multiplies the Lebesgue measure by $(\sigma_S/\operatorname{diam}(S))^{d_S}$.
So $\lambda_S(\bar\dsB_{\sigma_S}(s))\geq(\sigma_S/\operatorname{diam}(S))^{d_S}$, and in the same way $\lambda_A(\bar\dsB_{\sigma_A}(b))\geq(\sigma_A/\operatorname{diam}(A))^{d_A}$ for $b\in A$ and $\sigma_A\in(0,\operatorname{diam}(A)]$, as in the proof of \cref{cor:rates}(b).
So the normalizing constant of $U_{\sigma_S,\sigma_A}(\cdot\mid s,b)$ is at least $(\sigma_S/\operatorname{diam}(S))^{d_S}(\sigma_A/\operatorname{diam}(A))^{d_A}>0$.

\emph{Measurability.} The map $(s,b,s',a')\mapsto\dsone_E(s',a')\dsone\{\varrho_S(s,s')\leq\sigma_S\}\dsone\{\varrho_A(b,a')\leq\sigma_A\}$ is Borel, so by Tonelli's theorem $(s,b)\mapsto(\lambda_S\otimes\lambda_A)(E\cap(\bar\dsB_{\sigma_S}(s)\times\bar\dsB_{\sigma_A}(b)))$ is Borel, and with $E=S\times A$ this covers the normalizing constant.
So $(s,b)\mapsto U_{\sigma_S,\sigma_A}(E\mid s,b)$ is Borel, and $\nu^{\sigma_S,\sigma_A}$ is a probability measure by monotone convergence.

\emph{The two inequalities.} If $(s',a')\in\bar\dsB_{\sigma_S}(s)\times\bar\dsB_{\sigma_A}(b)$, then $|h(s,b)|\leq|h(s',a')|+\omega^{\mathrm{SA}}_h(\sigma_S,\sigma_A)$ by the definition of the modulus.
The measure $U_{\sigma_S,\sigma_A}(\cdot\mid s,b)$ is concentrated on $\bar\dsB_{\sigma_S}(s)\times\bar\dsB_{\sigma_A}(b)$, so integrating this bound against it, and then against $\nu(\dd s,\dd b)$, gives the first inequality.
For the second, $(\lambda_S\otimes\lambda_A)(E\cap(\bar\dsB_{\sigma_S}(s)\times\bar\dsB_{\sigma_A}(b)))\leq(\lambda_S\otimes\lambda_A)(E)$, so the volume bounds give $U_{\sigma_S,\sigma_A}(E\mid s,b)\leq(\operatorname{diam}(S)/\sigma_S)^{d_S}(\operatorname{diam}(A)/\sigma_A)^{d_A}(\lambda_S\otimes\lambda_A)(E)$ for every $(s,b)$.
Integrating against $\nu$ gives the bound on $\nu^{\sigma_S,\sigma_A}$, for every $\nu$, atoms included.

\emph{Transfer condition.} Under (b) of \cref{prop:sufficient}, averaged over $i$, there is an event of probability one on which $\bar\mu_k\geq c_\mu(\lambda_S\otimes\lambda_A)$ for every $k<K$.
Fix an outcome in this event and any choice of selectors, and let $C_{\mathrm{tr}}\defeq c_\mu^{-1}(\operatorname{diam}(S)/\sigma_S)^{d_S}(\operatorname{diam}(A)/\sigma_A)^{d_A}$.
\Cref{prop:propagation}, which uses only \cref{asm:model}, gives $\Xi_k\in\Delta(S\times A)$, and the second inequality gives $\Xi_k^{\sigma_S,\sigma_A}\leq C_{\mathrm{tr}}\bar\mu_k$.
So $\int(\dd\Xi_k^{\sigma_S,\sigma_A}/\dd\bar\mu_k)^2\,\dd\bar\mu_k\leq C_{\mathrm{tr}}$, and the Cauchy--Schwarz inequality gives $\int|e_k|\,\dd\Xi_k^{\sigma_S,\sigma_A}\leq\sqrt{C_{\mathrm{tr}}}\,\|e_k\|_{2,\bar\mu_k}$.
The residual $e_k=f_{\widehat\theta_{k+1}}-T^*f_{\widehat\theta_k}$ satisfies $\omega^{\mathrm{SA}}_{e_k}(\sigma_S,\sigma_A)\leq\omega^{\mathrm{SA}}_\calF(\sigma_S,\sigma_A)$, so the first inequality with $\nu=\Xi_k$ and $h=e_k$ gives \eqref{eq:transfer} with the stated constants, for every $k<K$.
No property of $\rho$, of $P$ beyond \cref{asm:model}, or of the selectors was used.
\end{proofof}

\begin{corollary}[Rates without a density ceiling]\label{cor:ratesSA}
In the setting of \cref{def:smoothingSA}, assume the conditions of \cref{cor:rates}, except (a) of \cref{prop:sufficient}, and suppose that $\omega^{\mathrm{SA}}_\calF(\sigma_S,\sigma_A)\leq L_S\sigma_S+L_A\sigma_A$ for all $\sigma_S\in(0,\operatorname{diam}(S)]$ and $\sigma_A\in(0,\operatorname{diam}(A)]$, with $L_S,L_A>0$.
Let $t_n\defeq\bigl((\operatorname{diam}(S)L_S)^{d_S/2}(\operatorname{diam}(A)L_A)^{d_A/2}a_n/(2\sqrt{c_\mu})\bigr)^{2/(d_S+d_A+2)}$.
If $n$ is large enough that $t_n\leq\min\{\operatorname{diam}(S)L_S,\operatorname{diam}(A)L_A\}$, then $\Ex{\|V^*-V^{\widehat\pi_K}\|_{1,\rho}}-8B\gamma^K/(1-\gamma)\leq16\,t_n/(1-\gamma)^2$.
\end{corollary}

\begin{proofof}{\cref{cor:ratesSA}}
Fix $t\in(0,\min\{\operatorname{diam}(S)L_S,\operatorname{diam}(A)L_A\}]$, and let $\sigma_S\defeq t/L_S\leq\operatorname{diam}(S)$ and $\sigma_A\defeq t/L_A\leq\operatorname{diam}(A)$.
By \cref{lem:transferSA}, \cref{asm:transfer} holds with $\omega_{\mathrm{tr}}=\omega^{\mathrm{SA}}_\calF(\sigma_S,\sigma_A)\leq L_S\sigma_S+L_A\sigma_A=2t$, and the definition of $t_n$ gives
\begin{equation*}
\sqrt{C_{\mathrm{tr}}}\,a_n=c_\mu^{-1/2}\Bigl(\frac{\operatorname{diam}(S)L_S}t\Bigr)^{d_S/2}\Bigl(\frac{\operatorname{diam}(A)L_A}t\Bigr)^{d_A/2}a_n=2\,t_n^{(d_S+d_A+2)/2}\,t^{-(d_S+d_A)/2}.
\end{equation*}
Then \eqref{eq:main}, with $\veps_k=a_n$ and $\sum_k\gamma^{K-1-k}\leq1/(1-\gamma)$, gives
\begin{equation*}
\Ex{\|V^*-V^{\widehat\pi_K}\|_{1,\rho}}-\frac{8B\gamma^K}{1-\gamma}\leq\frac4{(1-\gamma)^2}\Bigl(2\,t_n^{(d_S+d_A+2)/2}\,t^{-(d_S+d_A)/2}+2t\Bigr).
\end{equation*}
The condition on $n$ makes $t=t_n$ admissible, and there the bracket equals $4t_n$, so the right-hand side equals $16t_n/(1-\gamma)^2$.
\end{proofof}

\begin{lemma}[Lipschitz state--action modulus]\label{lem:modulusSA}
In the setting of \cref{def:smoothingSA}, suppose that, for all $\theta\in\Theta$ and all $(s,a),(s',a')\in S\times A$,
\begin{align*}
|f_\theta(s,a)-f_\theta(s',a')|&\leq L_{\calF,S}\,\varrho_S(s,s')+L_{\calF,A}\,\varrho_A(a,a'),\\
|r(s,a)-r(s',a')|&\leq L_{r,S}\,\varrho_S(s,s')+L_{r,A}\,\varrho_A(a,a'),
\end{align*}
and $W_1(P(\cdot\mid s,a),P(\cdot\mid s',a'))\leq L_{P,S}\,\varrho_S(s,s')+L_{P,A}\,\varrho_A(a,a')$, where $W_1$ is the Wasserstein-1 distance on $\Delta(S)$ for $\varrho_S$.
Then, for all $\sigma_S,\sigma_A>0$,
\begin{equation*}
\omega^{\mathrm{SA}}_\calF(\sigma_S,\sigma_A)\leq(L_{\calF,S}+L_{r,S}+\gamma L_{\calF,S}L_{P,S})\,\sigma_S+(L_{\calF,A}+L_{r,A}+\gamma L_{\calF,S}L_{P,A})\,\sigma_A .
\end{equation*}
For deterministic dynamics $P(\cdot\mid s,a)=\delta_{F(s,a)}$, the condition on $P$ reads $\varrho_S(F(s,a),F(s',a'))\leq L_{P,S}\,\varrho_S(s,s')+L_{P,A}\,\varrho_A(a,a')$.
\end{lemma}

\begin{proof}
Fix $\theta,\theta'\in\Theta$, and let $h\defeq f_\theta-T^*f_{\theta'}=f_\theta-r-\gamma\int V_{f_{\theta'}}\,\dd P$.
As a maximum of $L_{\calF,S}$-Lipschitz functions of the state, $V_{f_{\theta'}}$ is $L_{\calF,S}$-Lipschitz, and it is bounded by $B$.
Since $S$ is compact, the Kantorovich--Rubinstein duality gives
\begin{equation*}
\Bigl|\int V_{f_{\theta'}}\,\dd\bigl(P(\cdot\mid s,a)-P(\cdot\mid s',a')\bigr)\Bigr|\leq L_{\calF,S}\,W_1\bigl(P(\cdot\mid s,a),P(\cdot\mid s',a')\bigr).
\end{equation*}
With the triangle inequality and the three conditions, $|h(s,a)-h(s',a')|$ is at most $(L_{\calF,S}+L_{r,S}+\gamma L_{\calF,S}L_{P,S})\varrho_S(s,s')+(L_{\calF,A}+L_{r,A}+\gamma L_{\calF,S}L_{P,A})\varrho_A(a,a')$.
Taking the supremum over $\varrho_S(s,s')\leq\sigma_S$ and $\varrho_A(a,a')\leq\sigma_A$, and then over $\theta$ and $\theta'$, gives the claim.
For Dirac measures, the only coupling of $\delta_x$ and $\delta_y$ is $\delta_{(x,y)}$, so $W_1(\delta_x,\delta_y)=\varrho_S(x,y)$, which gives the last statement.
\end{proof}
\Cref{lem:modulusSA} holds for every class, complete or not.
By the proof, $T^*$ maps functions that are $L_{\calF,S}$-Lipschitz in the state to functions with the same property if $L_{r,S}+\gamma L_{\calF,S}L_{P,S}\leq L_{\calF,S}$, which requires $\gamma L_{P,S}<1$ when $L_{r,S}>0$.
For expansive dynamics, $Q^*$ need not be Lipschitz in the state.
For example, with $S=[0,1]$, $F(s,a)\defeq\min\{2s,1\}$, $r(s,a)\defeq s$ and $\gamma\geq\frac12$, we have $Q^*(s,a)=\sum_{t\geq0}\gamma^t\min\{2^ts,1\}\geq s\,(1+\lfloor\log_2(1/s)\rfloor)$ for $s\in(0,1]$, while $Q^*(0,a)=0$, so $Q^*(\cdot,a)$ is not Lipschitz at $0$.
In such cases the approximation error $\eapp[,k]$ of a Lipschitz class need not vanish.

\subsection{Details of \texorpdfstring{\cref{ex:continuous}}{Example \ref{ex:continuous}}}\label{app:example}

\begin{example}[Continuous states and actions, adaptive exploration]\label{ex:continuous}
Let $S=G\cup G'$ with $G=[0,1]$ and $G'=[2,3]$, $A=[0,1]$, and $\varphi(a)=\min\{1,\max\{0,4a-2\}\}$.
From $s\in G$, action $a$ leads to $\Unif(G)$ with probability $1-\varphi(a)$ and to $\Unif(G')$ with probability $\varphi(a)$; states in $G'$ lead to $\Unif(G')$.
The reward is $r(s,a)=\dsone\{s\in G\}$ and $\rho=\Unif(G)$.
Let $f_\theta(s,a)=\theta_1\dsone_G(s)+\theta_2\dsone_{G'}(s)-\theta_3\varphi(a)\dsone_G(s)$ with $\theta\in\Theta=\{\theta_1-\theta_2\geq1,\ \gamma\leq\theta_3\leq2\gamma/(1-\gamma),\ |\theta_1|,|\theta_2|\leq(1-\gamma)^{-1}\}$.
We take $B\defeq2\sqrt3/(1-\gamma)$ in \cref{asm:class}.
Data are collected with restarts $S_{k,i}\sim\Unif(G)$ and any history-dependent behavior rule whose action law on $G$ has density at least $c$ on $[0,\frac34]$.
An example is $(1-c)$ times a Gaussian centered at the current greedy action and truncated to $[0,\frac34]$, plus $c\,\Unif[0,\frac34]$.
Then:
\begin{itemize}
\item $\calF$ is Bellman complete and $Q^*\in\calF$, so $\eapp[,k]=0$, and $\omega_\calF(\sigma)\leq8\gamma\sigma/(1-\gamma)$.
\item $\Cad(\sigma)\leq1/(c\sigma)$ for $\sigma\in(0,\frac14]$.
If the behavior rules draw the action on $G$ from atomless laws, as the truncated-Gaussian rule does, the design is atomless in the action and $\Cad(0)=\infty$.
\item The smoothed all-policy coefficient (\cref{def:allpolicy}) is infinite for every $\sigma\geq0$ and every admissible behavior rule: the policy $a\equiv1$ reaches $G'$, where the design has no mass.
\end{itemize}
\Cref{thm:main} therefore gives $O(n^{-1/6})$ with the residual bound of \cref{prop:hilbert}, and $O((\log n/n)^{1/3})$ with that of \cref{cor:fast} ($d=3$).
Since the residuals are affine in $(1,\varphi(a))$ on $G$, the norm-equivalence argument behind \cref{cor:rates}(c), applied directly with the reference measure $\Unif[0,\frac34]$ (\cref{lem:nikolskii}), improves these to $O(n^{-1/4})$ and $O(\sqrt{\log n/n})$.
Every greedy selector of an element of $\calF$ is optimal here, so the example separates the two coefficients, not the losses of FQI.
\end{example}

\emph{Model.} $G$ and $G'$ are clopen in $S$, so $\dsone_G$ is continuous.
On $G\times A$, $P(\cdot\mid s,a)=(1-\varphi(a))\Unif(G)+\varphi(a)\Unif(G')$ depends continuously on $a$ and not on $s$, and on $G'\times A$ it is constant.
So \cref{asm:model} holds, as do the conditions of \cref{lem:modulus} with $L_r=0$ and $L_P=8$, since $\|\Unif(G)-\Unif(G')\|_{\mathrm{TV}}=2$ and $\varphi$ is $4$-Lipschitz.

\emph{Optimal values and greedy actions.} From $G'$ no reward is ever collected, so $Q^*=0$ on $G'\times A$.
On $G$, $Q^*(s,a)=1+\gamma(1-\varphi(a))V_G$ with $V_G=1+\gamma V_G$, so $Q^*=\frac1{1-\gamma}-\frac\gamma{1-\gamma}\varphi(a)$, i.e., $\theta^*=(\frac1{1-\gamma},0,\frac\gamma{1-\gamma})\in\Theta$.
For every $\theta\in\Theta$, $\theta_3\geq\gamma>0$, so $\argmax_af_\theta(s,a)=\{a:\varphi(a)=0\}=[0,\frac12]$ for $s\in G$.
Every greedy selector of $Q^*$ or of an element of $\calF$ therefore maps $G$ into $[0,\frac12]$.

\emph{Bellman completeness.} $V_{f_\theta}=\theta_1$ on $G$ and $\theta_2$ on $G'$.
So $T^*f_\theta=f_{\theta'}$ with $\theta'=(1+\gamma\theta_1,\gamma\theta_2,\gamma(\theta_1-\theta_2))$, and $\theta'\in\Theta$:
\begin{itemize}
\item $\theta'_1-\theta'_2=1+\gamma(\theta_1-\theta_2)\geq1$;
\item $\gamma\leq\theta'_3\leq2\gamma/(1-\gamma)$, because $1\leq\theta_1-\theta_2\leq2/(1-\gamma)$;
\item $|\theta'_1|\leq1+\gamma/(1-\gamma)=1/(1-\gamma)$ and $|\theta'_2|\leq\gamma/(1-\gamma)$.
\end{itemize}
So $T^*\calF\subseteq\calF$ and $\eapp[,k]=0$ for every design.
The class is linear in $\theta$ with continuous features $\psi=(\dsone_G,\dsone_{G'},-\varphi\dsone_G)$, $\|\psi\|\leq\sqrt2$, and $\Theta$ is compact, so \cref{asm:class} and \cref{prop:hilbert} apply with $W_0\leq\sqrt6/(1-\gamma)$.
A residual $f_\theta-f_{\theta''}$ depends on $a$ only through $-(\theta_3-\theta''_3)\varphi(a)\dsone_G(s)$ with $|\theta_3-\theta''_3|\leq2\gamma/(1-\gamma)$, so $\omega_\calF(\sigma)\leq8\gamma\sigma/(1-\gamma)$.

\emph{Choice of $B$.} On $G$, $f_\theta(s,a)=\theta_1-\theta_3\varphi(a)$ lies between $\theta_1-\theta_3$ and $\theta_1\leq(1-\gamma)^{-1}$, and $\theta_1-\theta_3\geq\theta_2+1-\theta_3\geq-(1-\gamma)^{-1}+1-2\gamma(1-\gamma)^{-1}=-3\gamma/(1-\gamma)$, while $|f_\theta|=|\theta_2|\leq(1-\gamma)^{-1}$ on $G'$.
So $\sup_\theta\|f_\theta\|_\infty\leq\max\{1,3\gamma\}/(1-\gamma)$.
For $\theta\in\Theta$ we have $\|\theta\|^2\leq(2+4\gamma^2)/(1-\gamma)^2\leq6/(1-\gamma)^2$, and $\|\psi\|\leq\sqrt2$.
So $B\defeq W_0\kappa=2\sqrt3/(1-\gamma)$, with $W_0=\sqrt6/(1-\gamma)$ and $\kappa=\sqrt2$, meets \cref{prop:hilbert}, and $B\geq\max\{1,3\gamma\}/(1-\gamma)\geq R_{\max}/(1-\gamma)$ meets \cref{asm:class}.
In \cref{cor:fast}, $\kappa W_0/B=1$ and $\iota=\log(3n)$, and $\Bres=1+2\sqrt3(1+\gamma)/(1-\gamma)$.

\emph{Coverage.} Take $\lambda_A=\Unif[0,1]$.
Every component of $\Xi_k$ starts from $\rho=\Unif(G)$ and chooses actions in $[0,\frac12]$ on $G$, where $\varphi=0$; by induction its state marginal is $\Unif(G)$ and its actions lie in $[0,\frac12]$.
For $\sigma\leq\frac14$ and $b\in[0,\frac12]$, $U_\sigma(\cdot\mid b)$ is uniform on an interval of length at least $\sigma$ contained in $[0,\frac34]$, so its Lebesgue density is at most $1/\sigma$.
The design has state marginal $\Unif(G)$ and, on $G$, action density at least $c$ on $[0,\frac34]$.
Comparing densities on $G\times[0,\frac34]$ gives $\Xi_k^\sigma\leq(c\sigma)^{-1}\bar\mu_k$, i.e., $\Cad(\sigma)\leq1/(c\sigma)$.
If the behavior rules draw the action on $G$ from atomless laws given $\calH_{k,i-1}\vee\sigma(S_{k,i})$, as the example rule does, then $\bar\mu_k(\{(s,g(s)):s\in S\})=0$ for every Borel $g$ by Fubini, and $\Cad(0)=\infty$ by \cref{lem:obstructions}(i).
For the stationary policy $a\equiv1$, the step-$1$ state marginal is $\Unif(G')$, so its smoothed step-$1$ marginal gives mass one to $G'\times A$ for every $\sigma\geq0$, while every $\mu_{k,i}$ has state marginal $\Unif(G)$ and gives $G'\times A$ mass zero.
So $C_{\mathrm{all}}(\sigma)=\infty$ for every $\sigma\geq0$ (\cref{def:allpolicy}), for every admissible behavior rule.
\Cref{thm:main} with $\veps_k=a_n$ as in \cref{cor:rates} and $\sigma\asymp a_n^{2/3}$ gives $\Ex{\|V^*-V^{\widehat\pi_K}\|_{1,\rho}}=O(a_n^{2/3})+8B\gamma^K/(1-\gamma)$, which is $O(n^{-1/6})$ with \cref{prop:hilbert} and $O((\log n/n)^{1/3})$ with \cref{cor:fast} (here $d=3$ and $\eapp[,k]=0$, and we take exact minimization, $\eopt[,k]=0$).

\emph{Norm-equivalence route.} Take $\lambda'_A=\Unif[0,\frac34]$.
On $G$ every residual has the form $h(s,a)=\alpha_0+\alpha_1\varphi(a)$, and on $G'$ it is constant.
Under $\lambda'_A$, $\E\varphi=\frac16$ and $\E\varphi^2=\frac19$, so the Gram matrix of $(1,\varphi)$ is $\begin{psmallmatrix}1&1/6\\1/6&1/9\end{psmallmatrix}$ with inverse $\begin{psmallmatrix}4/3&-2\\-2&12\end{psmallmatrix}$.
So $\sup_ah^2\leq\sup_{\varphi\in[0,1]}(\frac43-4\varphi+12\varphi^2)\int h^2\dd\lambda'_A=\frac{28}3\int h^2\dd\lambda'_A$.
The design satisfies $\bar\mu_k\geq\frac{3c}4\,\Unif(G)\otimes\lambda'_A$, and every component of $\Xi_k$ has state marginal $\Unif(G)$.
\Cref{lem:nikolskii} therefore gives $\int|e_k|\dd\Xi_k\leq\sqrt{112/(9c)}\,\|e_k\|_{2,\bar\mu_k}$.
\Cref{thm:residual,prop:hilbert} then yield $O(n^{-1/4})$, and \cref{cor:fast} yields $O(\sqrt{\log n/n})$.

\section{\texorpdfstring{Lower bounds}{Lower bounds}}\label{app:lower}

\begin{theorem}[Minimax lower bounds, full form]\label{thm:lower-full}
Let $L>0$, $\gamma\in(0,1)$, $d_A\geq1$ and $A\defeq[0,1]^{d_A}$, and let all cubes carry the $\ell_\infty$ norm $\|\cdot\|$.
For $\sigma\in(0,\frac18]$, let $\calV$ be the set of points of $[\sigma,1-\sigma]^{d_A}$ whose coordinates are odd multiples of $\sigma$.
The MDPs $P_v$, $v\in\calV$, below have the transition kernel $P(\cdot\mid s,a)=\delta_{s_0}$ and Bernoulli rewards with mean $\frac12+h_v(s,a)$.
A \emph{learner} observes $N$ transitions $(S_t,A_t,R_t,S'_t)$, where it draws $S_t$ and then $A_t$ from laws that may depend on everything observed before and on its internal randomness.
It outputs a policy $\widehat\pi\in\Pi$ whose action laws depend measurably on the state, the data and the internal randomness.
We write $\E_v$ for the expectation under $P_v$, and $V^*$ and $V^{\widehat\pi}$ for the values in $P_v$.
\begin{enumerate}[label=(\alph*)]
\item \emph{Actions.} Let $\beta\in(0,1]$ and $\sigma=\sigma_N\defeq(3/(64\cdot4^{d_A}L^2N))^{1/(d_A+2\beta)}$, and suppose that $\sigma_N\leq\frac18$ and $L\sigma_N^\beta\leq\frac14$.
Let $S=\{s_0\}$, $\rho=\delta_{s_0}$ and $h_v(s_0,a)\defeq L\sigma_N^\beta(1-\|a-v\|/\sigma_N)_+$.
Then every learner satisfies
\begin{equation*}
\max_{v\in\calV}\Ep{v}{\|V^*-V^{\widehat\pi}\|_{1,\rho}}\geq\frac1{2(1-\gamma)}\Bigl(\frac3{64\cdot4^{d_A}}\Bigr)^{\beta/(d_A+2\beta)}L^{d_A/(d_A+2\beta)}N^{-\beta/(d_A+2\beta)} .
\end{equation*}
The same holds for $\beta\in(1,2]$, with $h_v(s_0,a)\defeq L\sigma_N^\beta\prod_{i=1}^{d_A}(1-(a_i-v_i)^2/\sigma_N^2)_+^2$.
\item \emph{States and actions, restart data.} Let $d_S\geq1$ and $M_S>0$, let $S=[0,1]^{d_S}$, let $s_0$ be its center and $\rho$ the uniform law on $S$.
Let $\sigma=\sigma_N\defeq(3/(64\cdot2^{d_S}4^{d_A}M_SL^2N))^{1/(d_S+d_A+2)}$, let $h_v(s,a)\defeq L\sigma_N(1-\max\{\|s-s_0\|,\|a-v\|\}/\sigma_N)_+$, and suppose that $\sigma_N\leq\frac18$ and $L\sigma_N\leq\frac14$.
Then every learner whose conditional law of each $S_t$, given everything observed before and its internal randomness, has density at most $M_S$ with respect to $\rho$ satisfies
\begin{equation*}
\max_{v\in\calV}\Ep{v}{\|V^*-V^{\widehat\pi}\|_{1,\rho}}\geq\frac\gamma{2(1-\gamma)}\Bigl(\frac{3L^{d_S+d_A}}{64\cdot2^{d_S}4^{d_A}M_SN}\Bigr)^{1/(d_S+d_A+2)} .
\end{equation*}
\end{enumerate}
\end{theorem}

The hard families are those of the lower bounds for continuum-armed bandits: Bernoulli rewards with mean $\frac12$ plus a bump hidden in one of many disjoint balls.
\citet{bubeck2011xarmed} prove their minimax lower bound on the cumulative regret (Theorem~13, proof in their Appendix~A.6) with mean rewards $\frac12$ plus a tent whose height equals the radius of one of many disjoint balls.
This is the tent of \cref{thm:lower-full}(a) with $\beta=L=1$.
\citet{kleinberg2004nearly} proves the regret lower bound for one-dimensional actions with a nested family of bumps (Theorem~4.2), and \citet{kleinberg2008multi} nest such bumps over infinitely many scales in their needle-in-a-haystack construction.
For the simple regret of noisy optimization with adaptive queries, the rate of \cref{thm:lower}(a) is known.
\citet{wang2018optimization} obtain the lower bound $N^{-\beta/(d_A+2\beta)}$ after $N$ queries over $\beta$-H\"older functions on $[0,1]^{d_A}$ as a consequence of their local minimax lower bound (Theorem~2 and the discussion after it).
\citet{locatelli2018adaptivity} prove the same rate over a class of locally smooth functions (Theorem~1, with margin exponent $0$).

Throughout, $\sigma\defeq\sigma_N$, and $\dsB_\sigma(v)\defeq\{a\in A:\|a-v\|<\sigma\}$ is the open ball around $v\in\calV$.
With $m\defeq\lfloor1/(2\sigma)\rfloor$, the set $\calV$ equals $\{(2j-1)\sigma:1\leq j\leq m\}^{d_A}$, so $|\calV|=m^{d_A}$.
Two distinct points of $\calV$ differ by at least $2\sigma$ in some coordinate, so the balls $\dsB_\sigma(v)$, $v\in\calV$, are pairwise disjoint, and they lie in $A$.
Since $\sigma\leq\frac18$, we have $1/(2\sigma)\geq4$, so $m\geq4$, and $m\geq1/(4\sigma)$ because $\lfloor x\rfloor\geq x/2$ for $x\geq1$; so $|\calV|\geq\max\{4,(4\sigma)^{-d_A}\}$.
Each bump of \cref{thm:lower-full} is positive exactly on the open ball of radius $\sigma$ around its center (for the smooth bump, $(1-t^2)_+^2>0$ exactly for $|t|<1$), and it attains its maximum, $L\sigma^\beta$ in (a) and $L\sigma$ in (b), at the center.
So $0\leq h_v\leq L\sigma^\beta\dsone_{\dsB_\sigma(v)}$ in (a), and $0\leq h_v(s,a)\leq L\sigma\,\dsone\{\|s-s_0\|<\sigma\}\,\dsone_{\dsB_\sigma(v)}(a)$ in (b).

\begin{lemma}[Divergence decomposition]\label{lem:divergence}
Let $P_0$ and $P_v$ be MDPs with the same state space, action space and transition kernel, and with Bernoulli rewards with means $\frac12$ and $\frac12+h_v$, where $h_v\colon S\times A\to[0,\frac14]$ is Borel.
The rewards are drawn independently of the next state given the state and the action.
For a learner as in \cref{thm:lower}, let $\Pr_0$ and $\Pr_v$ be the joint laws of its internal randomness and its data under $P_0$ and $P_v$, and let $\E_0$ be the expectation under $\Pr_0$.
Then
\begin{equation*}
\KLdiv{\Pr_0}{\Pr_v}\leq\frac83\,\Ep{0}{\sum_{t=1}^Nh_v(S_t,A_t)^2}.
\end{equation*}
\end{lemma}

\begin{proof}
Let $U$ be the internal randomness and $Z_t\defeq(S_t,A_t,R_t,S'_t)$.
Under both models, $U$ has the same law, and given $U$ and $Z_1,\dots,Z_{t-1}$ the learner draws $S_t$ and then $A_t$ from the same kernels; this covers restarts, continued trajectories and every other way of choosing the input states.
Given $(S_t,A_t)$, the next state $S'_t$ is drawn from the same transition kernel in both models, independently of $R_t$, and only the law of $R_t\in\{0,1\}$ differs.
With the Bernoulli mass function $\Ber(r\mid p)\defeq p^r(1-p)^{1-r}$, $\Pr_0$ has the density $\prod_{t=1}^N\Ber(R_t\mid\frac12)/\Ber(R_t\mid\frac12+h_v(S_t,A_t))$ with respect to $\Pr_v$, and
\begin{align*}
\KLdiv{\Pr_0}{\Pr_v}&=\sum_{t=1}^N\Ep{0}{\log\frac{\Ber(R_t\mid\frac12)}{\Ber(R_t\mid\frac12+h_v(S_t,A_t))}}\\
&=\sum_{t=1}^N\Ep{0}{\KLdiv{\Ber(\tfrac12)}{\Ber(\tfrac12+h_v(S_t,A_t))}},
\end{align*}
where the second equality conditions on $U$, $Z_1,\dots,Z_{t-1}$, $S_t$ and $A_t$, given which $R_t\sim\Ber(\frac12)$ under $\Pr_0$.
For $x\in[0,\frac14]$,
\begin{equation*}
\KLdiv{\Ber(\tfrac12)}{\Ber(\tfrac12+x)}=-\frac12\log(1-4x^2)\leq\frac12\cdot\frac{4x^2}{1-4x^2}\leq\frac83x^2,
\end{equation*}
because $-\log(1-y)\leq y/(1-y)$ for $y\in[0,1)$ and $4x^2\leq\frac14$.
Inserting this bound with $x=h_v(S_t,A_t)$ proves the claim.
\end{proof}

\begin{proofof}{\cref{thm:lower}(a)}
\emph{Value gap.} Fix $v\in\calV$.
In the one-state MDP $P_v$, every policy $\pi$ has $Q^\pi(s_0,a)=\frac12+h_v(s_0,a)+\gamma V^\pi(s_0)$ by \cref{thm:fixedpoints}, so $V^\pi(s_0)=(\frac12+\int h_v(s_0,a)\,\pi(\dd a\mid s_0))/(1-\gamma)$, and $V^*(s_0)=(\frac12+L\sigma^\beta)/(1-\gamma)$, attained by the action $v$.
Since $h_v\leq L\sigma^\beta\dsone_{\dsB_\sigma(v)}$, every outcome satisfies
\begin{equation*}
\|V^*-V^{\widehat\pi}\|_{1,\rho}=V^*(s_0)-V^{\widehat\pi}(s_0)\geq\frac{L\sigma^\beta}{1-\gamma}\bigl(1-\widehat\pi(\dsB_\sigma(v)\mid s_0)\bigr).
\end{equation*}

\emph{Information.} Let $P_0$ be the MDP with the same states, actions and transitions and with reward mean $\frac12$, let $\Pr_0$ and $\E_0$ refer to the learner on $P_0$, and let $N_v\defeq|\{t\leq N:A_t\in\dsB_\sigma(v)\}|$.
Since $h_v^2\leq L^2\sigma^{2\beta}\dsone_{\dsB_\sigma(v)}$ and $L\sigma^\beta\leq\frac14$, \cref{lem:divergence} gives $\KLdiv{\Pr_0}{\Pr_v}\leq\frac83L^2\sigma^{2\beta}\Ep{0}{N_v}$.
The quantity $\widehat\pi(\dsB_\sigma(v)\mid s_0)$ is a measurable function of the internal randomness and the data with values in $[0,1]$, so Pinsker's inequality gives
\begin{align*}
\Ep{v}{\widehat\pi(\dsB_\sigma(v)\mid s_0)}&\leq\Ep{0}{\widehat\pi(\dsB_\sigma(v)\mid s_0)}+\sqrt{\tfrac12\KLdiv{\Pr_0}{\Pr_v}}\\
&\leq\Ep{0}{\widehat\pi(\dsB_\sigma(v)\mid s_0)}+L\sigma^\beta\sqrt{\tfrac43\Ep{0}{N_v}} .
\end{align*}

\emph{Averaging.} The balls $\dsB_\sigma(v)$ are disjoint, so $\sum_{v\in\calV}\widehat\pi(\dsB_\sigma(v)\mid s_0)\leq1$ and $\sum_{v\in\calV}N_v\leq N$.
Averaging the last display over $v\in\calV$, we obtain
\begin{align*}
\frac1{|\calV|}\sum_{v\in\calV}\Ep{v}{\widehat\pi(\dsB_\sigma(v)\mid s_0)}&\leq\frac1{|\calV|}+L\sigma^\beta\Bigl(\frac43\cdot\frac1{|\calV|}\sum_{v\in\calV}\Ep{0}{N_v}\Bigr)^{1/2} &\prstep{(concavity of $x\mapsto\sqrt x$)}\\
&\leq\frac14+L\sigma^\beta\Bigl(\frac43N(4\sigma)^{d_A}\Bigr)^{1/2} &\prstep{($|\calV|\geq\max\{4,(4\sigma)^{-d_A}\}$)}\\
&=\frac14+\frac14, &\prstep{(definition of $\sigma_N$)}
\end{align*}
since $\frac43\cdot4^{d_A}L^2N\sigma_N^{d_A+2\beta}=\frac1{16}$.

\emph{Conclusion.} The maximum over $v$ is at least the average, so the value gap gives
\begin{equation*}
\max_{v\in\calV}\Ep{v}{\|V^*-V^{\widehat\pi}\|_{1,\rho}}\geq\frac{L\sigma_N^\beta}{1-\gamma}\Bigl(1-\frac1{|\calV|}\sum_{v\in\calV}\Ep{v}{\widehat\pi(\dsB_\sigma(v)\mid s_0)}\Bigr)\geq\frac{L\sigma_N^\beta}{2(1-\gamma)},
\end{equation*}
and $L\sigma_N^\beta=(3/(64\cdot4^{d_A}))^{\beta/(d_A+2\beta)}L^{d_A/(d_A+2\beta)}N^{-\beta/(d_A+2\beta)}$, which is the claim.
The proof uses the bump only through $0\leq h_v\leq L\sigma^\beta\dsone_{\dsB_\sigma(v)}$ and $\max h_v=L\sigma^\beta$, so it holds verbatim for the smooth bump and $\beta\in(1,2]$.
The two conditions on $\sigma_N$ hold if and only if
\begin{equation*}
N\geq\max\Bigl\{\frac{3\cdot2^{d_A+6\beta-6}}{L^2},\ \frac34\,4^{d_A/\beta-d_A}L^{d_A/\beta}\Bigr\}.
\end{equation*}
\end{proofof}

\begin{proofof}{\cref{thm:lower}(b)}
\emph{Value gap.} Fix $v\in\calV$.
Every transition leads to $s_0$, so, by \cref{thm:fixedpoints}, every policy $\pi$ has $V^\pi(s)=\int r_v(s,a)\,\pi(\dd a\mid s)+\gamma V^\pi(s_0)$ for all $s\in S$, and $V^*(s)=\max_ar_v(s,a)+\gamma V^*(s_0)$.
At $s_0$ this gives $V^\pi(s_0)=(\frac12+\int h_v(s_0,a)\,\pi(\dd a\mid s_0))/(1-\gamma)$ and $V^*(s_0)=(\frac12+L\sigma)/(1-\gamma)$.
Since $\max_ar_v(s,a)\geq\int r_v(s,a)\,\widehat\pi(\dd a\mid s)$ and $h_v(s_0,\cdot)\leq L\sigma\dsone_{\dsB_\sigma(v)}$, every $s\in S$ satisfies
\begin{equation*}
V^*(s)-V^{\widehat\pi}(s)\geq\gamma\bigl(V^*(s_0)-V^{\widehat\pi}(s_0)\bigr)\geq\frac{\gamma L\sigma}{1-\gamma}\bigl(1-\widehat\pi(\dsB_\sigma(v)\mid s_0)\bigr),
\end{equation*}
and integrating against $\rho$ gives the same lower bound for $\|V^*-V^{\widehat\pi}\|_{1,\rho}$.

\emph{Information.} Define $P_0$, $\Pr_0$ and $\E_0$ as in (a), and let $N_v\defeq|\{t\leq N:\|S_t-s_0\|<\sigma,\ A_t\in\dsB_\sigma(v)\}|$.
Then $h_v^2\leq L^2\sigma^2$ on the set that $N_v$ counts, and $h_v=0$ off it, so \cref{lem:divergence} gives $\KLdiv{\Pr_0}{\Pr_v}\leq\frac83L^2\sigma^2\Ep{0}{N_v}$, and Pinsker's inequality gives the last display of the information step in (a) with $\beta=1$.
The balls $\dsB_\sigma(v)$ are disjoint, so $\sum_{v\in\calV}N_v\leq|\{t\leq N:\|S_t-s_0\|<\sigma\}|$.
The open cube $\{s:\|s-s_0\|<\sigma\}$ lies in $S$, since $\sigma\leq\frac18$, and has $\rho$-measure $(2\sigma)^{d_S}$.
The density bound on the laws of the input states is a property of the learner, so it holds under $P_0$ as well, and $\Ep{0}{\sum_{v\in\calV}N_v}\leq NM_S(2\sigma)^{d_S}$.

\emph{Averaging and conclusion.} As in (a),
\begin{equation*}
\frac1{|\calV|}\sum_{v\in\calV}\Ep{v}{\widehat\pi(\dsB_\sigma(v)\mid s_0)}\leq\frac14+L\sigma\Bigl(\frac43NM_S(2\sigma)^{d_S}(4\sigma)^{d_A}\Bigr)^{1/2}=\frac12,
\end{equation*}
because $\frac43\cdot2^{d_S}4^{d_A}M_SL^2N\sigma_N^{d_S+d_A+2}=\frac1{16}$.
With the value gap, $\max_{v\in\calV}\Ep{v}{\|V^*-V^{\widehat\pi}\|_{1,\rho}}\geq\gamma L\sigma_N/(2(1-\gamma))$, and $L\sigma_N=(3L^{d_S+d_A}/(64\cdot2^{d_S}4^{d_A}M_SN))^{1/(d_S+d_A+2)}$.
The two conditions on $\sigma_N$ hold if and only if
\begin{equation*}
N\geq\max\Bigl\{\frac{3\cdot2^{2d_S+d_A}}{M_SL^2},\ \frac34\cdot\frac{2^{d_S}L^{d_S+d_A}}{M_S}\Bigr\}.
\end{equation*}
\end{proofof}

\begin{lemma}[Conditions of the upper bounds for the hard families]\label{lem:lower-conditions}
In the setting of \cref{thm:lower-full}, the following holds.
In both parts, the finite class $\calF\defeq\{c^*+h_u:u\in\calV\}$, where $c^*\defeq\frac12+\gamma(\frac12+L\sigma_N^\beta)/(1-\gamma)$ with $\beta=1$ in (b), satisfies \cref{asm:class} with $R_{\max}=1$ and $B=(1-\gamma)^{-1}$, contains $Q^*$ and is Bellman complete for every $P_v$.
Let the data be fresh batches of uniform actions, with uniform restarts in (b), and let the regression be exact least squares over $\calF$.
In (a), every $P_v$ satisfies the conditions of \cref{cor:rates}(b$'$) other than the one on $n$, with $L_\beta=2L$ for $\beta\leq1$ and $L_\beta=16d_A^2L$ for $\beta\in(1,2]$, and for $\beta=1$ also those of \cref{cor:rates}(b) with $2L$ in place of $L$.
In (b), every $P_v$ satisfies the conditions of \cref{cor:ratesSA} other than the one on $n$, with $L_S=L_A=2L$.
In both parts, $\eapp[,k]=\eopt[,k]=0$.
\end{lemma}

\begin{proofof}{\cref{lem:lower-conditions}}
\emph{Model and class.} In both parts, $r_v=\frac12+h_v$ is continuous and $P(\cdot\mid s,a)=\delta_{s_0}$ does not depend on $(s,a)$, so \cref{asm:model} holds, and the rewards lie in $\{0,1\}$, so $R_{\max}=1$.
Write $L\sigma^\beta$ for the height of the bumps, with $\beta=1$ in (b), and let $u,v\in\calV$.
Every action leads to $s_0$, where the maximum of $c^*+h_u$ over the actions is $c^*+L\sigma^\beta$, attained at $u$, and $c^*+L\sigma^\beta=(\frac12+L\sigma^\beta)/(1-\gamma)$ by the definition of $c^*$.
So, in $P_v$,
\begin{equation*}
T^*(c^*+h_u)=\frac12+h_v+\gamma\bigl(c^*+L\sigma^\beta\bigr)=\frac12+h_v+\frac{\gamma(\frac12+L\sigma^\beta)}{1-\gamma}=c^*+h_v,
\end{equation*}
so $\calF$ is Bellman complete for $P_v$, and $Q^*=c^*+h_v\in\calF$ is the unique fixed point of $T^*$ (\cref{thm:fixedpoints}).
The parameter set $\calV$ is finite, each $c^*+h_u$ is continuous, and $0\leq c^*+h_u\leq(\frac12+L\sigma^\beta)/(1-\gamma)\leq(1-\gamma)^{-1}$, so \cref{asm:class} holds with $B=(1-\gamma)^{-1}=R_{\max}/(1-\gamma)$.
Bellman completeness gives $\eapp[,k]=0$ (\cref{sec:regression}), and least squares over the finite class, with ties broken in a fixed order of $\calV$, satisfies \cref{asm:erm} with $\eopt[,k]=0$.

\emph{Data.} With fresh batches whose actions are drawn from the uniform law $\lambda_A$ on $A$, and whose input states are drawn from $\rho$, every design measure is $\mu_{k,i}=\rho\otimes\lambda_A$, and \cref{asm:adaptive-batches} holds.
For part (a), conditions (a) and (b) of \cref{prop:sufficient} hold with $\lambda_S=\rho=\delta_{s_0}$ and $D=c_\mu=1$.
For part (b), the floor (b) of \cref{prop:sufficient} holds with $\lambda_S=\rho$ and $c_\mu=1$, and the input states have density $1$ with respect to $\rho$, so FQI with these data is a learner as in \cref{thm:lower}(b) with $M_S=1$.

\emph{Moduli.} Every residual is $f_\theta-T^*f_{\theta'}=h_u-h_v$ for some $u,v\in\calV$.
For the tent of (a) with $\beta\leq1$, the map $t\mapsto L\sigma^\beta(1-t/\sigma)_+$ is $L\sigma^{\beta-1}$-Lipschitz and bounded by $L\sigma^\beta$, so $|h_u(a)-h_u(b)|\leq L\sigma^\beta\min\{1,\|a-b\|/\sigma\}\leq L\|a-b\|^\beta$, where the last step uses $\beta\leq1$ if $\|a-b\|\leq\sigma$ and $\beta>0$ otherwise.
So $\omega^{\mathrm{lin}}_\calF(t)\leq\omega_\calF(t)\leq2Lt^\beta$ for all $t>0$.
The cube $A$, with its center as $o$ and $r_A=\frac12$, is as in \cref{def:smoothing2}, so the conditions of \cref{cor:rates}(b$'$) hold with $L_\beta=2L$, and for $\beta=1$ those of \cref{cor:rates}(b) hold with $2L$ in place of $L$.
In (b), the map $t\mapsto L\sigma(1-t/\sigma)_+$ is $L$-Lipschitz and $|\max\{x,y\}-\max\{x',y'\}|\leq\max\{|x-x'|,|y-y'|\}$, so $|h_u(s,a)-h_u(s',a')|\leq L\max\{\|s-s'\|,\|a-a'\|\}$.
So $\omega^{\mathrm{SA}}_\calF(\sigma_S,\sigma_A)\leq2L\max\{\sigma_S,\sigma_A\}\leq2L\sigma_S+2L\sigma_A$, and the cubes $S$ and $A$, with the uniform laws $\rho$ and $\lambda_A$, are as in \cref{def:smoothingSA}.

\emph{Smooth bumps.} For $\beta\in(1,2]$, let $\varphi(t)\defeq(1-t^2)_+^2$ and $\psi(x)\defeq\prod_{i=1}^{d_A}\varphi(x_i)$, so that $h_u(a)=L\sigma^\beta\psi((a-u)/\sigma)$.
The function $\varphi$ is continuously differentiable with $0\leq\varphi\leq1$ and $|\varphi'|\leq8/(3\sqrt3)$, and $\varphi'$ is $8$-Lipschitz, since $|\varphi''(t)|=|12t^2-4|\leq8$ for $|t|<1$ and $\varphi'=0$ for $|t|\geq1$.
Each partial derivative $\partial_i\psi(x)=\varphi'(x_i)\prod_{j\neq i}\varphi(x_j)$ is a product of $d_A$ factors.
Changing one factor at a time, and bounding the other factors by $|\varphi'|\leq8/(3\sqrt3)$ and $\varphi\leq1$, gives $|\partial_i\psi(x)-\partial_i\psi(y)|\leq(8+\frac{64}{27}(d_A-1))\|x-y\|$.
With the dual norm $\|\cdot\|_1$ of $\|\cdot\|$, summing over $i$ gives $\|\nabla\psi(x)-\nabla\psi(y)\|_1\leq8d_A^2\|x-y\|$, and $\|\nabla\psi(x)\|_1\leq8d_A/(3\sqrt3)\leq4d_A^2$, so $\|\nabla\psi(x)-\nabla\psi(y)\|_1\leq8d_A^2\min\{\|x-y\|,1\}$.
Since $\nabla h_u(a)=L\sigma^{\beta-1}\nabla\psi((a-u)/\sigma)$, we have, for $a,b\in A$,
\begin{equation*}
\|\nabla h_u(a)-\nabla h_u(b)\|_1\leq8d_A^2L\sigma^{\beta-1}\min\{\|a-b\|/\sigma,1\}\leq8d_A^2L\|a-b\|^{\beta-1},
\end{equation*}
where the last step uses $\sigma^{\beta-2}\leq\|a-b\|^{\beta-2}$ if $\|a-b\|\leq\sigma$ and $\sigma^{\beta-1}\leq\|a-b\|^{\beta-1}$ otherwise.
Taylor's formula along the segment from $b$ to $a$ gives
\begin{equation*}
\bigl|h_u(a)-h_u(b)-\langle\nabla h_u(b),a-b\rangle\bigr|\leq\int_0^18d_A^2L\,(t\|a-b\|)^{\beta-1}\|a-b\|\,\dd t=\frac{8d_A^2L}\beta\,\|a-b\|^\beta,
\end{equation*}
so $\omega^{\mathrm{lin}}_{h_u}(t)\leq8d_A^2Lt^\beta/\beta$.
The affine modulus is subadditive, because the affine approximations of two functions can be added, so $\omega^{\mathrm{lin}}_\calF(t)\leq16d_A^2Lt^\beta/\beta\leq16d_A^2Lt^\beta$ for all $t>0$.

\emph{Entropy.} The class $\calF$ is an internal cover of itself at every scale, so $N_\infty(B/n)\leq|\calV|\leq(2\sigma_N)^{-d_A}$, and the localized choice of $a_n$ in \cref{subsec:rates} applies, since $B>0$.
The thresholds on $N$ remove the dependence on $L$ from this bound.
In (a), $L\sigma_N^\beta\leq\frac14$ and the definition of $\sigma_N$ give $\sigma_N^{d_A+2\beta}=3/(64\cdot4^{d_A}L^2N)\geq3\sigma_N^{2\beta}/(4^{d_A+1}N)$, so $\sigma_N^{d_A}\geq3/(4^{d_A+1}N)$ and $\log N_\infty(B/n)\leq\log(2^{d_A+2}N/3)$.
In (b), the same argument with $L\sigma_N\leq\frac14$ and $\sigma_N\leq1$ gives $\sigma_N^{d_A}\geq\sigma_N^{d_S+d_A}\geq3/(4\cdot2^{d_S}4^{d_A}M_SN)$, so $\log N_\infty(B/n)\leq\log(2^{d_S+d_A+2}M_SN/3)$.
\end{proofof}

\emph{Matching upper bounds.} Take $K\defeq\lceil(1-\gamma)^{-1}\log N\rceil$ rounds of $n=N/K$ transitions, so that $\gamma^K\leq1/N$ and the initialization term of \cref{cor:rates} is at most $8B/((1-\gamma)N)=8/((1-\gamma)^2N)$.
We assume that $K$ divides $N$.
Otherwise, for $N\geq K$, FQI uses only the first $K\lfloor N/K\rfloor\geq N/2$ transitions, and the bound on $a_n$ below grows by at most a factor $\sqrt2$, since $\lfloor N/K\rfloor\geq N/(2K)$ and the logarithmic terms do not grow.
With $\Bres=2/(1-\gamma)$ and the entropy bound above, the localized $a_n$ is at most a constant that depends only on the dimensions times $(1-\gamma)^{-3/2}(\log^2N/N)^{1/2}$ for $N\geq3$, in (b) with $M_S=1$.
In (a), the first threshold on $N$ gives $L\geq(3\cdot2^{d_A+6\beta-6}/N)^{1/2}$, so $L^{d_A/(d_A+2\beta)}N^{-\beta/(d_A+2\beta)}$ is at least a constant that depends only on $d_A$ and $\beta$ times $N^{-1/2}$.
So the initialization term is at most a constant times the bound below, and the first threshold of (b) gives the same in (b).
Once $N$ is large enough for the condition on $n$ in \cref{cor:rates}(b$'$), which depends on $L$ and $\gamma$, this corollary bounds the expected policy loss in part (a) by a constant that depends only on $d_A$ and $\beta$ times
\begin{equation*}
(1-\gamma)^{-2-3\beta/(d_A+2\beta)}\,L^{d_A/(d_A+2\beta)}\bigl(\log^2N/N\bigr)^{\beta/(d_A+2\beta)}.
\end{equation*}
\Cref{cor:ratesSA} gives the same order in part (b), with $d_S+d_A$ in place of $d_A$ and $\beta=1$, once $t_n\leq2L$.
These upper bounds match the lower bounds of \cref{thm:lower} in $N$ and $L$ up to the factor $(\log^2N)^{\beta/(d_A+2\beta)}$ in part (a) and $(\log^2N)^{1/(d_S+d_A+2)}$ in part (b), and the powers of $(1-\gamma)^{-1}$ differ, since the lower bounds carry $(1-\gamma)^{-1}$ only.

\emph{Norm equivalence on the families of (a).} In $P_v$, every residual of $\calF$ is $h_u-h_v$ for some $u\in\calV$, since $T^*(c^*+h_{u'})=c^*+h_v$ for every $u'\in\calV$ (proof of \cref{lem:lower-conditions}).
For $u\neq v$ the two bumps have the same height $L\sigma^\beta$ and disjoint supports, so $\sup_a(h_u-h_v)^2=L^2\sigma^{2\beta}$ and $\int(h_u-h_v)^2\,\dd\lambda_A=2\int h_u^2\,\dd\lambda_A$.
As in the proof of \cref{prop:tight}, the tent has $\int h_u^2\,\dd\lambda_A=2^{d_A+1}L^2\sigma^{2\beta+d_A}/((d_A+1)(d_A+2))$, and the smooth bump has $\int h_u^2\,\dd\lambda_A=(256/315)^{d_A}L^2\sigma^{2\beta+d_A}$, since $\int_{-1}^1(1-t^2)^4\,\dd t=256/315$.
So the smallest $\Lambda$ for which the norm equivalence of \cref{cor:rates}(c) holds with $\lambda'_A=\lambda_A$ is $(d_A+1)(d_A+2)2^{-d_A-2}\sigma_N^{-d_A}$ for $\beta\leq1$ and $\frac12(315/256)^{d_A}\sigma_N^{-d_A}$ for $\beta\in(1,2]$, which is of order $(L^2N)^{d_A/(d_A+2\beta)}$ by the definition of $\sigma_N$.
With $D=c_\mu=1$ and the bound on $a_n$ above, the bound in \cref{cor:rates}(c) is then a constant that depends only on $d_A$ and $\beta$ times $(1-\gamma)^{-7/2}L^{d_A/(d_A+2\beta)}N^{-\beta/(d_A+2\beta)}\log N$, which has the exponent of \cref{cor:rates}(b$'$).

\section{\texorpdfstring{Threshold experiment}{Threshold experiment}}\label{app:threshold}

We simulate the family of \cref{thm:lower}(a) with $L=\beta=1$ and $d_A\in\{1,2\}$, for radii $\sigma$ with $1/(2\sigma)$ an integer, so that the balls $\dsB_\sigma(v)$, $v\in\calV$, cover $A$ up to a null set.
The center $v$ is drawn uniformly from $\calV$, and the actions uniformly from $A$.
The targets of FQI differ from the rewards by the constant $\gamma\max_af(s_0,a)$, so its regression step is least squares over the class of \cref{lem:lower-conditions} with a free constant, $\{c+h_u:u\in\calV,\ c\in\dsR\}$, which has a closed form in each ball.
We record whether it selects the ball of the bump.
For each $\sigma$, we estimate the sample size $N^*(\sigma)$ at which this happens with probability $\frac12$, from $1000$ trials at each $N$ of a geometric grid with ratio $2^{1/4}$, by log-linear interpolation of the isotonic fit of the success rates.
The estimates differ by at most $8.3\%$ from a Gaussian approximation in which the ball of the bump has to beat the largest of $|\calV|-1$ independent competitors, $N^*(\sigma)\approx|\calV|z^2/(4\sigma^2\Ex{t^2})$, which lies inside all bootstrap intervals (\cref{fig:threshold}).
Here $z$ is the shift at which a standard normal variable exceeds the maximum of $|\calV|-1$ independent ones with probability $\frac12$, which behaves as $\sqrt{2\log|\calV|}$ for large $|\calV|$.
The second moment of the unit tent at a uniform point of its ball is $\Ex{t^2}=2/((d_A+1)(d_A+2))$, so $N^*(\sigma)$ is of order $\sigma^{-(d_A+2)}\log(1/\sigma)$.
Dividing by $\sigma^{d_A+2}$ makes this factor and the bootstrap intervals visible: $N^*(\sigma)\,\sigma^{d_A+2}$ grows with $\log(1/\sigma)$ and follows the Gaussian approximation (\cref{fig:threshold-normalized}).
For a uniformly drawn center, the proof of \cref{thm:lower}(a) shows that no learner finds the bump with probability $\frac12$ from fewer than $3(\frac12-1/|\calV|)^2|\calV|/(4\sigma^2)$ samples, which is of order $\sigma^{-(d_A+2)}$.
So the regression step of FQI attains the exponent of the lower bound up to a logarithmic factor.
For $\beta=2$, we use the smooth bump of \cref{thm:lower-full}(a) with $L=16$, which keeps $L\sigma^2\leq\frac14$ for $\sigma\leq\frac18$, and replace $\sigma$ by the bump height $L\sigma^2$ in both formulas, with $\Ex{t^2}=(128/315)^{d_A}$ for the smooth bump.
The thresholds then grow with the exponent $d_A+4$ of the lower bound, up to a logarithmic factor, and differ by at most $6.1\%$ from the Gaussian approximation (open markers in \cref{fig:threshold}).
Both families have the bump height $\frac1{16}$ at $\sigma=\frac1{16}$, so their thresholds there differ only through $\Ex{t^2}$.
The curves for $\beta=1$ and $\beta=2$ meet for $d_A=2$, where $(128/315)^2$ is close to $\frac16$, and differ by a factor of about $1.2$ for $d_A=1$, where $\frac{128}{315}$ exceeds $\frac13$.

\begin{figure}[t]
\centering
\includegraphics{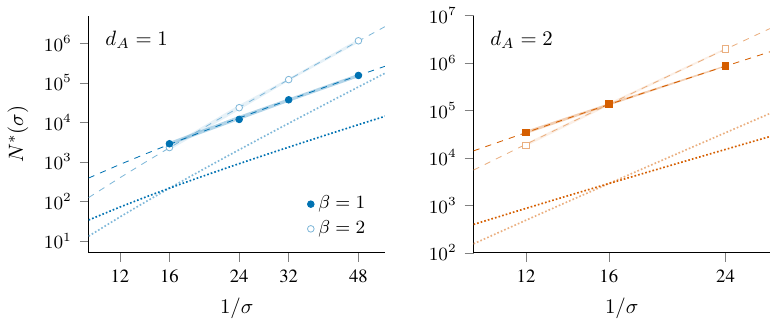}
\caption{\emph{Threshold of Least Squares.} The sample size $N^*(\sigma)$ at which the regression step of FQI finds a bump of radius $\sigma$ with probability $\frac12$ on the family of \cref{thm:lower}(a), for $d_A=1$ (left) and $d_A=2$ (right), on logarithmic axes.
Filled markers show the tent ($\beta=1$, $L=1$), and open markers and lighter lines the smooth bump ($\beta=2$, $L=16$).
Markers show the estimates with $95\%$ bootstrap bands over the trials (narrower than the markers at most points), and dashed lines the Gaussian approximation of \cref{app:threshold} without fitted constants.
Dotted lines show the information bound of the proof of \cref{thm:lower}(a), which holds for every learner.
The measured thresholds grow with the exponent $d_A+2\beta$ of the lower bound, up to a logarithmic factor, so the smooth bump needs more samples once $\sigma<\frac1{16}$.}
\label{fig:threshold}
\end{figure}

\begin{figure}[t]
\centering
\includegraphics{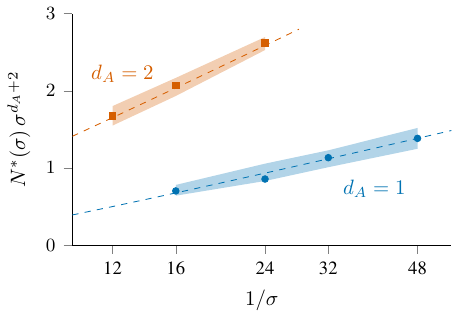}
\caption{\emph{Normalized Threshold.} The thresholds of \cref{fig:threshold} for the tent ($\beta=1$) multiplied by $\sigma^{d_A+2}$, for $d_A=1$ (blue circles) and $d_A=2$ (orange squares), against $1/\sigma$ on a logarithmic axis, with $95\%$ bootstrap bands and the Gaussian approximation of \cref{app:threshold} (dashed).
With the exponent divided out, the threshold still grows with $\log(1/\sigma)$, as the approximation predicts, and the approximation lies inside every bootstrap interval.}
\label{fig:threshold-normalized}
\end{figure}

\clearpage

\section{Function classes}\label{app:classes}

\textbf{Linear Classes and RKHS Balls.} These classes are covered by \cref{prop:hilbert}: $\alpha_n\leq2\Bres^2/\sqrt n$ with no logarithmic factor, and $\log(1/\delta)$ tails.
Finite-dimensional linear classes are also covered by \cref{cor:fast}, which localizes the argument and gives the residual rate $\sqrt{d\log n/n}$ instead of $n^{-1/4}$.

\textbf{General Classes.} Let $\mathfrak M_n(\calF)\defeq\sup_{\mbx}\E_\epsilon\sup_\theta\sum_t\epsilon_tf_\theta(x_t(\epsilon))$ be the (unnormalized) sequential Rademacher complexity of $\calF$ over $S\times A$-valued trees, and let $c_n\defeq1+4\sqrt2\log^{3/2}(en^2)$.

\begin{proposition}\label{prop:contraction}
Under \cref{asm:class}, if $\mathfrak M_n(\calF)\geq B$, then $\overline\Rseq_n\leq\Bres^2\sqrt n+32\,\Bres c_n\,\mathfrak M_n(\calF)$.
\end{proposition}

\begin{proof}

For any process $(X_\theta)$ and any $\theta_0$, $\sup_\theta|X_\theta|\leq\sup_\theta X_\theta+\sup_\theta(-X_\theta)+|X_{\theta_0}|$, and $\E|\sum_t\epsilon_t\ell_{\theta_0}(z_t)|\leq\Bres^2\sqrt n$.
Let $\calG\defeq\calF/B\subseteq[-1,1]^{Z}$ (functions of $z$ through $(s,a)$), and let $\chi_\pm(u,z)\defeq\pm(B\,\mathrm{clip}_{[-1,1]}(u)-y(z))^2$, which is $2B\Bres$-Lipschitz in $u$ for every $z$.
The contraction inequality of \citet[Cor.~5]{rakhlin2015}, applicable when the normalized complexity of $\calG$ is at least $1/n$, i.e., $\mathfrak M_n(\calF)\geq B$, bounds the sequential complexity of $\chi_\pm\circ\calG=\pm\calL_{\theta'}$ by $8\cdot2B\Bres c_n\,\mathfrak M_n(\calF)/B$.
\end{proof}

\textbf{Neural Nets.} Let $S\times A\subseteq[-X_\infty,X_\infty]^{m_0}$.
Let $\calF_1=\{x\mapsto\langle w,x\rangle:\|w\|_1\leq B_1\}$ and $\calF_\ell=\{x\mapsto\sum_jw_j\,\mathrm{relu}(f_j(x)):f_j\in\calF_{\ell-1},\|w\|_1\leq B_\ell\}$ with the ReLU activation $\mathrm{relu}$, and let $\calF=\calF_J$.

\begin{proposition}\label{prop:nn}
If $\mathfrak M_n(\calF_\ell)\geq X_\infty\prod_{j=1}^\ell B_j$ for $1\leq\ell\leq J-1$ (the normalized complexities of the rescaled layers are at least $1/n$, as used in the proof of \citealp[Prop.~15]{rakhlin2015}), then $\mathfrak M_n(\calF_J)\leq\bigl(\prod_{\ell=2}^J16B_\ell\bigr)c_n^{J-1}\,B_1X_\infty\sqrt{2n\log(2m_0)}$.
If moreover $\calF_J$ satisfies \cref{asm:class} for some parametrization, with $B\geq\max\{X_\infty\prod_\ell B_\ell,R_{\max}/(1-\gamma)\}$, and $\mathfrak M_n(\calF_J)\geq B$, then \cref{prop:contraction} gives
\begin{equation*}
\alpha_n\leq\frac{2\Bres^2}{\sqrt n}+64\sqrt2\,\Bres c_n^{J}\Bigl(\prod_{\ell=2}^J16B_\ell\Bigr)B_1X_\infty\sqrt{\frac{\log(2m_0)}{n}},
\end{equation*}
with no dependence on width.
\end{proposition}

\begin{proof}

\citet[Prop.~15]{rakhlin2015} prove the bound for $1$-Lipschitz activations from $\dsR$ to $[-1,1]$ that vanish at $0$.
For ReLU, let $M_\ell\defeq X_\infty\prod_{j\leq\ell}B_j\geq\sup_{f\in\calF_\ell}\|f\|_\infty$.
Positive homogeneity gives $\calF_\ell/M_\ell=\{\sum_ju_j\overline{\mathrm{relu}}(g_j):\|u\|_1\leq1,\ g_j\in\calF_{\ell-1}/M_{\ell-1}\}$ with $\overline{\mathrm{relu}}=\max\{0,\min\{\cdot,1\}\}$, which coincides with ReLU on $[-1,1]$.
Apply the proposition to these rescaled classes (with $B_\ell=1$) and undo the scaling.
The first layer is bounded by the finite-class inequality for the $2m_0$ signed coordinate functions, together with invariance under convex hulls.
For the exponent $J-1$, let $2\leq\ell\leq J$.
The step of the proof of \citet[Prop.~15]{rakhlin2015} (their display (29), via \citealp[Cor.~5]{rakhlin2015}) bounds the normalized complexity of $\calF_\ell/M_\ell$ by $16c_n$ times that of $\calF_{\ell-1}/M_{\ell-1}$, which needs the latter to be at least $1/n$.
The base case gives at most $\sqrt{2\log(2m_0)/n}$ for $\calF_1/M_1$, and iterating $J-1$ times and multiplying by $nM_J$ gives the bound.
\end{proof}

Networks are Lipschitz in $a$ with constant at most $\prod_\ell B_\ell$ (for $\ell_\infty$ inputs), which enters $\omega_\calF$ through \cref{lem:modulus}.
For the width-free networks of \cref{prop:nn} and for RKHS balls of large metric entropy, the residual rate stays at $n^{-1/4}$, which is the best the global complexity can give (\cref{subsec:fast}).
Networks of fixed architecture and RKHS balls of Gaussian kernels have the localized rate of \cref{thm:fast-entropy}, up to logarithmic factors (\cref{lem:entropy,subsec:fast}).

\section{Data reuse}\label{app:reuse}

\subsection{\texorpdfstring{Scope of the sampling model}{Scope of the sampling model}}\label{app:scope}

\Cref{asm:adaptive-batches} requires \emph{fresh} transitions: every transition enters exactly one regression, in the round in which it is collected.
Three data-reuse schemes should be distinguished.
(a) \emph{Fresh adaptive batches}, as above, with arbitrary history-dependent collection, are covered by \cref{thm:deviation,thm:residual,cor:fast,thm:main}.
(b) \emph{A single i.i.d.\ data set reused in all rounds} is not an instance of \cref{asm:adaptive-batches}, because the labels of round $k$ depend on the data through $\widehat Q_k$.
It is covered by a classical argument that is uniform over the pair (regression function, target), of the type used by \citet{antos2008}; see \cref{prop:single-dataset}.
(c) \emph{General replay reuse}, in which stored transitions re-enter later regressions, is not covered.
Replay-based deep RL algorithms such as SAC \citep{haarnoja2018soft} or TD3 \citep{fujimoto2018addressing} fall under (c).

\begin{proposition}[One i.i.d.\ data set reused in every round]\label{prop:single-dataset}
Let \cref{asm:model,asm:class} hold.
Let $Z_1,\dots,Z_n$ be i.i.d.\ transitions whose state--action pairs have law $\mu$, with conditional mean reward $r(S_i,A_i)$ and $S'_i\sim P(\cdot\mid S_i,A_i)$, and let every round use this data set: $\widehat\theta_{k+1}$ satisfies \cref{asm:erm} for $\widehat L(\cdot;\widehat\theta_k)$, where $\widehat L(\theta;\theta')\defeq\frac1n\sum_i(f_\theta(S_i,A_i)-R_i-\gamma V(\theta',S'_i))^2$.
Let $L(\theta;\theta')\defeq\E(f_\theta(S,A)-R-\gamma V(\theta',S'))^2$ and $\Delta^{\mathrm{pair}}\defeq\sup_{\theta,\theta'}|L(\theta;\theta')-\widehat L(\theta;\theta')|$.
Then, almost surely for every $k$,
\begin{equation*}
\|\widehat Q_{k+1}-T^*\widehat Q_k\|_{2,\mu}\leq\inf_\theta\|f_\theta-T^*\widehat Q_k\|_{2,\mu}+\sqrt{2\Delta^{\mathrm{pair}}}+\eopt[,k].
\end{equation*}
The deviation satisfies $\E\Delta^{\mathrm{pair}}\leq\frac2nR_n(\calL^{\mathrm{pair}})$, where $R_n$ is the i.i.d.\ Rademacher complexity of $\calL^{\mathrm{pair}}\defeq\{z\mapsto(f_\theta(s,a)-r-\gamma V(\theta',s'))^2:\theta,\theta'\in\Theta\}$, and $\Delta^{\mathrm{pair}}\leq\E\Delta^{\mathrm{pair}}+\Bres^2\sqrt{\log(1/\delta)/(2n)}$ with probability at least $1-\delta$.
The proof of \cref{thm:main} uses \cref{asm:adaptive-batches,asm:erm} only through the residual bounds, so its conclusions hold with $\bar\mu_k=\mu$ (including \cref{asm:transfer}, with $\mu$ in place of $\bar\mu_k$) and $\veps_k=\Ex{\eapp[,k]}+(2\E\Delta^{\mathrm{pair}})^{1/2}+\eopt[,k]$, where $\eapp[,k]\defeq\inf_\theta\|f_\theta-T^*\widehat Q_k\|_{2,\mu}$.
\end{proposition}

\begin{proof}

For fixed $\theta'$, the argument of \cref{thm:residual} with a fixed design gives $L(\theta;\theta')=\|f_\theta-T^*f_{\theta'}\|_{2,\mu}^2+v(\theta')$.
The bound on $\Delta^{\mathrm{pair}}$ is uniform in $\theta'$, so it can be applied at the data-dependent $\theta'=\widehat\theta_k$, and the chain of inequalities in the proof of \cref{thm:residual} goes through.
The expectation bound is classical symmetrization, and the tail bound is McDiarmid's inequality: the losses lie in $[0,\Bres^2]$, so changing one sample changes $\Delta^{\mathrm{pair}}$ by at most $\Bres^2/n$.
\end{proof}

With continuous actions, the losses in $\calL^{\mathrm{pair}}$ depend on $\theta'$ through $V(\theta',s')=\max_{a'}f_{\theta'}(s',a')$.
For classes of finite metric entropy in the supremum norm this at most doubles the log-covering number, up to a constant in the scale, since $\|V_f-V_g\|_\infty\leq\|f-g\|_\infty$.
For complexities measured at the data points, such as Rademacher complexities, it requires additional structure, for instance Lipschitz continuity in $a'$ together with a cover of $A$.
This uniformity over targets is the price of reuse that fresh batches avoid.

\textbf{General replay reuse.} Suppose a transition $Z_{k,i}$ that is $\calH_k$-measurable (it was stored earlier) re-enters the regression of round $k$.
Then $\Ex{Y_{k,i}\mid\calH_{k,i-1},S_{k,i},A_{k,i}}=Y_{k,i}$, because the label is already determined.
The identity \eqref{eq:condmean} fails, and $L_k(f)$ no longer decomposes into $\|f-h_k\|^2_{2,\bar\mu_k}$ plus a constant.
The conditional law of $Z_{k,i}$ given the past is also a point mass, so \cref{lem:symmetrization} yields only the trivial statement that the uniform deviation is $0$ while the bias is uncontrolled.
Recovering guarantees for replay therefore requires sample splitting, which reduces to scheme (a); or stability of the regression map, or multiplier-process arguments, which control the dependence of $\widehat Q_k$ on the reused samples.
We do not pursue these here.

\section{Online regret}\label{app:regret}

\textbf{Regret of the Behavior.} \Cref{thm:main} concerns the final policy, and two cumulative regrets have to be distinguished.
The regret of the behavior is the sum of the action gaps $V^*(S_{k,i})-Q^*(S_{k,i},A_{k,i})$ of the data-collecting policy, and the regret of the greedy policies $\widehat\pi_k$ is taken at the states visited in round $k$.
The first is linear in $N$ under the protocols of \cref{prop:protocols}, because their floor $c_A$ is fixed.
Let $\Gamma(s)\defeq\int(V^*(s)-Q^*(s,a))\,\lambda_A(\dd a)$, which is nonnegative and Borel, because $Q^*\in\calQ_0$.
With $\calG\defeq\calH_{k,i-1}\vee\sigma(S_{k,i})$ as in the proof of \cref{prop:protocols}, the behavior floor and $V^*-Q^*\geq0$ give $\Ex{V^*(S_{k,i})-Q^*(S_{k,i},A_{k,i})\mid\calG}\geq c_A\Gamma(S_{k,i})$.
With restarts, $\Ex{\Gamma(S_{k,i})\mid\calH_{k,i-1}}\geq c_S\int\Gamma\,\dd\lambda_S$, so the expected sum of the action gaps over the $N=nK$ transitions is at least $c_Ac_SN\int\Gamma\,\dd\lambda_S$.
With thinned trajectories, the proof of \cref{prop:protocols}(ii) gives the state floor $c_Am\lambda_S$ for every used transition, so the same argument gives at least $c_A^2mN\int\Gamma\,\dd\lambda_S$, and the skipped transitions add nonnegative gaps.
The integral $\int\Gamma\,\dd\lambda_S$ vanishes only if $Q^*(s,a)=V^*(s)$ for $(\lambda_S\otimes\lambda_A)$-almost every $(s,a)$, so the regret of the behavior is linear in $N$ for every nontrivial problem, whatever the residuals.
A floor $c_{A,k}$ that tends to $0$ gives $c_\mu\to0$, and then the bound on $\Cad(\sigma)$ in \cref{prop:sufficient} grows without bound, so such a floor does not combine with \cref{thm:main} as stated.
This regret is the price of the coverage that \cref{thm:main} uses.

\section{\texorpdfstring{Further related work}{Further related work}}\label{app:related}

\emph{Adaptive data.} Least squares with adaptively chosen regressors goes back to \citet{lai1982least}; self-normalized martingale inequalities \citep{abbasiyadkori2011improved} and sequential complexities \citep{rakhlin2015,rakhlin2015ptrf} are the tools used here.
Inference, off-policy evaluation, risk minimization and policy learning from adaptively collected bandit data are studied by \citet{hadad2021confidence,zhan2021offpolicy,bibaut2021risk,zhan2023policy}.
In offline RL, pessimistic value iteration for linear MDPs allows data collected by any online algorithm \citep{jin2021pessimism}, and \citet{madhow2025rates} obtain minimax rates for tabular MDPs with adaptive data.
In FQI the regression targets also depend on earlier data; with fresh batches, conditioning on the history freezes them (\cref{thm:deviation,cor:fast}), whereas reuse requires uniformity over targets (\cref{app:reuse}).
Least-squares policy iteration has been analyzed with a fresh $\beta$-mixing path of the current policy in each iteration and finite actions~\citep{lazaric2012}, a special case of adaptive fresh batches.

\emph{Online analyses and lower bounds.} Online methods with general function approximation (eluder and Bellman--eluder dimensions, bilinear classes, the decision--estimation coefficient) control regret through optimism or \emph{estimation-to-decisions (E2D)} \citep{russo2013,jiang2017contextual,wang2020eluder,jin2021bellman,du2021bilinear,foster2021dec}; they analyze different algorithms and do not address measurability on Borel spaces.
Coverage cannot be dropped for value-based offline methods \citep{chen2019}, and concentrability with realizability is insufficient without further structure such as Bellman completeness \citep{foster2022impossibility}; our bounds accordingly carry an approximation term.
\Cref{tab:landscape} summarizes the comparison.

\begin{table}[t]
\centering
\caption{\emph{Finite-Sample Analyses with Function Approximation.} The first block lists offline methods based on approximate value or policy iteration, and the second block lists a hybrid method and two online methods for contrast. ``Adaptive fresh batches'' means that each round's batch is collected by a behavior rule that may depend on the entire history (\cref{asm:adaptive-batches}).
Our analysis needs a modulus in the action only in \cref{cor:rates}(b) and (b$'$), and a modulus in state and action only in \cref{cor:ratesSA}.
Among the greedy methods, \citet{song2023hybrid} and ours cover adaptive data and compact action sets: they bound the regret against a policy that the offline data cover, under a bilinear rank, and we bound the loss of the final policy under the transfer condition of \cref{asm:transfer}.}
\label{tab:landscape}
\small
\setlength{\tabcolsep}{4.2pt}
\begin{tabular}{lllll}
\toprule
\textsc{Method} & Data & Actions & Policy step & Guarantee \\
\midrule
\citet{munos2008} & generative, $\mu$ & finite & greedy & final policy \\
\citet{antos2008} & $\beta$-mixing path, fixed $\pi_b$ & finite & greedy & final policy \\
\citet{antos2007} & $\beta$-mixing path, fixed $\pi_b$ & compact, Lipschitz & policy search & final policy \\
\citet{chen2019} & i.i.d., fixed $\mu$ & finite & greedy & final policy \\
\citet{jin2021pessimism} & adaptive, any online alg. & linear MDP & pessimistic & final policy \\
\midrule
\citet{song2023hybrid} & offline + online rollouts & bilinear rank & greedy & regret \\
\citet{jin2021bellman} & online & general & optimistic & regret \\
\citet{foster2021dec} & online & general & E2D & regret \\
\midrule
Ours & adaptive fresh batches & compact & greedy & final policy \\
\bottomrule
\end{tabular}
\end{table}

\clearpage
\section{Dependencies between the results}\label{appsec:dependencies}
\Cref{fig:roadmap} shows how the results depend on each other.
Two lines of argument meet in \cref{thm:main}, the per-round residual bounds of \cref{thm:residual,thm:fast-entropy} and the propagation of residuals to policy loss in \cref{prop:propagation}.
The rates of \cref{cor:rates,cor:ratesSA} combine \cref{thm:main} with the transfers of \cref{lem:instances,lem:transfer2,lem:transferSA}.
The lower bounds of \cref{thm:lower} use the upper bounds only through \cref{lem:lower-conditions}, which checks that the hard families satisfy their conditions.

\begin{figure}[!h]
\centering
\includegraphics[width=0.9\textwidth]{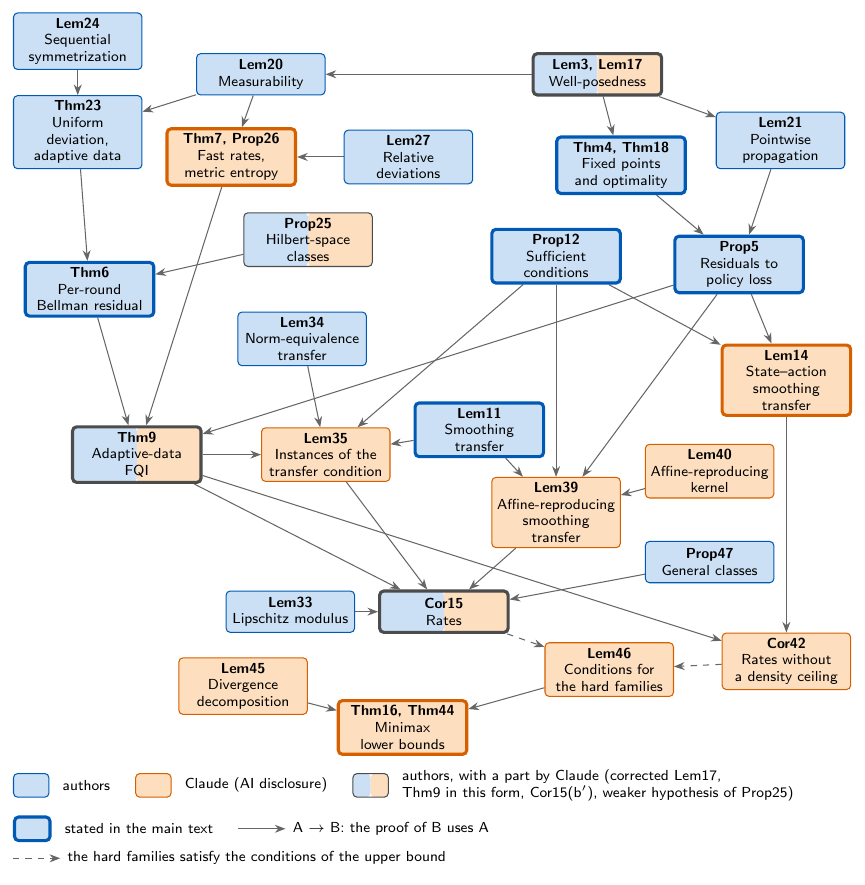}
\caption{\emph{Dependencies between the results.} An arrow from A to B means that the proof of B uses A.
The figure contains the results that the proofs of \cref{thm:main,cor:rates,cor:ratesSA,thm:lower} rely on, without the arrows that follow from longer paths, and a full form in the appendix shares its box with the main-text version.
Thick frames mark results stated in the main text.
Box colours follow the AI disclosure: blue boxes are results of the authors, orange boxes are results that Claude proposed and proved, and boxes in both colours are results of the authors with a part by Claude.
The dashed arrows are the only link from the upper to the lower bounds: \cref{lem:lower-conditions} checks that the hard families satisfy the conditions of \cref{cor:rates,cor:ratesSA}.}
\label{fig:roadmap}
\end{figure}

\end{document}